\documentclass[10pt,journal,twocolumn,cspaper,compsoc]{IEEEtran}
\ifCLASSOPTIONcompsoc
  \usepackage[nocompress]{cite}
\else
  \usepackage{cite}
\fi
\ifCLASSINFOpdf
\else
\fi

\usepackage[noend]{algpseudocode}
\usepackage{algorithm}

\usepackage{graphicx}
\usepackage{amsmath,amssymb} 
\usepackage{color}

\usepackage[utf8]{inputenc}
\usepackage{times}
\usepackage{epsfig}

\usepackage{amsmath}
\usepackage{amssymb}
\usepackage{bbold}
\usepackage{dsfont}
\usepackage{url}

\usepackage{enumerate}

\usepackage[font=small,labelfont=bf]{caption}

\usepackage{xcolor}
\usepackage{amsthm}
\usepackage{amsfonts}
\usepackage{setspace}
\usepackage{caption}
\usepackage{subcaption}
\usepackage{bm}
\usepackage{isomath}
\usepackage[numbers]{natbib}
\usepackage{bbold}
\usepackage{dsfont}
\usepackage{enumerate}

\usepackage[pagebackref=true,breaklinks=true,letterpaper=true,colorlinks,bookmarks=false]{hyperref}
\usepackage{footmisc}
\usepackage{array}
\usepackage{multirow}
\usepackage{ragged2e}
\usepackage{stmaryrd}
\usepackage{float}
\usepackage{fixltx2e}
\usepackage{dblfloatfix}
\usepackage{pbox}

\usepackage{xpatch}
\usepackage{xspace}
\usepackage{cleveref}
\usepackage{enumitem}
\usepackage{tcolorbox}
\usepackage{xcolor}
\usepackage{nameref}

\definecolor{beaublue}{rgb}{0.84, 0.9, 0.95}
\definecolor{blackish}{rgb}{0.2, 0.2, 0.2}

\definecolor{beaublue2}{rgb}{0.84, 0.9, 0.95}
\definecolor{blackish2}{rgb}{0.2, 0.2, 0.2}

\makeatletter
\newcommand\fs@nobottomruled{\def\@fs@cfont{\bfseries}\let\@fs@capt\floatc@ruled
  \def\@fs@pre{}
  \def\@fs@post{}
  \def\@fs@mid{\kern2pt\hrule\kern2pt}%
  \let\@fs@iftopcapt\iftrue}
\makeatother

\newcommand\revised[1]{{#1}}
\newcommand\revisedd[1]{{#1}}

\renewcommand\vec[1]{\ensuremath\boldsymbol{#1}}
\renewcommand\cdots{...}

\newcommand{\tF}{\vec{\mathcal{F}}}

\newcommand{\mY}{\mathbf{Y}}

\newcommand{\vy}{\mathbf{y}}

\newcommand{\mX}{\mathbf{X}}
\newcommand{\vx}{\mathbf{x}}

\newcommand{\mbr}[1]{\mathbb{R}^{#1}}

\newcommand{\vectorise}{\text{Vec}}

\newcommand{\vv}{\mathbf{v}}

\newcommand{\idx}[1]{\mathcal{I}_{#1}}
\newcommand{\semipd}[1]{\mathcal{S}_{+}^{#1}}
\newcommand{\spd}[1]{\mathcal{S}_{++}^{#1}}

\newcommand{\vu}{\mathbf{u}}

\newcommand{\vzeta}{\boldsymbol{\zeta}}
\newcommand{\vc}{\mathbf{c}}

\newcommand{\vphi}{\boldsymbol{\phi}}

\newcommand{\bigoh}{\mathcal{O}}
\newcommand{\mPsi}{\vec{\Psi}}
\newcommand{\vj}{\vec{j}}

\newcommand{\enorm}[1]{\left\|{#1}\right\|_2}

\newcommand{\set}[1]{\left\{#1\right\}}

\DeclareMathOperator*{\argmin}{arg\,min}

\DeclareMathOperator*{\sym}{Sym}

\DeclareMathOperator*{\trace}{Tr}

\DeclareMathOperator*{\diag}{diag}
\DeclareMathOperator*{\avg}{avg}

\DeclareMathOperator*{\logm}{Log}
\DeclareMathOperator*{\expm}{{Exp}}

\newcommand{\expl}[1]{\text{e}^{#1}}
\DeclareMathOperator*{\res}{Res}
\DeclareMathOperator*{\asinh}{Asinh}

\newtheorem{theorem}{Theorem}

\newtheorem{proposition}{Proposition}
\newtheorem{remark}{Remark}

\newcommand{\mLambda}{\bm{\lambda}}
\newcommand{\mU}{\bm{U}}

\newcommand{\vphibar}{\boldsymbol{\bar{\phi}}}

\def\eg{\emph{e.g.}}

\newcommand{\mygthree}[1]{\boldsymbol{\mathcal{G}}\!\left(\!#1\!\right)}

\newcommand{\tG}{\boldsymbol{\mathcal{G}}}
\newcommand{\tGhat}{\widehat{\boldsymbol{\mathcal{G}}}}

\newcommand{\mygthreee}[2]{\boldsymbol{\mathcal{G}}_{{\text{#1}}}\!\left(\!#2\!\right)}

\newcommand{\mygthreeehat}[2]{\boldsymbol{\widehat{\mathcal{G}}_{{\text{#1}}}}\!\left(\!#2\!\right)}

\newcommand{\mygthreeephat}[2]{\boldsymbol{\widehat{\mathcal{G}}'_{{\text{#1}}}}\!\left(\!#2\!\right)}

\newcommand{\vPhi}{\boldsymbol{\Phi}}

\newcommand{\mIdent}{\boldsymbol{\mathds{I}}}
\newcommand{\sIdent}{\mathds{I}}
\newcommand{\vOnes}{\mathbb{1}}

\newcommand{\mJ}{\mathbf{J}}

\newcommand{\mQ}{\mathbf{Q}}
\newcommand{\mK}{\mathbf{K}}

\newcommand{\mC}{\mathbf{C}}

\newcommand{\tS}{\vec{\mathcal{S}}}

\newcommand{\tNnb}{\mathcal{N}}

\newcommand{\mPhi}{\boldsymbol{\Phi}}
\newcommand{\mPhibar}{\boldsymbol{\bar{\Phi}}}

\newcommand{\mM}{\boldsymbol{M}}

\newcommand{\mW}{\boldsymbol{W}}
\newcommand{\mD}{\boldsymbol{D}}
\newcommand{\mT}{\boldsymbol{T}}
\newcommand{\mG}{\boldsymbol{G}}

\newcommand{\vmu}{\boldsymbol{\mu}}

\newcommand{\vvarphi}{\boldsymbol{\varphi}}

\newcommand{\stkout}[1]{{\ifmmode\text{\sout{\ensuremath{#1}}}\else\sout{#1}\fi}}

\newcommand{\mL}{\mathbf{L}}

\DeclareMathOperator*{\arcsinh}{arcsinh}
\renewcommand{\comment}[1]{}

\makeatletter
\DeclareRobustCommand\onedot{\futurelet\@let@token\bmv@onedotaux}
\def\bmv@onedotaux{\ifx\@let@token.\else.\null\fi\xspace}
\def\eg{\emph{e.g}\onedot} 
\def\ie{\emph{i.e}\onedot} 
\def\cf{\emph{c.f}\onedot} 
\def\etc{\emph{etc}\onedot} \def\vs{\emph{vs}\onedot}
\def\wrt{w.r.t\onedot}

\def\bigoh{\mathcal{O}}
\makeatother

\begin{document}
%
\title{Power Normalizations in Fine-grained Image, Few-shot Image and Graph Classification}
%
%
%
%

\author{Piotr~Koniusz 
        and~Hongguang~Zhang 
\IEEEcompsocitemizethanks{\IEEEcompsocthanksitem P. Koniusz and H. Zhang are with Data61/CSIRO (formerly known as NICTA) and the Australian National University, Canberra, Australia, ACT2601. 
E-mail: see http://claret.wikidot.com
\IEEEcompsocthanksitem Equal contribution: P. Koniusz was mainly concerned with the mathematical analysis/modeling while H. Zhang with the deep learning modeling.
}
\thanks{Manuscript submitted Dec-2018. Manuscript accepted by TPAMI on 02-Jul-2020.}
}

%
%

\markboth{IEEE Transactions on Pattern Analysis and Machine Intelligence,~Submitted, December~2018,~Accepted, July~2020}%
{Shell \MakeLowercase{\textit{et al.}}: Bare Demo of IEEEtran.cls for Computer Society Journals}
%



\IEEEtitleabstractindextext{%
\vspace{-0.1cm}
\begin{justify}
\begin{abstract}
Power Normalizations ({\em PN}) are useful non-linear operators which tackle feature imbalances in classification problems. We study PNs in the deep learning setup via a novel PN layer pooling feature maps. 
Our layer combines the feature vectors and their respective spatial locations in the feature maps produced by the last convolutional layer of CNN into a positive definite matrix with second-order statistics 
to which PN operators are applied, forming so-called Second-order Pooling ({\em SOP}). 
\revisedd{
As the main goal of this paper is to study Power Normalizations,  we investigate the role and meaning of MaxExp and Gamma, two popular PN functions. 
To this end, we provide probabilistic interpretations of such element-wise operators and discover surrogates with well-behaved derivatives for end-to-end training. Furthermore, we look at the spectral applicability of MaxExp and Gamma by studying Spectral Power Normalizations ({\em SPN}). We show that SPN on the autocorrelation/covariance matrix and the Heat Diffusion Process (HDP) on a graph Laplacian matrix are closely related,  
thus sharing their properties. Such a finding leads us to the culmination of our work, a fast spectral MaxExp which is a variant of HDP for covariances/autocorrelation matrices. 
We evaluate our 
ideas on fine-grained recognition, scene recognition, and material classification, as well as in few-shot learning and graph classification. 
} 
\end{abstract}
\end{justify}
\vspace{-0.3cm}
}

%

\maketitle


\IEEEpeerreviewmaketitle

\ifdefined\arxiv
\else
\begin{figure*}[t]
\vspace{-0.2cm}
\hspace{-1cm}
\centering
%
\begin{subfigure}[t]{0.545\linewidth}
\centering\includegraphics[trim=0 0 0 0, clip=true, width=9.0cm]{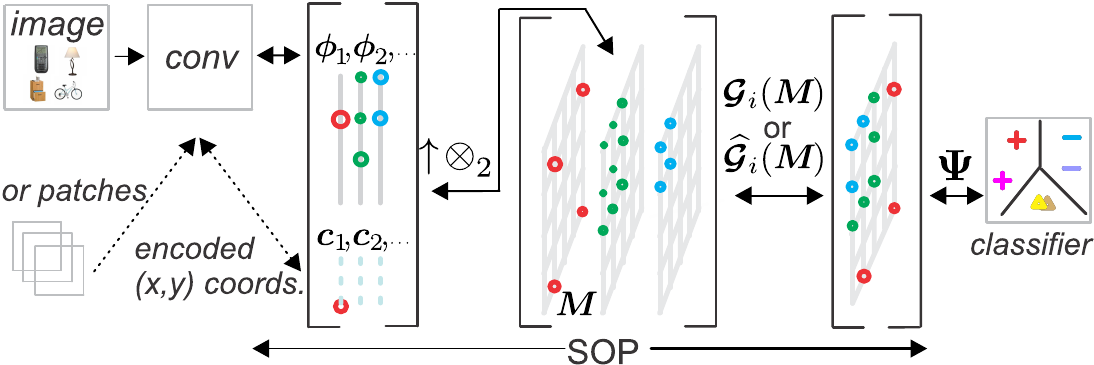}
\caption{\label{fig:princ1}}
\end{subfigure}
\begin{subfigure}[t]{0.45\linewidth}
\centering\includegraphics[trim=0 0 0 0, clip=true, width=9.0cm]{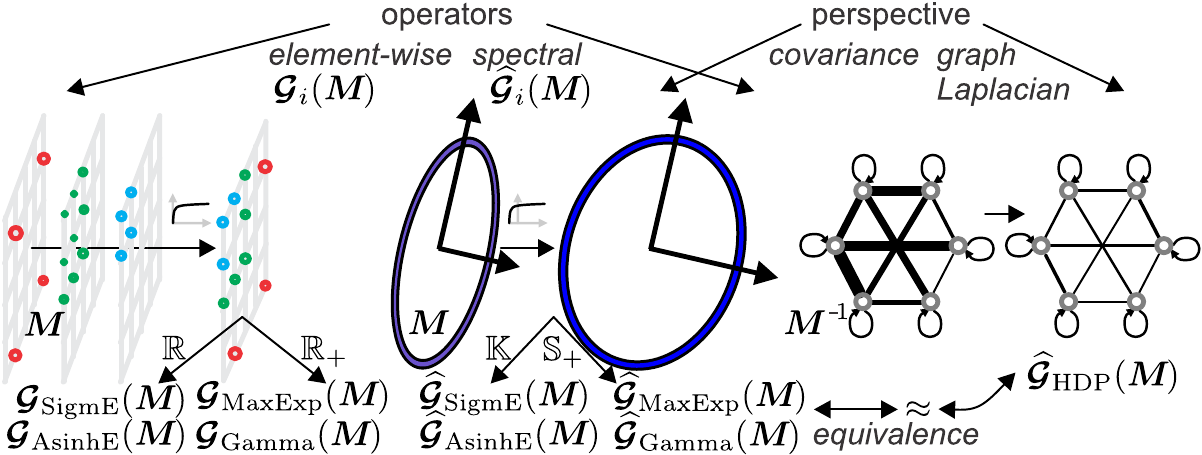}
\caption{\label{fig:whatwedo}}
\end{subfigure}
%
%
%
%
\vspace{-0.2cm}
\caption{\revisedd{Our end-to-end classification setting (Fig. \ref{fig:princ1}). We pass an image (or patches) via CNN and extract feature vectors $\vphi$ from its last conv. layer and augment them by encoded spatial coordinates $\vc$. We perform pooling on a second-order matrix $\mM$ by the Element-wise or Spectral Power Normalization $\tG_{\text{i}}$ or $\tGhat_{\text{i}}$, resp. Index $i$ indicates the operator \ie, MaxExp or Gamma. Figure \ref{fig:whatwedo} shows the taxonomy of pooling. We distinguish element-wise and spectral operators (which operate on autocorrelation/covariance or the graph Laplacian matrices). Element-wise and spectral operators can be further divided into working with non-negative  $\mathbb{R}_{+}$ and real $\mathbb{R}$ values, and $\mathbb{S}_{+}$/$\mathbb{S}_{++}$ and Krein ($\mathbb{K}$) spaces, resp. 
We derive so-called element-wise and spectral MaxExp operators, and we show that spectral MaxExp/Gamma (on autocorrelation/covariance) are approx. equivalent to the time-reversed Heat Diffusion Process (on the loopy graph Laplacian). 
MaxExp makes the underlying multivariate Gaussian closer to isotropic. Equivalently, the strong connections in the graph (thick edges) become weaker (thin lines) and more equalized.
%
}}\vspace{-0.3cm}
\label{fig:principle}
\end{figure*}
\fi

\section{Introduction}
\label{sec:intro}
Second-order statistics of data features are used in 
 object recognition, texture categorization, action representation, and human tracking \cite{tuzel_rc,porikli2006tracker,guo2013action,carreira_secord,me_tensor}. For example, the popular region covariance descriptors~\cite{tuzel_rc} compute a covariance matrix over multiple features extracted from image regions. 
Given Bag-of-Words histograms or local descriptors of an image, second-order co-occurrence pooling of such vectors captures correlations between pairs of features, and improves performance of semantic segmentation and visual recognition  compared to first-order methods~\cite{carreira_secord,me_tensor_tech_rep,me_tensor}. Extensions to higher-order descriptors~\cite{me_tensor_tech_rep,me_tensor,sparse_tensor_cvpr} improve results further. 

However, second- and higher-order statistics require robust aggregation/pooling mechanisms to obtain the best classification results \cite{carreira_secord,me_tensor_tech_rep,me_tensor}. Once the statistics are captured in the matrix form, they undergo  a non-linearity such as Power Normalization \cite{me_ATN} whose role is to reduce/boost contributions from frequent/infrequent visual stimuli in an image, respectively. The popular Bag-of-Words provide numerous insights into the role played by pooling during the aggregation step. The theoretical relation between Average and Max-pooling was studied in~\cite{boureau_midlevel}. A likelihood-inspired analysis of  pooling \cite{boureau_pooling}  led to a \emph{theoretical expectation of Max-pooling}. Max-pooling was recognized as a lower bound of the likelihood of \emph{`at least one particular visual word being present in an image'}~\cite{liu_sadefense} while Power Normalization was also applied to Fisher Kernels~\cite{perronnin_fisherimpr}. According to \cite{me_ATN}, pooling methods are closely related but \cite{me_ATN} does not study second-order pooling  or end-to-end training. Element-wise Power Normalization (PN) and Eigenvalue Power Normalization (EPN) were first applied to autocorrelation/covariance matrices and tensors  in \cite{me_tensor_tech_rep}.

In this paper, we 
revisit the above pooling methods 
in end-to-end setting and  interpret them in the context of second-order matrices. Firstly, we formulate a kernel which combines feature vectors collected from the last convolutional layer of ResNet-50 together with so-called spatial location vectors \cite{me_SCC,me_ATN,me_tensor_tech_rep} 
which contain Cartesian coordinates (spatial locations) of feature vectors in feature maps. We linearize such a kernel into  a second-order matrix 
to capture correlations of combined feature vectors. \revised{Next, we study the role of Power Normalizations in end-to-end setting. We show that  PNs have well-founded probabilistic interpretation in the context of second-order statistics. We propose PN surrogates with well-behaved derivatives for end-to-end training. Finally, we study PNs in the spectral domain, so-called Spectral Power Normalizations (SPN). We show that the Heat Diffusion Process (HDP) \cite{smola_graph} on a graph Laplacian is closely related to SPNs: HDP and  SPN play the same role for graph Laplacian and autocorrelation/covariance matrices, resp. As SPN and the HDP share properties, we propose a fast spectral MaxExp. 
To summarize:
\renewcommand{\labelenumi}{\roman{enumi}.}
\vspace{-0.3mm}
\begin{enumerate}[leftmargin=0.5cm]
\item We aggregate feature vectors extracted from CNNs and their spatial coordinates into a second-order matrix. 
\item We revisit Power Normalization functions, derive them for second-order representations and show how PNs emerge if we assume that features are drawn from the Bernoulli or Normal distributions. We also suggest PN surrogates with well-behaved derivatives for end-to-end training.
\item We show that Spectral Power Normalizations are in fact a time-reversed ($t\!<\!1$) Heat Diffusion Process, an important connection that explains the role of SPN. Thus, we propose a fast spectral MaxExp whose profile closely resembles HDP.
\item In addition to our standard fine-grained pipeline, we develop second-order relational representations for few-shot learning and we even consider graph classification.
\end{enumerate}
}


Figures \ref{fig:princ1} and \ref{fig:princ2} show our classification pipeline (we use ResNet-50 pre-trained on ImageNet) and our few-shot learning Second-order Similarity Learning Network ({\em SoSN}). We experiment on ImageNet, Flower102, MIT67, FMD and Food-101 (classification setting), \textit{mini}ImageNet, Flower102, Food-101 and Open MIC (few-shot setting), and MUTAG, PTC, PROTEINS, NCI1, COLLAB, REDDIT-BINARY/MULTI-5K (graph classification).

\ifdefined\arxiv
\begin{figure*}[t]
\centering
%
\centering\includegraphics[trim=0 0 0 0, clip=true, width=14.0cm]{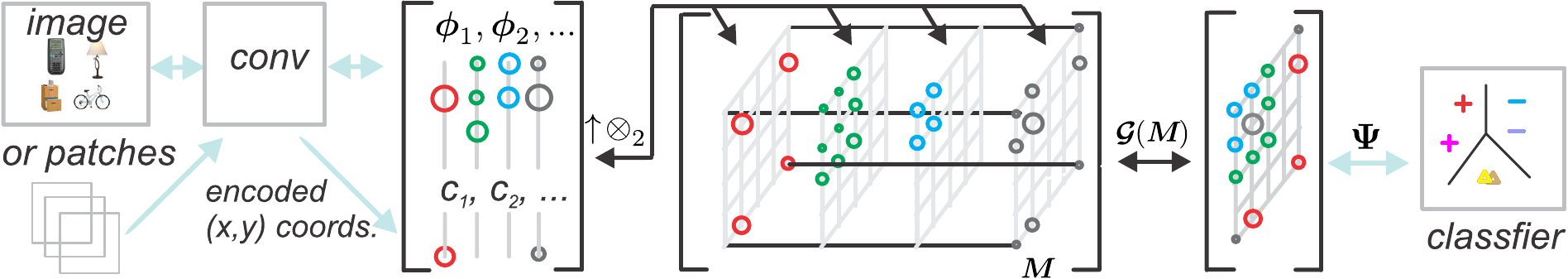}
%
%
\vspace{-0.2cm}
\caption{Our end-to-end pipeline. We pass an image (or patches) to CNN and extract feature vectors $\vphi$ from its last conv. layer and augment them by encoded spatial coordinates $\vc$. We perform pooling on second-order matrix $\mM$ by the Power Normalization function $\tG$.}\vspace{-0.3cm}
\label{fig:principle}
\end{figure*}
\fi

%
%
%
%

\revisedd{We explore Power Normalizing functions for second-order image \cite{koniusz2018deeper} and graph classification, and few-shot learning \cite{sosn}.

\begin{tcolorbox}[width=1.0\linewidth, colframe=blackish,colback=beaublue, boxsep=0mm, arc=3mm, left=1mm, right=1mm, top=1mm, bottom=1mm]
Motivated by the need to discuss PNs (\ie, see tutorial \cite{secordcv_tutorial_2019}), we propose a taxonomy in Fig. \ref{fig:whatwedo} and analyze (i) element-wise operators (fast but less robust) and (ii) spectral operators (exploit correlations between features but are slower). Second level branches of Fig. \ref{fig:whatwedo} consider non-negative  $\mathbb{R}_{+}$ and real $\mathbb{R}$ inputs to PN, and $\mathbb{S}_{+}$/$\mathbb{S}_{++}$ and Krein ($\mathbb{K}$) input spaces of SPN.
\end{tcolorbox}
\vspace{-0.10cm}

To interpret the statistical meaning of PNs \wrt inputs, 
we show that element-wise operators 
similar to Gamma \cite{me_ATN,me_tensor_tech_rep} emerge from statistical models assuming features being drawn from an i.i.d. Bernoulli or Normal distribution. This encourages us to look at PN as a wider family of functions (\cf Gamma/square root).  As the i.i.d. assumption in element-wise operators is limiting, 
we consider SPNs  \cite{me_tensor_tech_rep} due to their feature decorrelating properties. We show that SPNs are in fact an equivalent of  the Heat Diffusion Process \cite{smola_graph}, specifically  a time-reversed ($t\!<\!1$) HDP. As Fig. \ref{fig:whatwedo} shows, SPN and HDP operate on autocorrelation/covariance and the graph Laplacian matrix respectively (different theoretical perspectives). Finally, to tackle the speed bottleneck, we propose a fast SVD-free back-propagation through spectral MaxExp. 

\begin{tcolorbox}[width=1.0\linewidth, colframe=blackish, colback=beaublue, boxsep=0mm, arc=3mm, left=1mm, right=1mm, right=1mm, top=1mm, bottom=1mm]
Importantly, on few-shot learning, we find that SPNs benefit fine-tuning but not training from scratch. Second-order matrices alone cannot reduce the diffusion of signal between correlated features during fine-tuning, a cause of so-called catastrophic forgetting. 
However, SPNs limit the diffusion between features of second-order matrices and thus reduce catastrophic forgetting.
\end{tcolorbox}
\vspace{-0.1cm}
}

Sections \ref{sec:related_work} and \ref{sec:background} provide the related work, background and notations. Section \ref{sec:problem} and \ref{sec:spns_main} present our mathematical analysis. Sections \ref{sec:pipes}, \ref{sec:expts} and \ref{sec:conclude} present the pipeline, experiments and conclusions.

\vspace{-0.2cm}
\section{Related Work}
\label{sec:related_work}

Below we introduce  Region Covariance Descriptors (RCD) which are perhaps the oldest second-order descriptors \cite{tuzel_rc,tuzel2008detection}. 

\vspace{0.00cm}
\noindent{\textbf{Region Covariance Descriptors (RCD).}} 
RCDs typically capture co-occurrences of luminance, first- and/or second-order derivatives of texture patterns, 
and have been applied to tracking \cite{porikli2006tracker}, semantic segmentation \cite{carreira_secord} and object category recognition \cite{me_tensor_tech_rep,me_tensor}. 
%
As RCDs typically require a Non-Euclidean distance to compare positive (semi-)definite RCD datapoints, we list below popular choices.

\vspace{0.05cm}
\noindent{\textbf{Non-Euclidean distances.}} 
The distance between two positive definite datapoints is typically measured according to the Riemannian geometry while Power-Euclidean distances \cite{dryden_powereuclid} extend to positive semi-definite datapoints. 
In particular, Affine-Invariant Riemannian Metric \cite{PEN06,bhatia_pdm}, KL-Divergence Metric ({\em KLDM}) \cite{wang_jeffreys}, Jensen-Bregman LogDet Divergence ({\em JBLD}) \cite{anoop_logdet} and Log-Euclidean ({\em LogE}) \cite{arsigny2006log} have been used in diffusion imaging, RCD-based methods, 
%
dictionary and metric learning  \cite{mehrtash_dict_manifold,Roy_CVPR_2018,Kumar_CVPR_2018}. 


\revised{
\vspace{0.05cm}
\noindent{\textbf{Second-order Pooling in CNNs.}} 
%
%
%
%
There has  been a revived interest in co-occurrence patterns in CNN setting. 
Methods \cite{bilinear_finegrained,bilinear_pami,deep_cooc} fuse CNN streams via the outer product for the fine-grained image recognition. Approach \cite{face_cooc} uses co-occurrences of CNN feature vectors and facial attribute vectors for face recognition. 


We note that the Log-Euclidean distance and Power Normalization have recently been implemented in the CNN setting \cite{sminchisescu_matrix,vangol_riem_net,peihua_fast,lin2017improved} for the purpose of image classification. 
For instance, approaches \cite{peihua_fast,lin2017improved} build on so-called Eigenvalue Power Normalization \revisedd{as introduced by} \cite{me_tensor_tech_rep,me_tensor}, however, they extend it to end-to-end CNN setting with 
back-propagation via SVD explored in \cite{sminchisescu_matrix} which is both slow and unstable due to so-called non-simple eigenvalues. To this end, recent advances in \cite{peihua_fast,lin2017improved} propose so-called Newton-Schulz iterations to compute the square root of matrix (a special case of  Gamma \cite{me_ATN,me_tensor_tech_rep}). Authors of \cite{peihua_fast,lin2017improved} motivate the use of spectral Gamma with the notions of burstiness and whitening, something considered 
first in the context of Eigenvalue Power Normalization in early works \cite{me_tensor_tech_rep,me_tensor}. 
 A visualization approach into bilinear models including $\alpha$-pooling, a form of Power Normalization, is proposed in \cite{the_whole}.
%


}


\vspace{0.05cm}
\noindent{\textbf{Power Normalizations.}} 
Image representations suffer from the so-called burstiness which is `{\em the property that a given visual element appears more times in an image than a statistically independent model would predict}' \cite{jegou_bursts}. Power Normalization~\cite{boughorbel_intersect, perronnin_fisherimpr, jegou_bursts}  suppresses the burstiness which has been studied/evaluated in the context of Bag-of-Words \cite{me_ATN,me_tensor}. 
The theoretical study of Average and Max-pooling \cite{boureau_midlevel,boureau_pooling} highlighted their statistical models and a connection to Max-pooling. 
A relationship between the likelihood of `\emph{at least one particular visual word being present in an image}' and Max-pooling was studied in \cite{liu_sadefense}. Survey \cite{me_ATN} shows that all  Power Normalization functions are closely related.

\revised{
We show that MaxExp for element-wise co-occurrence pooling emerges from the Multinomial modeling while authors of \cite{boureau_midlevel,boureau_pooling} use a Binomial setting for first-order signals. To paraphrase, we show why it is theoretically meaningful to use MaxExp with co-occurrences, and by connecting  MaxExp to HDP in the spectral setting, we show why MaxExp can work with the spectrum.}


\vspace{0.05cm}
\noindent{\textbf{One- and Few-shot Learning}}, motivated by the human ability to learn from few samples, has been widely studied in both shallow \cite{miller_one_example,Li9596,NIPS2004_2576} 
 and deep learning scenarios \cite{vinyals2016matching,snell2017prototypical,finn2017model,sung2017learning}. 



%
%
Matching Network \cite{vinyals2016matching}, Prototypical Networks \cite{snell2017prototypical}, Model-Agnostic Meta-Learning (MAML) \cite{finn2017model} and Relation Net \cite{sung2017learning}  learn similarity between pairs of images rather than  class concepts. 
%
%
%
%
Our few-shot learning pipeline is similar to Relation Net \cite{sung2017learning} which uses first-order representations. However, we build on few-shot learning with second-order representations and Power Normalization as proposed in \cite{sosn}. Moreover, we consider element-wise and spectral operators, and we show a theoretical analysis that PN is especially beneficial in relation learning. Finally, we note that second-order representations are gaining momentum in few-shot learning \cite{wertheimer2019few,zhang2019few,Zhang_2020_ACCV} which use localization mechanisms while approach \cite{christian_subs} uses subspace-based class-wise prototypes.

\vspace{-0.3cm}
\section{Background}
\label{sec:background}
\vspace{-0.1cm}

Below we review our notations, the background on kernel linearizations and the Power Normalization family. 

\vspace{-0.3cm}
\subsection{Notations}
\label{sec:notations}
%
Let $\vx\in\mbr{d}$ be a $d$-dimensional feature vector. 
$\idx{N}$ stands for the index set $\set{1, 2,\cdots,N}$. 
%
%
The spaces of symmetric positive semi-definite and definite matrices are $\semipd{d}$ and $\spd{d}$. $\sym(\mX)\!=\!\frac{1}{2}(\mX\!+\!\mX^T\!)$. 
A vector with all coefficients equal one is denoted by $\vOnes$, $\vj_k$ is a vector of all zeros except for the $k$-th coefficient which is equal one, and $\mJ_{kl}$ is a matrix of all zeros with a value of one at the position $(k,l)$. 
Moreover, $\odot$ is the Hadamard product, 
\revisedd{ $\vectorise(\mX)$ vectorizes matrix $\mX$ in analogy to $\mX(:)$ in Matlab and `$\dagger$' is the  Moore-Penrose pseudoinverse.} 
We use the Matlab notation $\vv\!=\![\text{begin}\!:\!\text{step}\!:\!\text{end}]$ to generate a  vector $\vv$ with elements starting as {\em begin}, ending as {\em end}, with stepping {\em step}. Operator `$;$' in $[\vx; \vy]$ is the concat. of vectors $\vx$ and $\vy$ (or scalars). 

\vspace{-0.3cm}
\subsection{Autocorrelation matrix}
\label{sec:som}
Below we show that autocorrelation (second-order) matrices emerge from a linearization of sum of Polynomial kernels.
\begin{proposition}
\label{pr:linearize}
Let $\mPhi_A\equiv\{\vphi_n\}_{n\in\tNnb_{\!A}}$, $\mPhi_B\!\equiv\{\vphi^*_n\}_{n\in\tNnb_{\!B}}$ be datapoints from two images $\Pi_A$ and $\Pi_B$, and $N\!=\!|\tNnb_{\!A}|$ and $N^*\!\!=\!|\tNnb_{\!B}|$ be the numbers of data vectors \eg, obtained from the last convolutional feature map of CNN for images $\Pi_A$ and $\Pi_B$. Autocorrelation feature maps result from a linearization of the sum of Polynomial kernels of degree $2$:
\comment{
\begin{align}
& \!\!\!\text{\fontsize{8}{9}\selectfont$K(\mPhi_A, \mPhi_B)\!=\!\frac{1}{NN^*\!}\!\!\sum\limits_{
\!n\in \tNnb_{\!A}}\!\sum\limits_{\!n'\!\in \tNnb_{\!B}}\!\!\!\left<\vphi_n, \vphi^*_{n'}\right>^2\!$}\nonumber\\
&\quad\quad\quad\quad\text{\fontsize{8}{9}\selectfont$\,=\!\big<\vectorise\left(\mM(\mPhi_A)\right)\!, \vectorise\left(\mM(\mPhi_B)\right)\!\big>$}\nonumber\\
&\quad\quad\quad\quad\text{\fontsize{8}{9}\selectfont$\,=\!\Big\langle\vectorise\Big(\frac{1}{N}\!\!\sum\limits_{
n\in \tNnb_{\!A}}\!\!{\vphi_n\vphi_n^T}\Big), \vectorise\Big(\frac{1}{N^*\!}\!\!\!\sum\limits_{
n'\in \tNnb_{\!B}}\!\!{\vphi^*_{n'}{\vphi^*_{n'}}^{\!\!\!T}}\Big)\Big\rangle$}.\label{eq:hok1}\!\!\!\!
\end{align}
}
\begin{align}
& \!\!\!\text{\fontsize{8}{9}\selectfont$K(\mPhi_A, \mPhi_B)\!=\!\frac{1}{NN^*\!}\!\!\sum\limits_{
\!\!\!n\in \tNnb_{\!A}}\!\!\!\!\!\!\sum\limits_{\;\;\;\;n'\!\in \tNnb_{\!B}}\!\!\!\!\!\!\left<\vphi_n, \vphi^*_{n'}\right>^2\!\!\!=\!\big<\vectorise\left(\mM(\mPhi_A)\right)\!, \vectorise\left(\mM(\mPhi_B)\right)\!\big>$}\nonumber\\
&\quad\quad\quad\quad\text{\fontsize{8}{9}\selectfont$\,=\!\Big\langle\vectorise\Big(\frac{1}{N}\!\!\sum\limits_{
n\in \tNnb_{\!A}}\!\!{\vphi_n\vphi_n^T}\Big), \vectorise\Big(\frac{1}{N^*\!}\!\!\!\sum\limits_{
n'\in \tNnb_{\!B}}\!\!{\vphi^*_{n'}{\vphi^*_{n'}}^{\!\!\!T}}\Big)\Big\rangle$}.\label{eq:hok1}\!\!\!\!
\end{align}
\end{proposition}
\begin{proof}
See \cite{koniusz2017domain} for the details of such an expansion.
\end{proof}

\revised{Thus, we obtain the following (kernel) feature map 
 \revisedd{$\mM(\mPhi)$ on features $\mPhi$} which coincides with the autocorrelation matrix:
\begin{align}
& 
\mM(\mPhi)\!=\!\frac{1}{N}\sum_{n\in\mathcal{N}}\!\vphi_n\vphi_n^T.\label{eq:hok3}
\end{align}
\vspace{-0.2cm}
}

\revisedd{\noindent{For}  simplicity of notation, we drop $\mPhi\equiv\{\vphi_n\}_{n\in\mathcal{N}}$ from $\mM(\mPhi)$ where possible, that is, we often write $\mM$ rather than $\mM(\mPhi)$.}

\comment{
\begin{remark}
In what follows, we will use second-order matrices obtained from the above expansion for $r\!=\!2$, that is:
\begin{align}
& 
\!\!\!\!\frac{1}{NN^*\!}\!\!\sum\limits_{
n\in \tNnb_{\!A}}\sum\limits_{n'\!\in \tNnb_{\!B}}\!\!\!\left<\vphi_n, \vphi^*_{n'}\right>^2\!\!=\!
\Big\langle\frac{1}{N}\sum\limits_{
n\in \tNnb_{\!A}}{\vphi_n\vphi_n^T}, \frac{1}{N^*\!}\sum\limits_{
n\in \tNnb_{\!B}}{\vphi^*_{n'}{\vphi^*_{n'}}^{\!\!\!T}}\Big\rangle.\!\!\label{eq:hok2}
\end{align}
%
%
%
\revised{Thus, we obtain the following (kernel) feature map
 $\mM$ which coincides with the autocorrelation matrix:
\begin{align}
& 
\mM\!=\!\frac{1}{N}\sum_{n\in\mathcal{N}}\!\vphi_n\vphi_n^T.\label{eq:hok3}
\end{align}
Note that (kernel) feature maps are not conv. maps. They are two separate notions that happen to share the same name.
}
\end{remark}
}
\vspace{-0.3cm}
\subsection{Power Normalization Family (first-order variants)}
\label{sec:pn}

\revisedd{Pooling is an aggregation step of feature vectors $\mPhi$ that produces their signature used in for instance training an SVM. We study pooling of second-order representations to preserve second-order statistics of input vectors. For clarity, we firstly introduce highly-related first-order PNs, that is MaxExp and Gamma operators. Traditionally, PN is a function $g([0,\iota])\!\rightarrow\![0,\nu]$ such that $g(0)\!=\!0$, $g(\iota)\!=\!\nu$ (\ie, $\iota\!=\!\nu\!=\!1$), $g(p)$ is monotonically non-decreasing on $0\!\leq\!p\!\leq\!\iota$ and the slope of $g(p)$ rises fast/slow for $p\!\approx\!0$ and slow/fast for $p\!\approx\!\iota$, resp. Fig. \ref{fig:pow4} illustrates such properties.
}

\vspace{0.05cm}
\revisedd{\noindent{\textbf{MaxExp (first-order).}}}  Drawing features from the Bernoulli distribution under the i.i.d. assumption \cite{boureau_pooling}  leads to so-called {\em Theoretical Expectation of Max-pooling} ({\em MaxExp}) operator \cite{me_ATN} related to Max-pooling \cite{boureau_midlevel} and {\em Gamma} \cite{me_ATN}. The following proposition formalizes this process. 

\vspace{-0.1cm}
\begin{proposition}\label{prop:maxexp}
Assume a vector $\vphi\!\in\!\{0,1\}^{N}$ which stores $N$ outcomes of drawing from Bernoulli distribution under the i.i.d. assumption for which the probability $p$ of an event $(\phi_{n}\!=\!1)$ and $1\!-\!p$ for $(\phi_{n}\!=\!0)$ can be estimated as an expected value \eg, $p\!=\!\avg_n\phi_n$. Then the probability of at least one positive event in $\vphi$ from $N$ trials becomes $\psi\!=\!1\!-\!(1\!-\!p)^{N}$.
\comment{
\vspace{-0.2cm}
\begin{equation}
\label{eq:my_maxexp1}
\psi\!=\!1\!-\!(1\!-\!p)^{N}.
\end{equation}
}
\end{proposition}
\begin{proof}
\label{pr:maxexp}
\revised{The proof can be found in \cite{koniusz2018deeper}.}
\comment{The proof follows the school syllabus for a fair coin toss. The probability of all $N$ outcomes to be $\{(\phi_{1}\!=\!0),\cdots,(\phi_{N}\!=\!0)\}$ amounts to $(1\!-\!p)^N$. 
The probability of at least one positive outcome 
$(\phi_{n}\!=\!1)$ amounts to applying the logical `or' $\{(\phi_{1}\!=\!1)\,|\cdots|\,(\phi_{N}\!=\!1)\}$ and leads to:
\vspace{-0.2cm}
\begin{equation}\label{eq:my_maxexp2}
1\!-\!(1\!-\!p)^{N}=\,\sum_{n=1}^{N} \binom{N}{n} p^n(1\!-\!p)^{N-n}.
\vspace{-0.9cm}
\end{equation}}
\end{proof}
%
%

\vspace{-0.1cm}
\noindent{A} practical implementation of this pooling strategy \cite{me_ATN} is given by $\psi_k\!=\!1\!-\!(1\!-\!\avg_n\phi_{kn})^{\eta}$, where $0\!<\!\eta\!\approx\!N$ is an adjustable parameter and $\phi_{kn}$ is a $k$-th feature of an $n$-th feature vector \eg, as defined in Prop. \ref{pr:linearize}, 
normalized to range 0--1.
%

\vspace{0.05cm}
\revisedd{\noindent{\textbf{Gamma (first-order).}}} 
\label{re:pn}
It was shown in \cite{me_ATN} that Power Normalization ({Gamma}) given by $\psi_k\!=\!(\avg_n\phi_{kn})^\gamma$, where $0\!<\!\gamma\!\leq\!1$ is an adjustable parameter, is in fact an approximation of MaxExp.
\vspace{-0.2cm}
\section{Problem Formulation}
\label{sec:problem}

\revisedd{The main goal of this paper is a theoretical study of Power Normalizations for second-order representations, their interpretation and theoretical connections. Fig. \ref{fig:whatwedo} introduces the taxonomy of operators we follow. We choose  existing PN operators Gamma and MaxExp (other operators are typically their modifications). We look at (i) element-wise PN (fast but suboptimal) and (ii) SPN (slow but exploiting feat. correlations). We then consider SPNs on autocorr./covariances and their connection to graph Laplacians.

{\begin{tcolorbox}[width=1.0\linewidth, colframe=blackish,colback=beaublue, boxsep=0mm, arc=3mm, left=1mm, right=1mm, right=1mm, top=1mm, bottom=1mm]
We show that the Heat Diffusion Process on graphs and SPN are two sides of the same mathematical process. Thus, SPNs reverse the diffusion of signal in autocorrelation/covarience matrices rather than just reduce the burstiness of features. The culmination of our work is MaxExp(F), the fast spectral MaxExp, (and a fast approx. HDP for completeness) whose runtime scales sublinearly \wrt its parameter. MaxExp(F) rivals the approx. matrix
square root via Netwon-Schulz iter. \cite{lin2017improved,peihua_fast}.
\end{tcolorbox}}
\vspace{-0.4cm}


\vspace{0.3cm}
Sec. \ref{sec:cooc} proposes a spatially augmented autocorrelation matrix (Fig. \ref{fig:princ1}) 
that can be seen as introducing spatially localized nodes into a graph. 
Sec. \ref{sec:well_mot} explains why MaxExp is applicable to co-occurrences. While MaxExp for first-order signals emerges from Binomial modeling of features, 
for co-occurrences the same operator emerges from Multinomial modeling. Following the taxonomy (Fig. \ref{fig:whatwedo}), in Sec. \ref{sec:pool_der} we generalize MaxExp/Gamma \cite{me_ATN,me_tensor} (work on $\mathbb{R}_{+}$) to Logistic a.k.a. Sigmoid ({\em SigmE}) and the Arcsin hyperbolic ({\em AsinhE}) functions (work on $\mathbb{R}$). SigmE emerges from modeling Normal distr. Both functions extend readily to Krein spaces (Table \ref{tab:smd2}). 
Sec. \ref{sec:spns} introduces spectral pooling. 
}

\ifdefined\arxiv
\newcommand{\PowH}{3.0cm}
\newcommand{\PowHB}{2.875cm}
\newcommand{\PowW}{3.65cm}
\else
\newcommand{\PowH}{3.6cm}
\newcommand{\PowHB}{3.4cm}
\newcommand{\PowW}{4.5cm}
\fi

\begin{figure*}[t]
\centering
\hspace{-0.3cm}
\begin{subfigure}[t]{0.24\linewidth}
\centering\includegraphics[trim=0 0 0 0, clip=true, height=\PowH, width=\PowW]{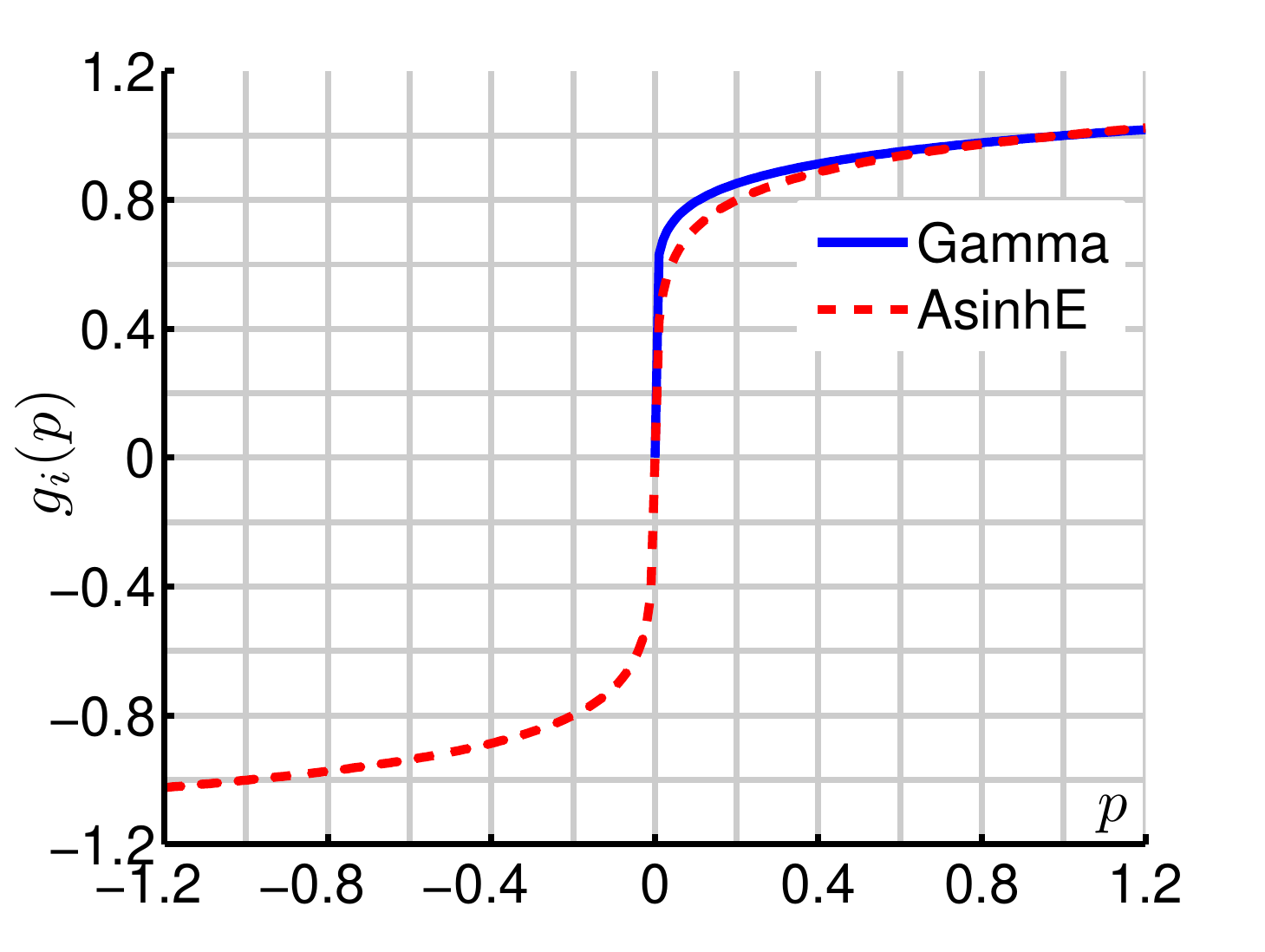}\vspace{-0.2cm}
\caption{\label{fig:pow1}}
\end{subfigure}
\begin{subfigure}[t]{0.24\linewidth}
\centering\includegraphics[trim=0 -18 0 15, clip=true, height=\PowHB, width=\PowW]{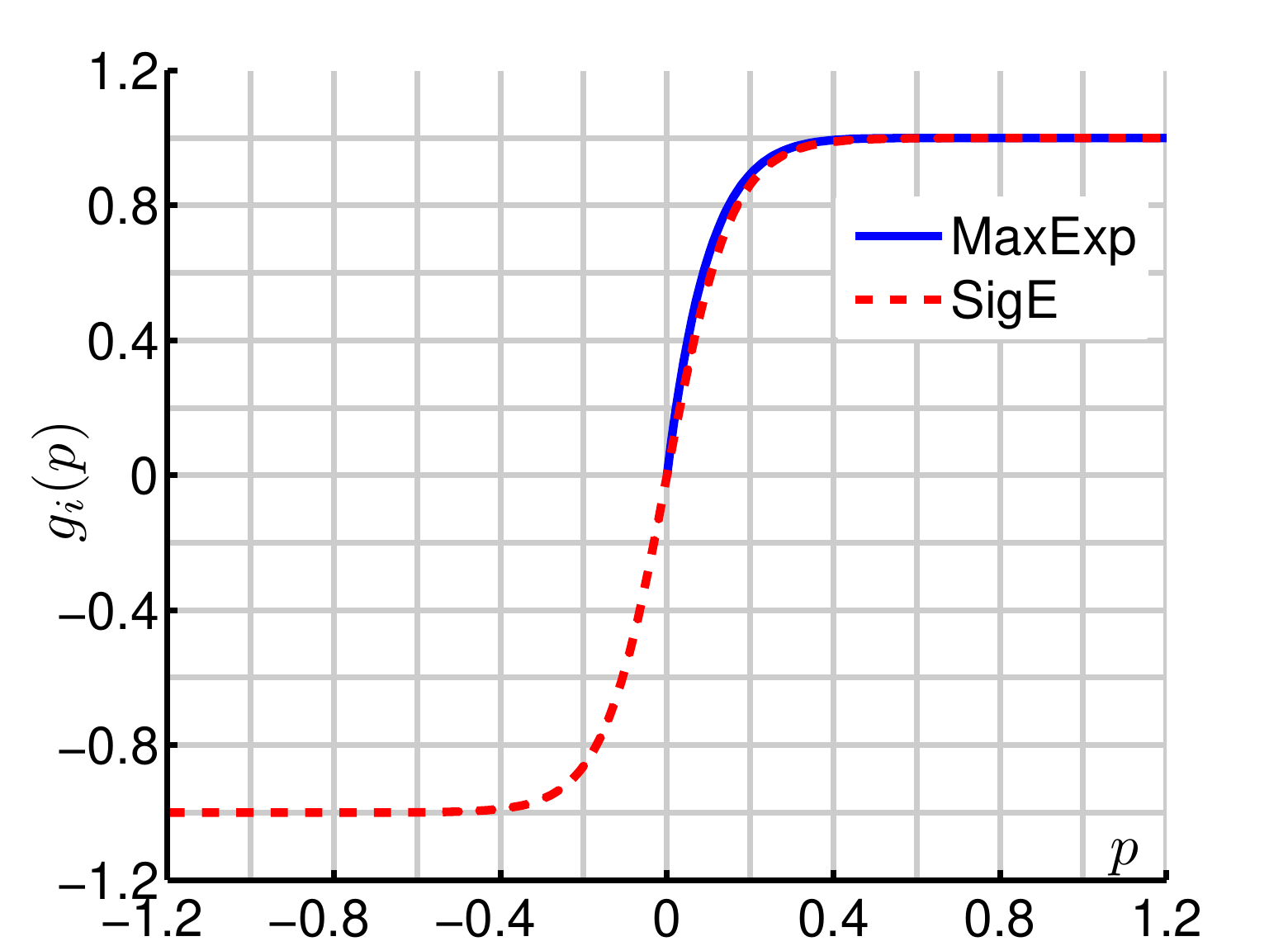}\vspace{-0.2cm}
\caption{\label{fig:pow2}}
\end{subfigure}
\begin{subfigure}[t]{0.24\linewidth}
\centering\includegraphics[trim=0 0 0 0, clip=true, height=\PowH, width=\PowW]{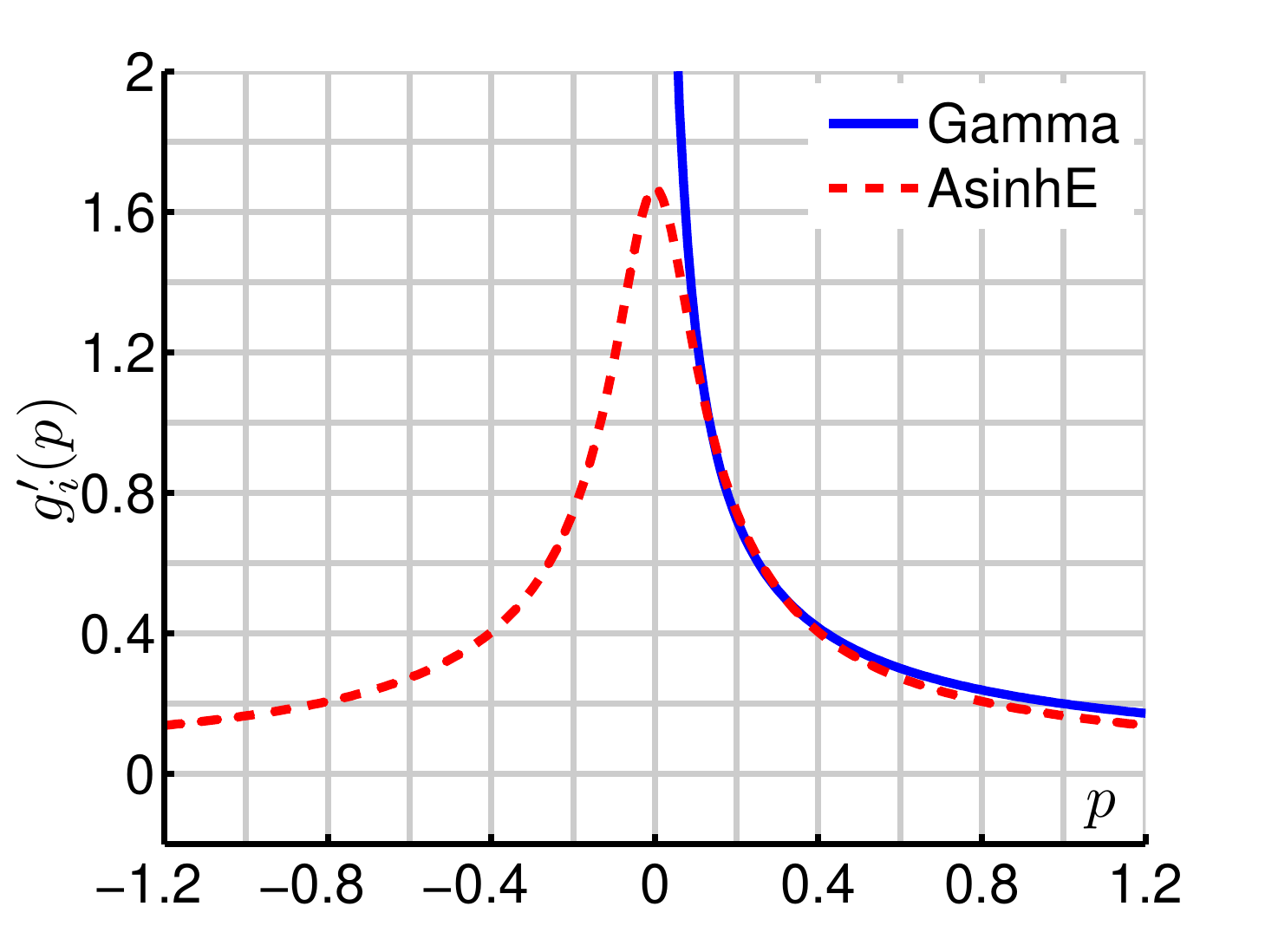}\vspace{-0.2cm}
\caption{\label{fig:pow3}}
\end{subfigure}
\begin{subfigure}[t]{0.24\linewidth}
\centering\includegraphics[trim=0 0 0 0, clip=true, height=\PowH, width=\PowW]{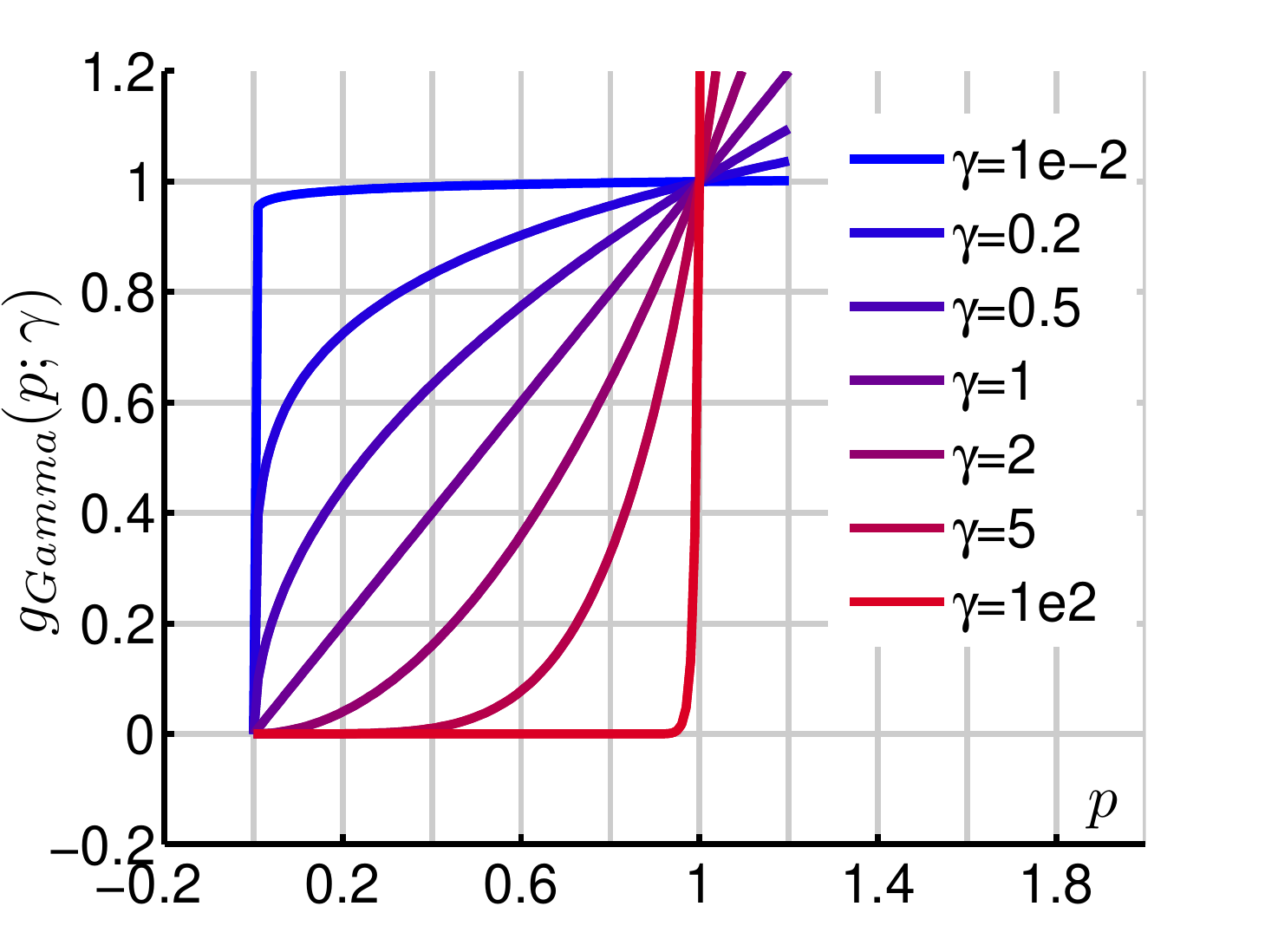}\vspace{-0.2cm}
\caption{\label{fig:pow4}}
\end{subfigure}
%
\caption{Gamma, AsinhE, MaxExp and SigmE are illustrated in Figures \ref{fig:pow1} and \ref{fig:pow2} while derivatives of Gamma and AsinhE are shown in Figure \ref{fig:pow3}. Lastly, Gamma for several values of $\gamma$ is shown in Figure \ref{fig:pow4} from which its similarity to MaxExp in range 0--1 is clear.}
\vspace{-0.3cm}
\label{fig:power-norms}
\end{figure*}

\subsection{Augmented autocorrelation matrix}
\label{sec:cooc}

\revisedd{The autocorrelation matrix defined in Section \ref{sec:som} is perfectly applicable to considerations in our paper. We enhance autocorrelation matrices by two steps detailed below for the best performance but are secondary to the main analysis of Power Normalizations.}
%


\vspace{0.05cm}
\noindent{\textbf{$\boldsymbol{\beta}$-centering.}} As in Eq. \eqref{eq:hok3}, let $\mPhi\equiv\{\vphi_n\}_{n\in\mathcal{N}}$ but $\vphi_n$ are rectified so that $\vphi_n\!\geq\!0$. Subsequently, $\beta$-centering \wrt data mean $\vmu\!=\!\avg_{n\in\tNnb}\vphi_n$ is obtained as $\vphi^{(\beta)}_n\!\!=\!\vphi_n\!-\!\beta\vmu$ 
for $0\!\leq\!\beta\!\leq\!1$. \revisedd{For brevity, we drop superscript ${(\beta)}$ from $\vphi^{(\beta)}_n$. 
} 

The role of $\beta$-centering is to address anti-occurrences \ie, some Bag-of-Word models use so-called negative visual words, the evidence that a given visual stimulus is missing from an image. Authors of \cite{negoccur} define it as `{\em the negative evidence \ie, a visual word that is mutually missing in two descriptions being compared}'. Lack of certain visual stimuli correlates with some visual classes \eg, lack of the sky may imply an indoor scene. Thus, we offset vectors $\vphi$ by  $\beta\vmu$ ($\vmu$ computed per-image) so that the positive/negative values yield correlations/anti-correlations. 

\vspace{0.05cm}
\noindent{\textbf{Positional embedding.}} 
\revised{As in Prop. \ref{pr:linearize}, let 
$\mPhi_A$ and $\mPhi_B$ be obtained from the last conv. feature maps of CNN for images $\Pi_A$ and $\Pi_B$. 
Then, Cartesian (spatial) coordinates \cite{me_SCC,me_tensor} at which feature vectors are extracted along the channel mode from conv. feature maps ($\mPhi_A$ and $\mPhi_B$) are normalized to range $[0,1]$ yielding $(x,y)$ and $(x^*,y^*)$ which are embedded into a non-linear Hilbert space. Firstly, the normalization is performed as $x_n\!=\!x'_n/(W\!-\!1)$ and $y_n\!=\!y'_n/(H\!-\!1)$ \wrt the width $W$ and height $H$ of conv. feature maps, where $x'_n$ and $y'_n$ are Cartesian coordinates. 
%
Then, we form the following sum kernel and its linearization: 
\begin{align}
\label{eq:encode_sc}
& \text{\fontsize{8}{9}\selectfont $K([x_n;y_n], [x_{n'}^*;y_{n'}^*])\!=\!\alpha^2G_{\sigma}(x_n\!-\!x_{n'}^*)\!+\!\alpha^2G_{\sigma}(y_n\!-\!y_{n'}^*)\approx$}\nonumber\\ 
& \text{\fontsize{8}{9}\selectfont $\left<\alpha\vvarphi(x_n,\vzeta), \alpha\vvarphi(x_{n'}^*\!,\vzeta)\right>\!+\!\left<\alpha\vvarphi(y_n,\vzeta), \alpha\vvarphi(y_{n'}^*\!,\vzeta)\right>$},
\end{align}
where $\vvarphi(x)$ and $\vvarphi(x^*\!)$ are feature maps linearizing an RBF kernel $G_{\sigma}(x\!-\!x^*\!)\!=\!\exp(-\!(x\!-\!x^*\!)^2/{2\sigma^2})\!\approx\!\left<\sqrt{c}\vvarphi(x), \sqrt{c}\vvarphi(x^*\!)\right>$: 
\vspace{-0.1cm}
\begin{align}
&\!\!\!\!\!\!\!\text{\fontsize{8}{9}\selectfont $\vvarphi(x)=\left[{G}_{\sigma/\sqrt{2}}(x-\zeta_1),\cdots,{G}_{\sigma/\sqrt{2}}(x-\zeta_Z)\right]^T$}\!\!\!\!,\label{eq:gauss_lin2}
\end{align}
where  $\sigma\!>\!0$ is the RBF bandwidth, const. $c\!=\!1$. 
For $Z$ pivots $\vzeta\!=\![\zeta_1;\cdots;\zeta_Z]$, we use $Z$ in range 3--10 and equally spaced intervals  $\vzeta\!=\![-0.2:1.4/(Z\!-\!1):1.2]$ to encode 
$x_n$ and $y_n$}. \revisedd{The derivation and the choice of pivots are explained in Appendix \hyperref[{app:kern_linear}]{L}.}

The above formulation extends to the aggregation over patches 
extracted 
from images as shown in Figure \ref{fig:princ1}. 
We form vectors $\vphibar_n\!=\![\vphi_n; \vc_n]$  augmented by encoded spatial coordinates $\vc_n\!=\![\alpha\vvarphi(x_n,\vzeta); \alpha\vvarphi(y_n,\vzeta)]$. Thus, we define the total length of $\vc_n$ as $Z'\!\!=\!2Z$. Then, we  define $\mPhibar\equiv\{\vphibar_n\}_{n\in\mathcal{N}}$.  
Combining augmented vectors with the Prop. \ref{pr:linearize} and Eq. \eqref{eq:hok3} yields $\mM(\mPhibar)$ (\cf $\mM(\mPhi)$). 
\comment{
\vspace{-0.4cm}
\begin{align}
%
& 
\mM(\mPhibar)\!=\!\frac{1}{N}\sum\limits_{n\in\mathcal{N}}\vphibar_n\vphibar_n^T.\label{eq:pn_simple1}
\end{align}
}


\revisedd{We note that the autocorrelation matrix is sometimes called as co-occurrence matrix in the literature. If $\{\vphibar_{n}, n\!\in\!\mathcal{N}\}$ contains only binary features $\{0,1\}$ then the autocorrelation matrix is a form of co-occurrence matrix normalized by $N$ that captures counts $M_{kl}\!=\!\frac{1}{N}\sum_{n\in\mathcal{N}}\bar{\phi}_{kn}\!\cdot\!\bar{\phi}_{ln}$.}

\revisedd{Our pipelines apply element-wise or spectral pooling $\mPsi\!=\!\mygthreee{i}{\mM(\mPhibar);\cdot\,}$ or $\mPsi\!=\!\mygthreeehat{i}{\mM(\mPhibar);\cdot\,}$, resp., where $\mM(\mPhibar)$ is an autocorrelation/covariance matrix defined above, $\mygthreee{i}{\mM(\mPhibar);\cdot\,}$ and $\mygthreeehat{i}{\mM(\mPhibar);\cdot\,}$ are element-wise and spectral Power Normalizations on $\mM(\mPhibar)$ with some parameter `$\cdot$' and $i$ is replaced by a specific name of PN. Finally, $\mPsi$ is the resulting PN feature map. For brevity, we often drop $\mPsi$, $i$ and `$\cdot$', and write $\mM$ (\cf $\mM(\mPhibar)$).}

\vspace{-0.3cm}
\subsection{Well-motivated Pooling Approaches (second-order element-wise variants on non-negative features: $M_{kl}\!\geq\!0$)}
\label{sec:well_mot}
\revisedd{Following the first branch of the taxonomy in Fig. \ref{fig:whatwedo} (element-wise operators), we extend PN operators MaxExp and Gamma (first-order variants) introduced in Section \ref{sec:pn} to their second-order element-wise counterparts (pooling acts on elements of autocorrelation matrix). As such models lack any previous analysis, we demonstrate that MaxExp for co-occurrences emerges naturally from the Multinomial modeling which also gives it a nice interpretation as a co-occurrence detector.}

\vspace{0.05cm}
\noindent{\textbf{Derivation.}} 
Prop. \ref{prop:maxexp} states that $1\!-\!(1\!-\!p)^N$ is the probability of at least one success being detected in the pool of the $N$ i.i.d. trials $\vphi\!\in\!\{0,1\}^{N}$ following the Bernoulli dist. (success prob. $p$). We extend Prop. \ref{prop:maxexp} to the case of co-occurrences as follows. 

\revisedd{
\begin{theorem}
\label{pr:cooc}
Let two event vectors $\vphi,\vphi'\!\!\in\!\{0,1\}^{N}$ store $N$ trials each, performed according to the Bernoulli distribution under i.i.d. assumption, 
 for which the probability $p$ of an event $((\phi_{n}\!=\!1)\!\wedge\!(\phi'_{n}\!=\!1))$ denotes a co-occurrence, and $1\!-\!p$ for $((\phi_{n}\!=\!0)\!\vee\!(\phi'_{n}\!=\!0))$ denotes the lack of it. Let $p$ be estimated as $p\!=\!\avg_n\phi_n\phi'_{n}$. Then the probability of at least one co-occurrence event $((\phi_{n}\!=\!1)\!\wedge\!(\phi'_{n}\!=\!1))$ in $\phi_n$ and $\phi'_n$ simultaneously in $N$ trials becomes:
\vspace{-0.1cm}
\begin{equation}
\label{eq:my_maxexp3}
\psi\!=\!1\!-\!(1\!-\!p)^{N}.
\vspace{-0.1cm}
\end{equation}
\end{theorem}
\begin{proof}
The probability of all $N$ outcomes to be $\{\,((\phi_{1}\!=\!0)\!\vee\!(\phi'_{1}\!=\!0)),\cdots,((\phi_{N}\!=\!0)\!\vee\!(\phi'_{N}\!=\!0))\,\}$ is $(1\!-\!p)^N$. 
The probability of at least one positive outcome 
$((\phi_{n}\!=\!1)\!\wedge\!(\phi'_{n}\!=\!1))$ amounts to the probability of event $(\,((\phi_{1}\!=\!1)\!\wedge\!(\phi'_{1}\!=\!1))\,\vee\cdots \vee\,((\phi_{N}\!=\!1)\!\wedge\!(\phi'_{N}\!=\!1))\,)$ equal $1\!-\!(1\!-\!p)^{N}\!$, where $p\!=\!\avg_n\phi_n\phi'_{n}$.

A stricter proof uses a Multinomial distribution model with four events for $(\phi_{n})$ and $(\phi'_{n})$ which describe all possible outcomes. Let probabilities $p, q, s$ and $1\!-\!p\!-\!q\!-\!s$ add up to 1 and correspond to events $((\phi_{n}\!=\!1)\!\wedge\!(\phi'_{n}\!=\!1))$, $((\phi_{n}\!=\!1)\!\wedge\!(\phi'_{n}\!=\!0))$, $((\phi_{n}\!=\!0)\!\wedge\!(\phi'_{n}\!=\!1))$ and $((\phi_{n}\!=\!0)\!\wedge\!(\phi'_{n}\!=\!0))$. The first event is a co-occurrence, the latter two are occurrences only and the last event is the lack of the first three events. The probability of at least one co-occurrence $((\phi_{n}\!=\!1)\!\wedge\!(\phi'_{n}\!=\!1))$ in $N$ trials becomes:
\begin{align}
& \!\!\!\!\textstyle\sum\limits_{n=1}^{N}\sum\limits_{n'=0}^{N\!-n}\sum\limits_{n''\!=0}^{N\!-n-n'}\!\!\!\!\!\binom{N}{n,n'\!,n''\!,N\!-n-n'\!-n''\!} p^nq^{n'\!}s^{n''\!\!}(1\!\!-\!\!p\!\!-\!\!q\!\!-\!\!s)^{N\!-n-n'-n''}\!\!\!\!.\nonumber\\[-10pt]
& \label{eq:my_maxexppr2}
\end{align}
One can verify algebraically/numerically that Eq. \eqref{eq:my_maxexppr2} and \eqref{eq:my_maxexp3} are equivalent \wrt $p$ which completes the proof.
\end{proof}
}
\revisedd{
A proof with a Multinomial distribution shows that MaxExp for co-occurrences (Eq. \eqref{eq:my_maxexp3}) has exactly the same form as MaxExp for vectors (Prop. \ref{prop:maxexp}), and it acts as a co-occurrence detector. This justifies why Eq. \eqref{eq:my_maxexp3} is applicable both to first- and second-order representations in the element-wise pooling regime.} Below we explain practical details of this pooling \eg, how to apply it to an autocorrelation/covariance matrix and why this is meaningful.

\begin{table}[b]
\vspace{-0.3cm}
\begin{center}
{
\setlength{\tabcolsep}{0.3em}
\centering
\begin{tabular}{c | c c c c}
\kern-0.6em Pooling function\kern-0.3em & $g_i(p)$ & $g_i'(p)$ if & \multirow{2}{*}{$g_i(p)$} & \multirow{2}{*}{$g_i'(p)$} \\
(operator $i$) & if $p\!<\!0$ & $p\!\rightarrow\!0^+$ ($0^-$) & &\\
\hline
{\em Gamma}$\;$ \cite{me_tensor} & \multicolumn{1}{c}{inv.} & $\infty$ ({\fontsize{9}{8}\selectfont$-\infty\sqrt{-1}$}) & $p^{\gamma}$ & $\gamma p^{\gamma\!-\!1}$ \\
%
{\em MaxExp} \cite{me_tensor} &  inv. & fin.: $\eta$ & \kern-0.4em$1\!-\!(1\!-\!p)^{\eta}$ & \kern-0.3em$\eta(1\!-\!p)^{\eta\!-\!1}$ \\
\hline
{\em AsinhE} & ok & fin.: $\frac{\gamma'}{\sqrt{1+\gamma'^2}}$  & \kern-0.4em$\asinh(\gamma'p)$ & $\frac{\gamma'}{\sqrt{1+\gamma'^2p^2}}$\\
{\em SigmE} & ok & fin.: $0.5\eta'$ & \kern-0.4em $\frac{2}{1\!+\!e^{-\eta'p}}\!-\!1$ & $\frac{2\eta'\!e^{-\eta'p}}{(1+e^{-\eta'p})^2}$\\
{\em HDP} & inv. & 0 ($\infty$) & \kern-0.4em $e^{-t/p}$ & $\frac{te^{-t/p}}{p^2}$
\end{tabular}
}
\end{center}
\vspace{-0.3cm}
\caption{A collection of Power Normalization functions.
Variables $\gamma\!>\!0$, $\gamma'\!\!>\!0$, $\eta\!\geq\!1$, $\eta'\!\!\geq\!1$ and $t\!>\!0$ control the level of power normalization. For specific pooling type indicated by $i$, we indicate properties of $g_i(p)$ such as finite ({\em fin.}) or infinite ($\infty$) derivative $g_i'(p)$ at $p\!\rightarrow\!0^+$ ($0^-$), and invalid ({\em inv.}) or valid ({\em ok}) power norm. if $p\!<\!0$. HDP is a spectral operator (Sec. \ref{sec:hdp}) but it might act element-wise too.$\!\!$}
\label{tab:smd}
\vspace{-0.3cm}
\end{table}

\vspace{0.05cm}
\revisedd{
\noindent{\textbf{MaxExp pooling (second-order element-wise).}} 
%
In practice, we have $\psi_{kl}\!=\!1\!-\!(1\!-\!\tilde{p}/\kappa)^{\eta}$, where $0\!<\!\eta\!\approx\!N$ is an adjustable parameter, $\tilde{p}\!=\!\avg_n\phi_{kn}\phi_{ln}$, scalars $\phi_{kn}$ and $\phi_{ln}$ are $k$-th and $l$-th features of an $n$-th feature vector \eg, as defined in Prop. \ref{pr:linearize}, and $\kappa\!>\!0$ ensures that $0\!\leq\!\tilde{p}/\kappa\!\leq\!1$. 
As $\tilde{p}$ is an expected value over $N$ channel-wise correlations (product op.) of feature pairs from rectified CNN maps rather than co-occurrences of binary variables, we assume that $\phi_{kn}\phi_{ln}$ is proportional to the confidence that simultaneous detection of stimuli that channels $k$ and $l$ 
represent is correct, and $\kappa\!\geq\!\max_n\phi_{kn}\phi_{ln}$ is a value corresponding to the confidence equal one. We observe that:
%
\begin{equation}
1\!-\!\prod\limits_{n\in\mathcal{N}}(1\!-\!\phi_{kn}\phi_{ln}/\kappa)\!\geq\!1-(1\!-\!\alpha\tilde{p}/\kappa)^{\eta^*}\!\approx\!1-(1\!-\!p)^N\!,
\label{eq:adj1}
\vspace{-0.2cm}
\end{equation}

\noindent{where} $\eta^*\!\!=\!N$ and $\alpha\!=\!1$. The left-hand side eq. of \eqref{eq:adj1} is the likelihood of at least one co-occurrence if drawing from an unknown distribution under the i.i.d. assumption, thus we may desire to adjust the middle eq. in \eqref{eq:adj1} toward this upper bound. As the proportionality assumption may be violated in practice, $\alpha\!\neq\!1$ helps achieve a good estimate $p\!\approx\!\alpha\tilde{p}/\kappa$ and/or adjust the middle eq. in \eqref{eq:adj1} toward the left-hand eq. in \eqref{eq:adj1}.
Parameters $\eta^*$ and $\alpha$ may be tied together as we observe that $\log((1\!-\!\alpha\tilde{p}/\kappa)^{\eta^*\!})\!=\!\eta^*\!\log(1\!-\!\alpha\tilde{p}/\kappa)\!\!\approx\!-\eta^*\alpha\tilde{p}/\kappa\!=\!-\eta\tilde{p}/\kappa\!\approx\!\log((1\!-\!\tilde{p}/\kappa)^{\eta})$ as $\log(1\!-\!x)\!\approx\!-x$ and $\eta\!=\!\eta^*\alpha$ (tied parameter). Reversing the logarithm operation yields $1-(1\!-\!\tilde{p}/\kappa)^{\eta}\!\approx\!1-(1\!-\!\alpha\tilde{p}/\kappa)^{\eta^*}$ where $\eta\!\approx\!N$ refines 
$1-(1\!-\!\tilde{p}/\kappa)^{\eta}$ towards left- and right-hand side equations in \eqref{eq:adj1}. In matrix form, we have:
%
\begin{align}
& 
\mygthreee{MaxExp}{\,\mM;\eta\,}\!=\!1\!-\!\left(1\!-\!\mM/(\trace(\mM)\!+\!\varepsilon)\right)^\eta, 
\label{eq:my_maxexp4}
\end{align}
where $\kappa\!=\!\trace(\mM)\!+\!\varepsilon$ to ensure  $\kappa\!\geq\!\max_{k,l}M_{kl}$ is sufficiently large, 
 $\varepsilon\!\approx\!1e\!-\!6$, the global  param. $\eta$ is chosen via cross-validation to compensate for an estimate of $\kappa$, violation of the proportionality assumption, variations of distr., and the approx. of logarithm. 
%
%
%
}

\vspace{-0.2cm}
\begin{remark}
\label{re:maxexp_resid}
$\tG^{*}_{\text{MaxExp}}(\mM;\eta)\!=\!\tG_{\text{MaxExp}}(\mM;\eta)(\trace(\mM)\!+\!\varepsilon)^\gamma, \gamma\!\approx\!\frac{1}{2}$, compensates for the trace in \eqref{eq:my_maxexp4} which affected the input-output ratio of norms. $\tG^{\ddagger}_{\text{MaxExp}}(\mM,\eta)\!=\!\tG_{\text{MaxExp}}(\mM;\eta)\!+\!\kappa\mM$ prevents vanishing gradients in pooling. Both terms can be combined. \revised{
We suggest adding a small linear slope by $\tG^{\ddagger}\!$ ($\kappa\!\!>\!\!0$) if one experiences vanishing gradients. We use $\tG^{*}\!$ only for fine-tuning on off-the-shelf pre-trained CNN which produces feature vectors with the $\ell_1$ norms varying from region to region and/or image to image. As these norms have an impact on the quality of separation between different class concepts (because originally they were not excluded from training), they need to be adapted to the new dataset. 
}
\end{remark}

\vspace{-0.2cm}
\vspace{0.05cm}
\begin{tcolorbox}[width=1.0\linewidth, colframe=blackish2,colback=beaublue2, boxsep=0mm, arc=3mm, left=1mm, right=1mm, right=1mm, top=1mm, bottom=1mm]
\noindent{\textbf{Gamma pooling (second-order element-wise).}} \revisedd{For completeness, we present Gamma pooling based on the def. in Sec. \ref{sec:pn}:}
\begin{align}
& 
\mygthreee{Gamma}{\,\mM;\gamma\,}\!=(\mM\!+\!\varepsilon)^\gamma.
\label{eq:my_gamma1}
\end{align}

\noindent{Rising}  $\mM$ to the power of $\gamma$ is element-wise, $\varepsilon$ is a small reg. constant. Appendix \hyperref[{app:der}]{B} lists derivatives of Gamma and MaxExp. 
%
\end{tcolorbox}
\vspace{-0.4cm}

\vspace{-0.1cm}
\subsection{From MaxExp to MaxExp$\,$($\pm$) to SigmE (motivating second-order element-wise variants for $M_{kl}\!\in\!\mbr{}$).}
\label{sec:pool_der}
\revisedd{
Moving one branch down in the taxonomy from Fig. \ref{fig:whatwedo}, we note that matrix $\mM$  is built from features that may have negative values ($\beta$-centering or non-rectified $\mM$). Negative entries of $M_{kl}$ break MaxExp/Gamma. 
%
%
%
Thus, we extend Prop. \ref{pr:cooc} to MaxExp$\,$($\pm$) pooling which works also with negative co-occurrences interpreted by us as two anti-correlating visual words. 
To interpret such a pooling variant, we show that if trials follow a mixture of two Normal distributions, we obtain SigmE pooling (a zero-centered sigmoid function) which acts as a detector of co-occurrence/negative co-occurrence hypothesis. To establish affinity between MaxExp$\,$($\pm$) and SigmE, we show that both functions are identical if their parameters $\eta,\eta'\!\!\rightarrow\!\infty$.
However, the derivative of MaxExp$\,$($\pm$) is non-smooth at $0$ while SigmE has an almost identical profile to MaxExp$\,$($\pm$) but its der. is smooth (important in optimization). 
}

%
%
%
\revisedd{
\vspace{0.15cm}
\vspace{0.05cm}
\noindent{\textbf{Derivation.}} MaxExp$\,$($\pm$) and SigmE are derived in Proposition \ref{pr:cooc_anti} and Theorem \ref{pr:sigmoids} while a non-essential Remark \ref{re:maxexp_sigmoid} proves their affinity.$\!\!\!\!\!\!$
\vspace{-0.15cm}
\begin{proposition}
\label{pr:cooc_anti}
Let event vectors $\vphi,\vphi'\!\!\in\!\{-1,0,1\}^{N}$ 
store $N$ trials performed according to the Multinomial distr. under i.i.d. assumption, 
 for which we have the probability $p$ of a co-occurrence event $(\,((\phi_{n}\!=\!1)\!\wedge\!(\phi'_{n}\!=\!1))\vee((\phi_{n}\!=\!-1)\!\wedge\!(\phi'_{n}\!=\!-1))\,)$, the probability $q$ of a anit-correlating co-occurrence event $(\,((\phi_{n}\!=\!1)\!\wedge\!(\phi'_{n}\!=\!-1))\vee((\phi_{n}\!=\!-1)\!\wedge\!(\phi'_{n}\!=\!1))\,)$, and $1\!-\!p\!-\!q$ for $((\phi_{n}\!=\!0)\!\vee\!(\phi'_{n}\!=\!0))$ which denotes the lack of the first two events. Then the prob. difference between at least one co-occurrence event and at least one anti-correlating co-occurrence event in $N$ trials, encoded by $\vphi,\vphi'$ becomes 
\comment{
\vspace{-0.2cm}
\begin{equation}
\psi\!=\!(1\!-\!q)^{N}\!-\!(1\!-\!p)^{N}.
\label{eq:my_maxexp33}
\end{equation}
}
$\psi\!=\!(1\!-\!q)^{N}\!-\!(1\!-\!p)^{N}$.
\end{proposition}
\begin{proof}
\vspace{-0.2cm}
See  Appendix \hyperref[{app:eq14}]{A}.
\comment{One can derive the above difference of probabilities by directly applying the Multinomial calculus as follows:
%
\begin{align}
& \!\!\!\!\textstyle\sum\limits_{n=1}^{N}\sum\limits_{n'=0}^{N\!-n}\!\!\binom{N}{n,n'\!,N\!-n-n'\!-n''\!}\!\!\left(p^nq^{n'\!}\!-\!p^{n'\!}q^{n}\right)\!(1\!\!-\!\!p\!\!-\!\!q)^{N\!-n-n'}\!.
\label{eq:my_maxexppr22}
\end{align}
%
One can verify algebraically/numerically that Eq. \eqref{eq:my_maxexppr22} and \eqref{eq:my_maxexp33} are equivalent.
}
\end{proof}

%
\vspace{-0.2cm}
\vspace{0.05cm}
\noindent{\textbf{MaxExp$\,$($\pm$) (second-order element-wise).}} 
\noindent{In} practice, we simply estimate $p$ and $q$ as $p\!=\!\max(0,p^*\!)$ and $q\!=\!\max(0,-p^*\!)$, $p^*\!\!=\!\avg_n\phi^*_n$. Note $p^*\!\!<\!0$ or $p^*\!\!>\!0$ if the majority of co-occurrences between event vectors $\vphi, \vphi'$ captured as $\vphi^*\!\!=\!\vphi\odot\vphi\!\in\!\{-1,0,1\}^N$ are anti-correlating or correlating, respectively. 
}

\vspace{0.2cm}
\revisedd{
Below we show that SigmE 
has a derivation that follows a slightly different statistical interpretation than MaxExp, which explains that SigmE is a likelihood-based detector of co-occurrence \vs negative co-occurrence hypothesis.
\begin{theorem}
\label{pr:sigmoids}
\vspace{-0.15cm}
Assume an event vector $\vphi^*\!\in\mbr{N}$ whose coefficients $\phi_n^*\!$ represent anti-occurrences or occurrences drawn from $\mathcal{N}(-1,2)$ or $\mathcal{N}(1,2)$, resp. Then the probability of an event $\phi_n^*\!$ being an anti-occurrence and occurrence is given by $\psi^{(-)}(\phi_n^*)\!=\!\frac{G_{\sigma}(\phi_n^*\!+1)}{G_{\sigma}(\phi_n^*\!-1)\!+\!G_{\sigma}(\phi_n^*\!+1)}$ and $\psi^{(+)}(\phi_n^*)\!=\!\frac{G_{\sigma}(\phi_n^*\!-1)}{G_{\sigma}(\phi_n^*\!-1)\!+\!G_{\sigma}(\phi_n^*\!+1)}$, resp. We note that $\psi^{(-)}(0)\!=\!\psi^{(+)}(0)\!=0.5$ which means that for $\phi_n^*\!\!=\!0$, the determination of event type that the feature represents cannot be made in the statistical sense. As we want to factor out such cases, we simply set $\psi\!=\!\psi^{(+)}(\phi_n^*)\!-\!\psi^{(-)}(\phi_n^*)$ which reduces to $\frac{2}{1\!+\!\expl{-\eta'\phi_n^*\!}}\!-\!1$ for $\eta'\!\!=\!2/\sigma^2$. Thus, if $p^*\!\!=\!\avg_n\phi^*_n$,  SigmE given by $\frac{2}{1\!+\!\expl{-\eta'p^*\!}}\!-\!1$ (see Eq. \eqref{eq:sigme1}) simply tells, on average, whether events $\phi_n$ came more likely from $\mathcal{N}(-1,2)$ or $\mathcal{N}(1,2)$.
\end{theorem}
\begin{proof}
\vspace{-0.1cm}
Appendix \hyperref[{app:mix_proof}]{J} is the proof.
\end{proof}

\begin{remark}
\vspace{-0.05cm}
\label{re:maxexp_sigmoid}
One can verify that for $\eta\!\rightarrow\!\infty$ and $\eta'\!\!\rightarrow\!\infty$, we have $\lim_{\eta,\eta'\!\rightarrow\!\infty}\int_0^1|\frac{2}{1\!+\!\expl{-\eta'p\!}}\!-\!1\!-\!(1\!-\!(1\!-\!p)^\eta)|\,\mathrm{d}p\!=\!0$. In the limit, both formulations are identical on interval $[0,1]$. For the finite $\eta$ and $\eta'\!$, minimizing the above integral has no closed form but parametrizations $\eta'(\eta)\!=\!\frac{\log(\sqrt{3}\!+\!2)}{1\!-\!((4\!-\!2\sqrt{3})/3)^{1/(2\!\eta)}}$ and $\eta(\eta'\!)\!=\!\frac{\log((4\!-\!2\sqrt{3})/3)}{2\log(1\!-\!\log(\sqrt{3}\!+\!2)/\eta')}$ yield a low approx. error due to aligning SigmE with MaxExp at a point $p''(\eta')\!=\!\log(\sqrt{3}\!+\!2)/\eta'$ for which the concavity of SigmE on interval $[0,1]$ is at its maximum. 
\end{remark}
\begin{proof}
\vspace{-0.1cm}
Appendix \hyperref[{app:sigme_align_maxexp}]{K} is the proof. 
\end{proof}


\vspace{-0.2cm}
\vspace{0.05cm}
\begin{tcolorbox}[width=1.0\linewidth, colframe=blackish2,colback=beaublue2, boxsep=0mm, arc=3mm, left=1mm, right=1mm, right=1mm, top=1mm, bottom=1mm]
\noindent{\textbf{SigmE pooling (second-order element-wise).}} In practice, zero-centered Logistic a.k.a. Sigmoid ({\em SigmE}) functions below may be used in lieu of MaxExp$\,$($\pm$) and MaxExp (Eq. \eqref{eq:my_maxexp4}):
\begin{align}
& \!\!\!\!\!
\mygthreee{SigmE}{\,\mM;\eta'\,}\!=\!\frac{2}{1\!+\!\expl{-\eta'\mM}}\!-\!1\text{ or }\frac{2}{1\!+\!\expl{\frac{-\eta'\mM}{\trace(\mM)+\varepsilon}}}\!-\!1.\!
\label{eq:sigme1}
\end{align}
\end{tcolorbox}
\vspace{-0.2cm}

\vspace{0.05cm}
\noindent{\textbf{AsinhE pooling (second-order element-wise).}} For completeness,  we present AsinhE, an alternative to Gamma in Eq. \eqref{eq:my_gamma1} as Gamma has an infinite derivative for $M_{kl}\!\rightarrow\!0$ and $\varepsilon\!\rightarrow\!0$, and assumes $M_{kl}\!\geq\!0$. Its reg. $\varepsilon\!>\!0$ may affect results as Gamma magnifies signals close to $0$ \ie, $\varepsilon\!\approx\!1e\!-\!3$ will mask smaller signals. 

\begin{tcolorbox}[width=1.0\linewidth, colframe=blackish2,colback=beaublue2, boxsep=0mm, arc=3mm, left=1mm, right=1mm, right=1mm, top=1mm, bottom=1mm]
\textbf{AsinhE pooling} (Arcsin hyperbolic) by contrast has an almost identical profile with Gamma, it has a finite/smooth derivative and operates on $M_{kl}\!\in\!\mbr{}$ w/o regularization:
\begin{align}
& \!\!\!\!
\mygthreee{AsinhE}{\,\mM;\gamma'\,}\!=
\!\log(\gamma'\!\mM+\sqrt{1+{\gamma'}^2\!\mM^2}),
\label{eq:asinhe1}
\end{align}

Table~\ref{tab:smd}  lists  properties of Power Normalizations. Figure \ref{fig:power-norms} illustrates MaxExp/Gamma and their extensions SigmE/AsinhE (for $p\!\in\!\mbr{}$) whose derivatives (Appendix \hyperref[{app:der2}]{C}) are smooth/finite. 
\end{tcolorbox}
\vspace{-0.2cm}

\begin{figure}[t]
\centering
\vspace{-0.3cm}
\hspace{-0.3cm}
\begin{subfigure}[t]{0.495\linewidth}
\centering\includegraphics[trim=0 0 0 0, clip=true, height=\PowH, width=\PowW]{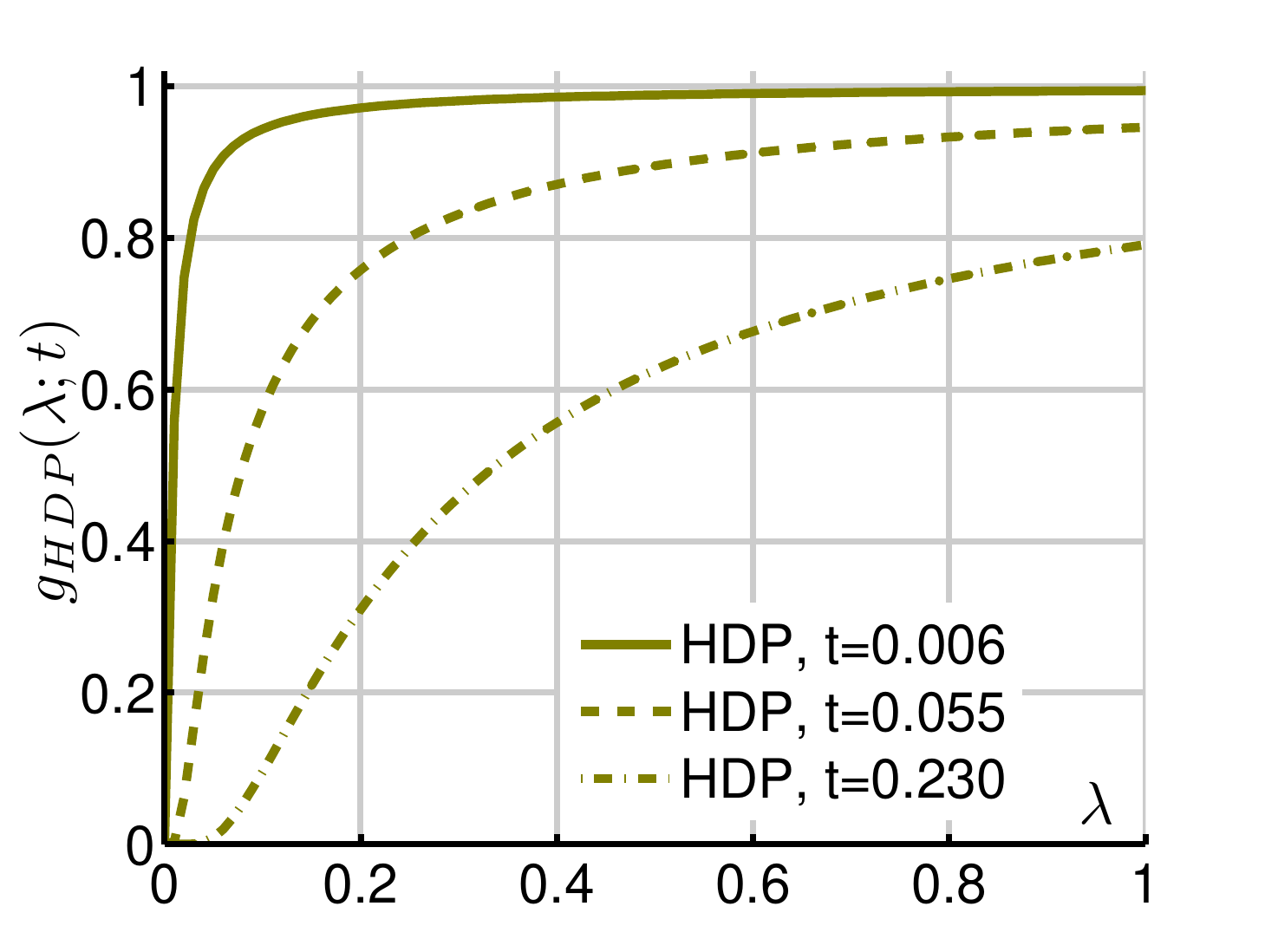}
\caption{HDP$\!\!\!\!\!\!\!\!$}\label{fig:pow33}
\end{subfigure}
\begin{subfigure}[t]{0.495\linewidth}
\centering\includegraphics[trim=0 0 0 0, clip=true, height=\PowH, width=\PowW]{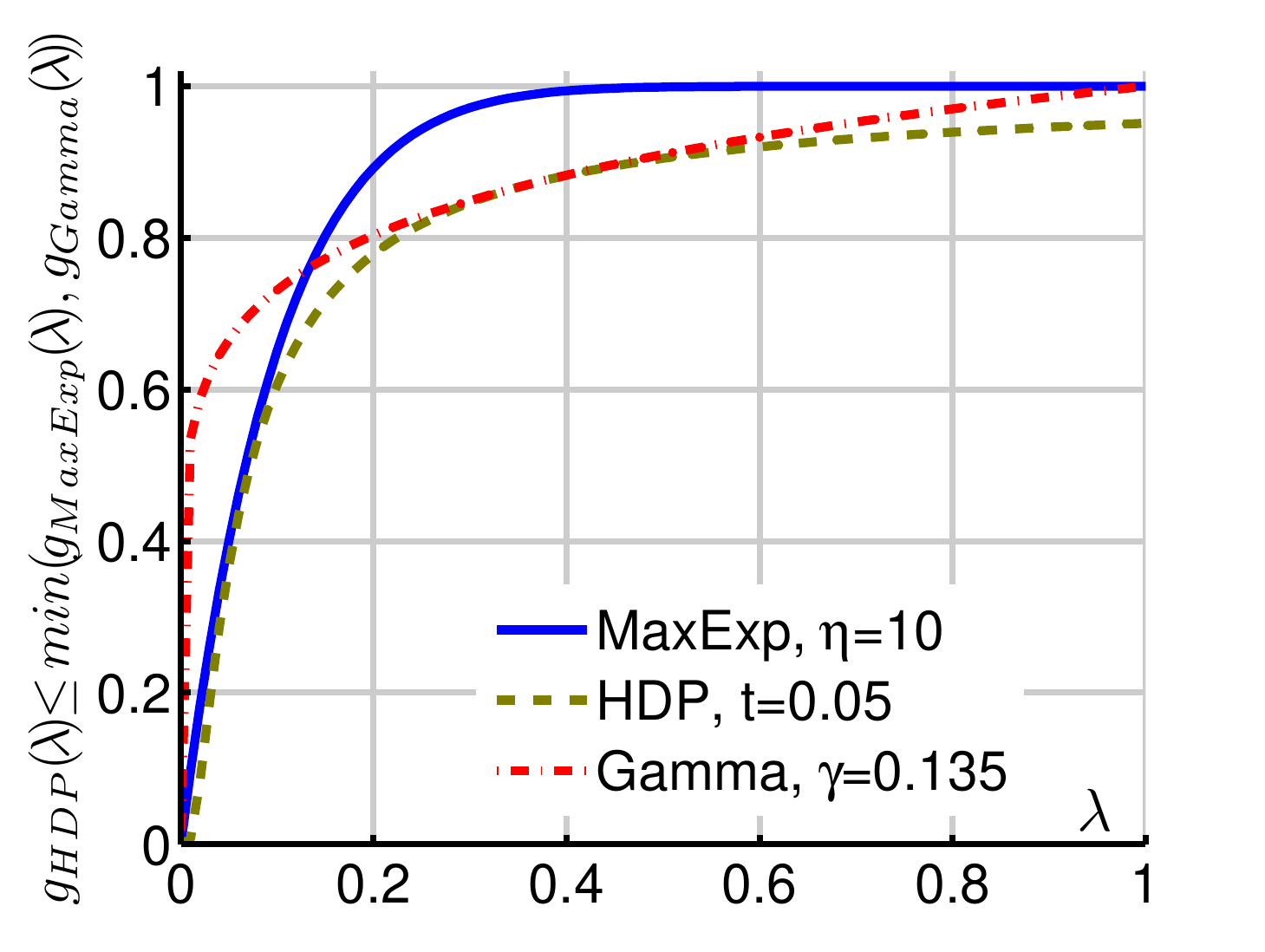}
\caption{\!\!\!\!\!\!\!\!}\label{fig:pow5}
\end{subfigure}
\comment{
\begin{subfigure}[t]{0.245\linewidth}
\centering\includegraphics[trim=0 0 0 0, clip=true, height=\PowH, width=\PowW]{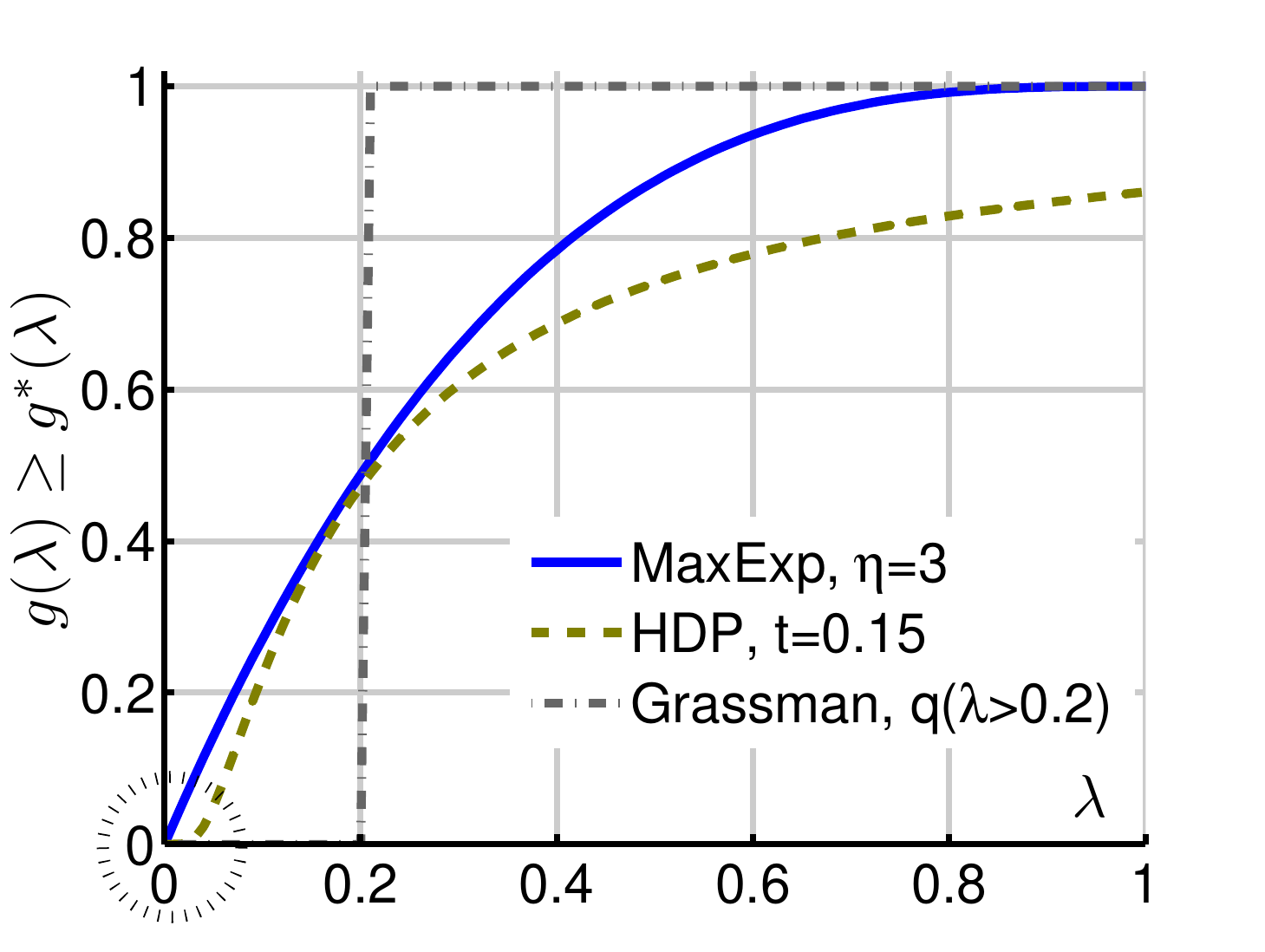}
\caption{Grassmann$\!\!\!\!\!\!\!\!$}\label{fig:pow44}
\end{subfigure}
\begin{subfigure}[t]{0.245\linewidth}
\centering\includegraphics[trim=0 0 0 0, clip=true, height=\PowH, width=\PowW]{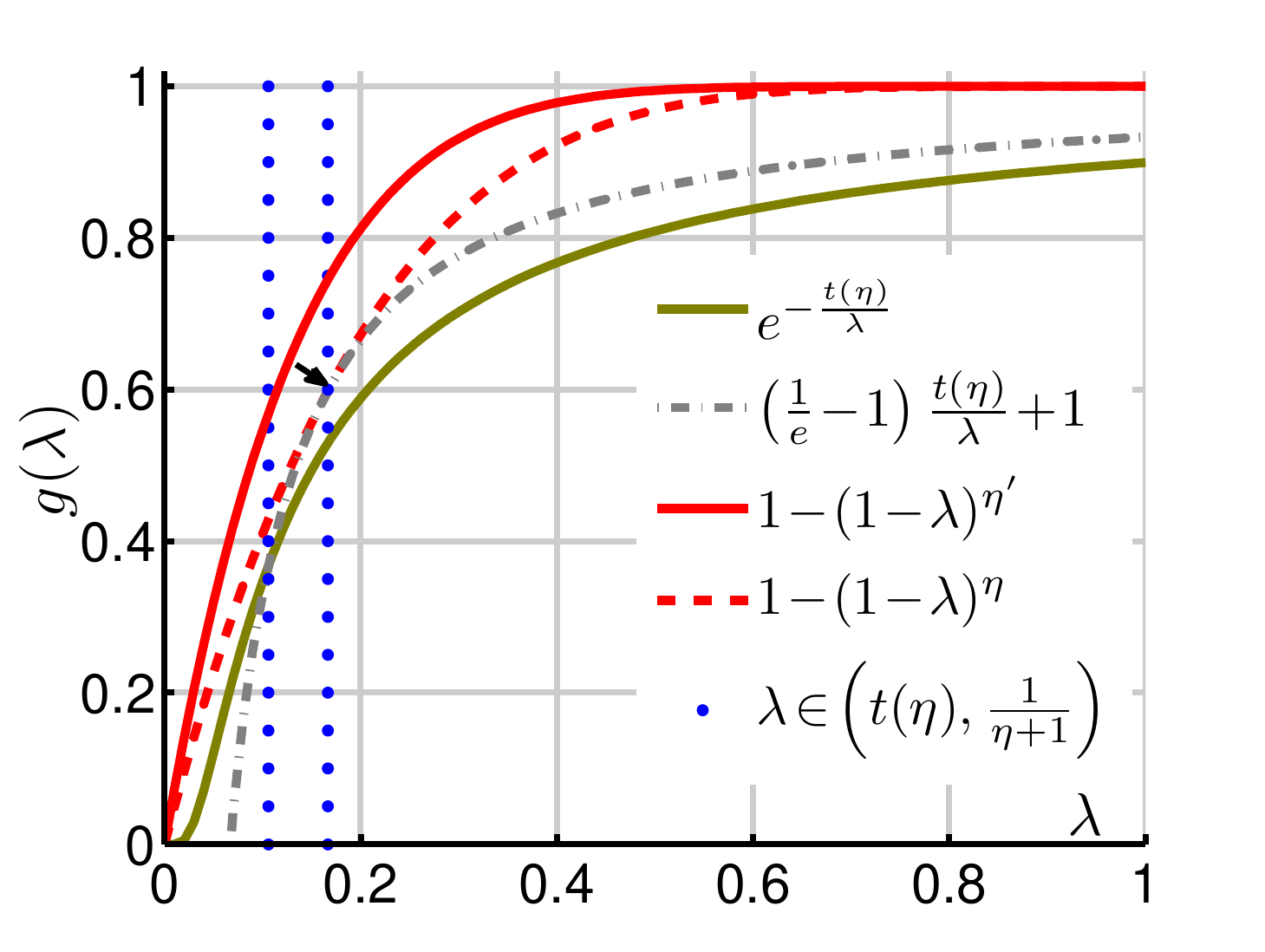}
\caption{$\!\!\!\!\!\!\!\!$}\label{fig:pown3}
\end{subfigure}
}
%
\caption{\revised{The profile of HDP 
is shown in Fig. \ref{fig:pow33}. Note the similarity of HDP to MaxExp. Fig. \ref{fig:pow5} shows that MaxExp and Gamma given by $g_{\text{MaxExp}}(\lambda)$ and $g_{\text{Gamma}}(\lambda)$ are in fact upper bounds of HDP ($g_{\text{HDP}}(\lambda)$). 
}}
\vspace{-0.5cm}
\label{fig:power-norms2}
\end{figure}

\revised{
\vspace{0.05cm}
\noindent{\textbf{Discussion.}} 
 %
Pooling functions with similar profiles can be formulated in numerous ways. MaxExp is shown to act as a burstiness-suppressing detector of at least one event in the collection of events drawn from the Bernoulli distribution. The Gamma operator can be viewed as a function that whitens signal. SigmE simply tells, on average, whether a collection of events $\phi_n$ is more likely to come from the Normal distribution representing anti-occurrences ($\mu\!=\!-1$) or its counterpart representing occurrences ($\mu\!=\!1$). The unifying factor for MaxExp and SigmE is how events are modeled \eg, according to the Bernoulli or Normal distributions. As element-wise operators do not take into account correlations between features, below we introduce and study Spectral Power Normalizations (second main taxonomy branch of Fig. \ref{fig:whatwedo}). 
}
}

\comment{
\begin{algorithm}[b]
\caption{\revised{Generalized Spectral Power Normalization}}
{\bf Input:} $\mM$ for a given forward pass,
$\varepsilon'\!$: minimum spectral gap, $d$: side dim. of $\mM$, $i\!\in\!\{\text{Gamma, MaxExp, AsinhE, SigmE, HDP}\}$
\begin{algorithmic}[1]
\State{$\mU\mLambda\mU^T\!\!=\!\mM$}\\
/*prevent non-simple eigenvalues*/
\While{$|\lambda_{i'}+\xi_{i'}\!-\!(\lambda_j+\xi_j)|\!\leq\!\varepsilon': \forall {i'}\!\neq\!j$}
	\State $\xi_{i'}\!\sim\!\mathcal{U}(\varepsilon',2d\varepsilon'), \forall i'$
\EndWhile\State \textbf{end}
\State{$\mLambda\gets\mLambda+\vec{\xi}$}
\State{Compute $\mygthreeehat{i}{\,\mM}$ and $\mygthreeephat{i}{\,\mM}$ as in Eq. \eqref{eq:spec}, Eq. \eqref{eq:eig_chain} and \eqref{eq:eig_der}}
\end{algorithmic}
{\bf Output:} $\mygthreeehat{i}{\,\mM}$ and $\mygthreeephat{i}{\,\mM}$
\label{code:algorithm}
\end{algorithm}
}

\comment{
\algblock[TryCatchFinally]{try}{endtry}
\algcblock[TryCatchFinally]{TryCatchFinally}{finally}{endtry}
\algcblockdefx[TryCatchFinally]{TryCatchFinally}{catch}{endtry}
	[1]{\textbf{except} #1}
	{\textbf{end}}
\algcblockdefx[TryCatchFinally]{TryCatchFinally}{elsee}{endtry}
	[1]{\textbf{else} #1}
	{\textbf{end}}
}

\algblock{while}{endwhile}
\algblock[TryCatchFinally]{try}{endtry}
\algcblockdefx[TryCatchFinally]{TryCatchFinally}{catch}{endtry}
	[1]{\textbf{except}#1}{}
\algcblockdefx[TryCatchFinally]{TryCatchFinally}{elsee}{endtry}
	[1]{\textbf{else}#1}{}
\algtext*{endwhile}
\algtext*{endtry}

\algblockdefx{ifff}{endifff}
	[1]{\textbf{if}#1}{}
\algtext*{endifff}
	
\begin{algorithm}[b]
\caption{\revised{Generalized Spectral Power Normalization}}
\label{code:algorithm}
\vspace{-0.2cm}
\begin{tcolorbox}[width=1.01\linewidth, colframe=blackish,colback=beaublue, boxsep=0mm, arc=3mm, left=1mm, right=1mm, right=1mm, top=1mm, bottom=1mm]
{\bf Input:} $\mM$ for a given forward pass,
$\varepsilon'\!$: desired spectral gap, $d$: side dim. of $\mM$, $i\!\in\!\{\text{Gamma, MaxExp, AsinhE, SigmE, HDP}\}$
\begin{algorithmic}[1]
\State{$k\!=\!0$} /*prevent non-simple eigenvalues*/
\while{ True:}
	\State $k\!\gets\!k\!+\!1,\;\{\xi_{i'}\!\sim\!\mathcal{U}(\varepsilon',\varepsilon'\!\!+\!d\varepsilon')\}_{i'\!=\!1,\cdots,d}$
	\try{:}
		\State{$\mU\mLambda\mU^T\!\!=\!\mM\!+\!\diag(k\boldsymbol{\xi})$}
		\State{get $\mygthreeehat{i}{\,\mM}$ and $\mygthreeephat{i}{\,\mM}$ by Eq. \eqref{eq:spec}, \eqref{eq:eig_chain} and \eqref{eq:eig_der}}
	\elsee{:}
		\State{break}
	\endtry
\endwhile
\end{algorithmic}
{\bf Output:} $\mygthreeehat{i}{\,\mM}$ and $\mygthreeephat{i}{\,\mM}$
\end{tcolorbox}
\vspace{-0.2cm}
\end{algorithm}


\revised{
\vspace{-0.3cm}
\section{Spectral Power Normalizations}
\label{sec:spns_main}
\revisedd{According to the main branch two in Fig. \ref{fig:whatwedo}, Spectral Power Normalizations act  on the spectrum of autocorrelation/covariance matrices instead of individual pairs of coefficients in order to respect the data correlation between principal directions of the  multivariate Normal Distribution represented by $\mM$. In this section, we provide a generalized recipe on computations of SPNs, Moreover, we show that Gamma/MaxExp are upper bounds of the time-reversed ($t\!<\!1$) Heat Diffusion Process \cite{smola_graph} which explains their good performance in classification (fourth main taxonomy branch of Fig. \ref{fig:whatwedo}). From these considerations emerges our culminating contribution, a fast spectral MaxExp that rivals recent Newton-Schulz iterations \cite{peihua_fast,lin2017improved}.}

\vspace{-0.2cm}
\subsection{Generalized SPNs}
\label{sec:spns}
\begin{tcolorbox}[width=1.0\linewidth, colframe=blackish2,colback=beaublue2, boxsep=0mm, arc=3mm, left=1mm, right=1mm, right=1mm, top=1mm, bottom=1mm]
\textbf{Spectral pooling} versions of our pooling operators can be obtained via SVD given as $\mU\mLambda\mU^T\!\!=\!\mM$. Let us define:
\vspace{-0.1cm}
\begin{align}
\mygthreeehat{i}{\,\mM}\!=\!\mU{\diag}^{\dagger}\!\left(\vec{g}_i\left(\vec{p}(\mLambda)\right)\right)\mU^T\!\!, 
\label{eq:spec}
\end{align}

\vspace{-0.1cm}
\noindent{where} Spectral Power Normalizations are realized by operators $g_i(\cdot)$ in Table \ref{tab:smd}, $\vec{p}(\mLambda)\!=\![p_1,\cdots,p_K]\!=\!\frac{\diag(\mLambda)}{\trace(\mLambda)\!+\!\varepsilon}$, 
the normalization by $\trace(\mLambda)\!+\!\varepsilon$ applies to MaxExp and SigmE. A small constants  $\varepsilon\!\approx\!1e\!-\!6$ prevents the vanishing trace, $\diag(\cdot)$ extracts the diagonal from $\mLambda$ into a vector while ${\diag}^{\dagger}(\cdot)$ places a vector into the diagonal of a matrix (off-diagonal coeffs. equal zero).
\end{tcolorbox}
\vspace{-0.15cm}

For a generic back-propagation through Eq. \eqref{eq:spec}, one relies on the back-propagation through SVD and the following chain rule:

{

\vspace{-0.3cm}
\begin{align}
&\!\!\!\text{\fontsize{7}{8}\selectfont$\frac{\mygthreeehat{i}{\,\mM}}{\partial M_{kl}}\!=\!\frac{\partial \mU{\diag}^{\dagger}\!\left(\vec{g}_i\left(\vec{p}(\mLambda)\right)\right)\mU^T}{\partial M_{kl}}\!=\!$}\nonumber\\
&\!\!\!\text{\fontsize{7}{8}\selectfont$\quad 2\sym\left(\frac{\partial \mU}{\partial M_{kl}}{\diag}^{\dagger}\!\left(\vec{g}_i\left(\vec{p}(\mLambda)\right)\right)\mU^T\!\right)
\!+\!\mU{\diag}^{\dagger}\!\left(\frac{\partial \vec{g}_i\left(\vec{p}(\mLambda)\right)}{\partial M_{kl}}\right)\mU^T\!\!.$}\label{eq:eig_chain}
\end{align}
\vspace{-0.3cm}
}
%
%
%

\noindent{The} back-propagation through eigenvectors and eigenvalues of SVD is a well-studied problem \cite{rogers_der,rudsil_der,magnus_der}:
%
\begin{align}
& \frac{\partial\lambda_{ii}}{\partial M}\!=\!\vu_i\!\vu_i^T,\quad
\frac{\partial u_{ij}}{\partial M}\!=\!u_{ij}(\lambda_{jj}\mIdent\!-\!M)^{\dagger}. 
\label{eq:eig_der}
\end{align}
\vspace{-0.3cm}
%

\noindent{As} the back-propagation through SVD breaks down in the presence of so-called non-simple eigenvalues, that is $\lambda_i\!=\!\lambda_j\!:i\!\neq\!j$, Algorithm \ref{code:algorithm} ensures that we draw small regularization coefficients $\xi_i$ from the uniform distribution until the spectral gap $\varepsilon'\!$ is ensured.

\begin{table}[b]
\vspace{-0.3cm}
\hspace{0.0cm}
\setlength{\tabcolsep}{0.3em}
\hspace{-0.2cm}
\begin{tabular}{c | c c}
\kern-0.6em {\scriptsize oper. $i$}\kern-0.2em  & \kern-0.2em\scriptsize$\!\!\!\mygthreeehat{i}{\mM}$\kern-0.2em & \kern-0.6em {\scriptsize type of back-propagation (speed)}\kern-0.6em \\ 
\hline
\kern-0.6em {\scriptsize\em Gamma}\kern-0.2em  & \kern-2.7em \scriptsize$\mM^\gamma$ (\eg, $\mM^\frac{1}{2}$)\kern-0.2em  & \kern-0.8em$\!\!\!\!\!\!\!\!\!\!\!\!$\pbox{5.7cm}{\scriptsize Sylv. Eq. \eqref{eq:sylv} of App. \hyperref[{app:der3}]{D} (extremely slow) / Sylv. Eq. \eqref{eq:sylv} via Bartels-Stewart alg. \cite{lin2017improved} (slow) / SVD back-prop. Eq. \eqref{eq:eig_chain} (slow) / Newton-Schulz \cite{lin2017improved} (fast)}\kern-0.1em\\
\hline
\kern-0.4em {\scriptsize\em MaxExp}\kern-0.2em  & \kern-2.7em\scriptsize$\mIdent\!-\!(\mIdent\!-\!\!\frac{\mM}{{\trace(\mM)+\varepsilon}})^\eta$\kern-0.3em  & \kern-3.0em\scriptsize SVD Eq. \eqref{eq:eig_chain} (slow) / Eq. \eqref{eq:der_maxexp} (fast) / Alg. \ref{code:exp_sqr} (very fast)\kern-0.8em\\
\kern-0.6em {\scriptsize\em AsinhE}\kern-0.2em  & \kern-0.1em \scriptsize$\logm\!\Big(\!\gamma'\!\mM\!\!+\!(\mIdent\!\!+\!\!\gamma'^2\!\mM^2)^{\frac{1}{2}}\!\Big)$   & $\!\!\!\!\!\!$\scriptsize SVD back-prop. Eq. \eqref{eq:eig_chain} (slow)\\
\kern-0.6em {\scriptsize\em SigmE}\kern-0.2em  & \kern-0.65em\scriptsize$2\Big(\mIdent\!\!+\!\!\expm({\!\frac{-\eta'\mM}{{\trace(\mM)+\varepsilon}}})\Big)^{\!-\!1}\!\!\!\!\!-\!\mIdent$\kern-0.6em  & $\!\!\!\!\!\!$\scriptsize SVD back-prop. Eq. \eqref{eq:eig_chain} (slow)\\
\kern-0.6em {\scriptsize\em HDP}\kern-0.2em  & \kern-0.2em \scriptsize$\expm(\!-t\mM^{\dagger})$\kern-0.2em  & $\!\!\!\!\!\!$\scriptsize SVD back-prop. Eq. \eqref{eq:eig_chain} (slow)\\
\hline
\end{tabular}
\caption{Spectral Power Normalizations. The square, square root, power, log and exp are matrix operations. The type of back-propagation available and speed are indicated. Gamma is the only well-explored Power Norm. yet other SPNs also work  well. They are all approx. of the time-reversed Heat Diffusion Process (Sec. \ref{sec:hdp}).}
\label{tab:smd2}
\vspace{-0.1cm}
\end{table}

Combining Eq. \eqref{eq:eig_chain} and \eqref{eq:eig_der} yields the same back-propagation equation as in \cite{sminchisescu_matrix}. Table \ref{tab:smd2} shows closed-form expressions for Power Normalizations realized via operators in Table \ref{tab:smd}.

\begin{figure*}[t]
\centering
\vspace{-0.3cm}
\hspace{-0.3cm}
\begin{subfigure}[t]{0.248\linewidth}
\centering\includegraphics[trim=0 0 0 0, clip=true, height=\PowH]{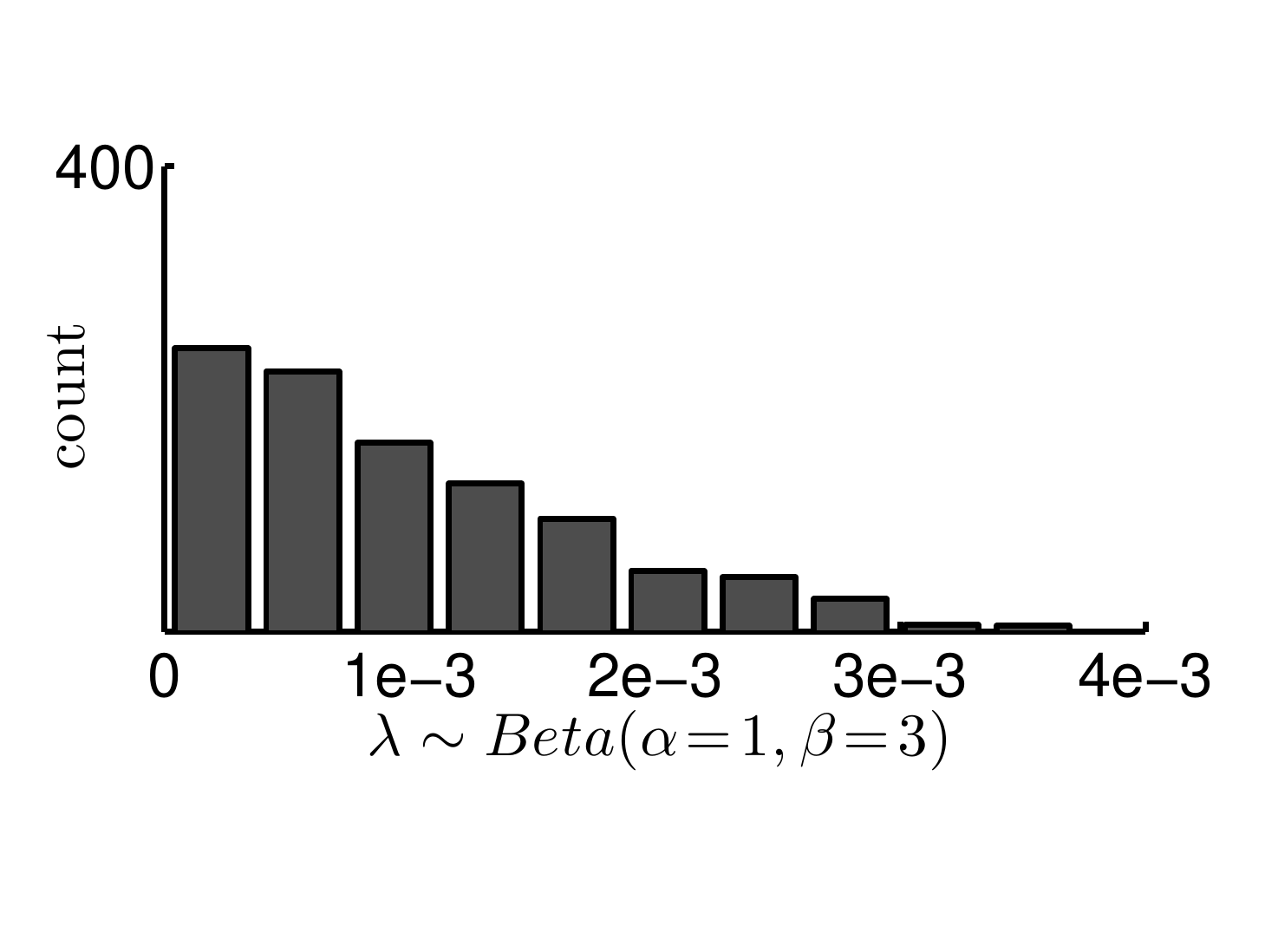}
\caption{Initial spectral dist.$\!\!\!\!\!\!\!\!$}\label{fig:dist1}
\end{subfigure}
\begin{subfigure}[t]{0.248\linewidth}
\centering\includegraphics[trim=0 0 0 0, clip=true, height=\PowH]{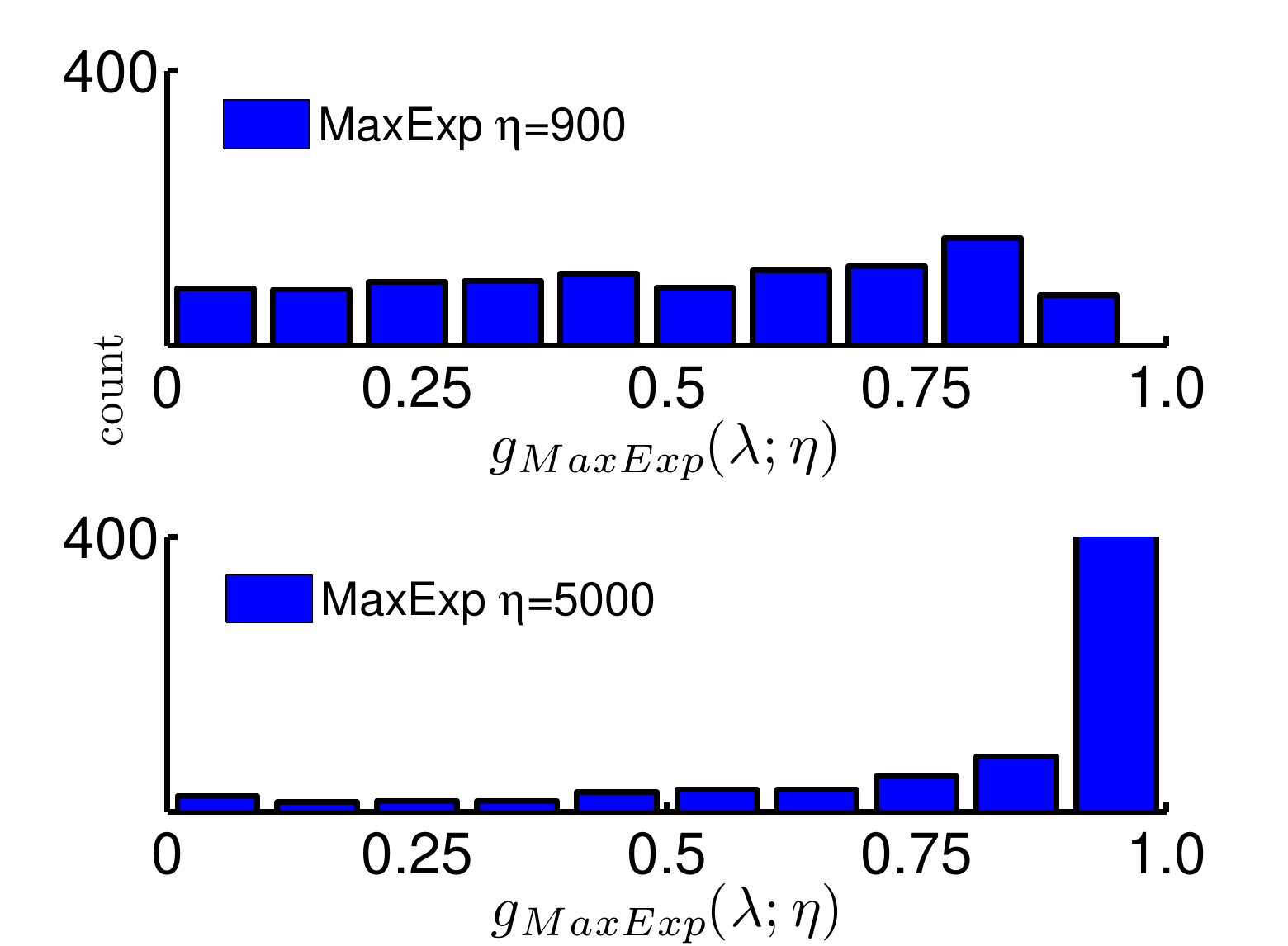}
\caption{Push-forward (MaxExp)$\!\!\!\!\!\!\!\!$}\label{fig:dist2}
\end{subfigure}
\begin{subfigure}[t]{0.248\linewidth}
\centering\includegraphics[trim=0 0 0 0, clip=true, height=\PowH]{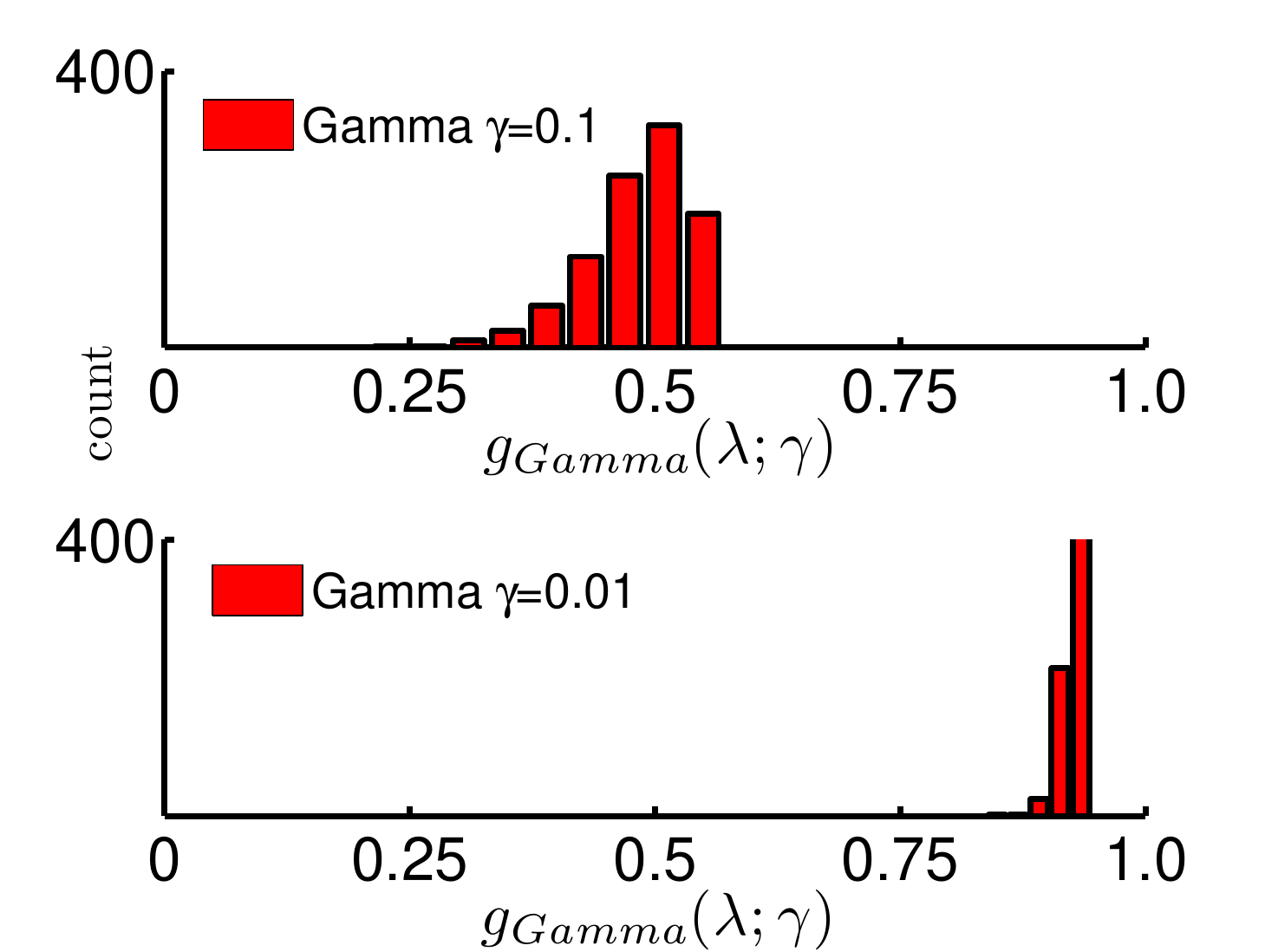}
\caption{Push-forward (Gamma)$\!\!\!\!\!\!\!\!$}\label{fig:dist3}
\end{subfigure}
\begin{subfigure}[t]{0.248\linewidth}
\centering\includegraphics[trim=0 0 0 0, clip=true, height=\PowH]{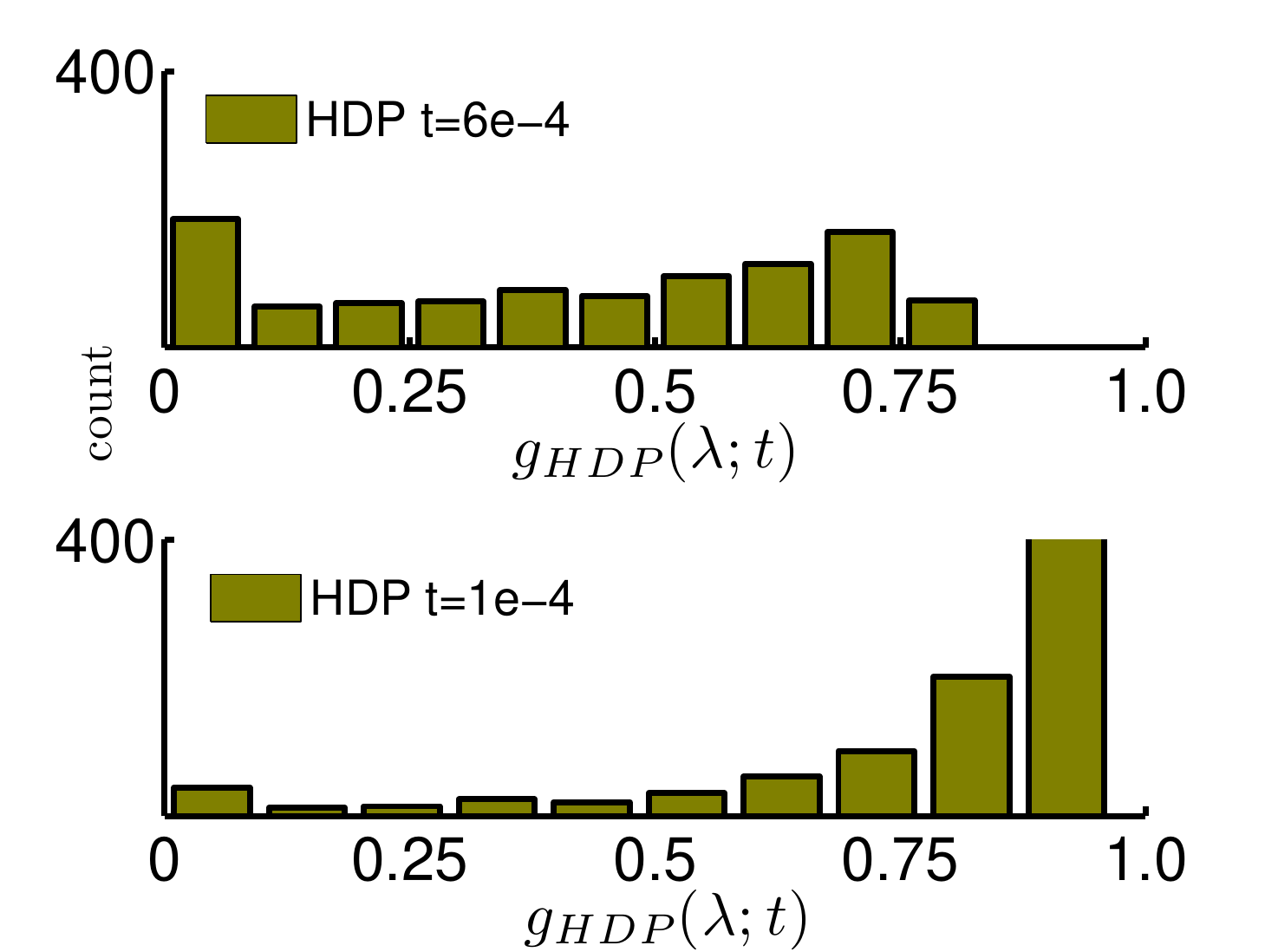}
\caption{Push-forward (HDP)$\!\!\!\!\!\!\!\!$}\label{fig:dist4}
\end{subfigure}
\caption{\revised{The intuitive principle of the SPN. Given a discrete eigenspectrum following a Beta distribution in Fig. \ref{fig:dist1}, 
the push-forward distributions of MaxExp and HDP in \cref{fig:dist2,fig:dist4} are very similar. For small $\gamma$, Gamma in Fig. \ref{fig:dist3} is also similar to MaxExp and HDP. Note that all three SPN functions whiten the spectrum (map the majority of values to be $\sim$1) thus reversing diffusion (acting as a spectral detector).}}
\vspace{-0.3cm}
\label{fig:epn-dist}
\end{figure*}

\vspace{-0.2cm}
\subsection{Fast Spectral MaxExp.}
\label{sec:fast_maxexp} 
\revisedd{
Below we present our culminating contribution: the fast spectral MaxExp. 
Running SVD is slow and back-prop. via SVD suffers from large errors as the spectral gap between eigenvalues narrows. Thus, a very recent trend in spectral pooling is to use Newton-Schulz iterations \cite{lin2017improved} to obtain a fast stable approximate matrix square root and its derivative, a special case of spectral Gamma with $\gamma\!=\!0.5$ which  is often argued to be a close-to-optimal parameter for Power Normalization. However, contradictory observations come from \cite{me_tensor,pk_maji} (see Table 1 in \cite{pk_maji}). To this end, we point that our spectral MaxExp has a nice property--its forward pass and its derivative can be computed fast for any $\eta\!\geq\!1$ with matrix-matrix multiplications (no SVD is needed). 

{
\footnotetext[1]{\label{foot:hdpcont}Application of HDP to fine-grained classification is also our contribution.}
}

\vspace{-0.1cm}
\begin{tcolorbox}[width=1.0\linewidth, colframe=blackish, colback=beaublue, boxsep=0mm, arc=3mm, left=1mm, right=1mm, right=1mm, top=1mm, bottom=1mm]
Without the loss of generality, if $\mM$ is trace-normalized then spectral MaxExp
\vspace{-0.4cm}
\begin{align}
& \;\;\mygthreeehat{MaxExp}{\mM}\!=\!\mIdent\!-\!\left(\mIdent\!-\!\mM\right)^\eta 
\label{eq:maxexp_spectral}
\end{align}

\vspace{-0.2cm}
\noindent{has} the closed-form derivative 
defined as:
\vspace{-0.1cm}
\begin{align}
& \!\!\!\!\frac{\partial\mygthreeehat{MaxExp}{\mM}}{\partial\ M_{kl}}\!=\!-\!\sum\limits_{n\!=\!0}^{{\eta}-1}\left(\mIdent\!-\!\mM\right)^n\!\mJ_{kl}\left(\mIdent\!-\!\mM\right)^{{\eta}-1-n}, 
\label{eq:binom_mat_der}
\end{align}

\vspace{-0.2cm}
\noindent{whose} detailed form given in Appendix \hyperref[{app:maxexp_fast_spec}]{N} scales linearly (runtime) \wrt $\eta$. However, evaluation time of MaxExp and its derivative can scale sublinearly \wrt $\eta$ given Algorithm \ref{code:exp_sqr}.
\end{tcolorbox}
\vspace{-0.3cm}
%

\vspace{0.15cm}
\noindent{\textbf{Forward pass.}} Given an integer $\eta\!\geq\!1$, computing $\left(\mIdent\!-\!\mM\right)^\eta$ has a lightweight complexity $\bigoh(m\log({\eta}))$, where $m$ is the cost of the matrix-matrix multiplication in exponentiation by squaring \cite{expsquare} whose cost is $\log({\eta})$. In contrast, the matrix square root via $k$ Newton-Schulz iterations has complexity $\bigoh(mk)$. Head-to-head, our subroutine performed $8$ matrix-matrix multiplications for $\eta\!=\!50$ (typically $20\!\leq\!\eta\!\leq\!80$ as in Fig. \ref{fig:eval4}). Newton-Schulz iter. performed $3k\!=\!60$ matrix-matrix mult. for $k\!=\!20$ set in \cite{lin2017improved}.

\vspace{0.05cm}
\noindent{\textbf{Backward pass.}} Given an integer $\eta\!\geq\!1$, the required powers of $\left(\mIdent\!-\!\mM\right)$ come from the forward-pass. Auto-differentiation runs  
along the recursion path of exponentiation by squaring whose complexity is $\bigoh(m\log({\eta}))$. In contrast, der. of the matrix square root via Newton-Schulz iterations has  complexity $\bigoh(mk)$ ($k\!\approx\!20$ \cite{lin2017improved}). Head-to-head, we required $11$ matrix-matrix multiplications ($3\!\times\!2$ and $5\!\times\!1$ for the derivative in line 4 and 8 respectively of Alg. \ref{code:exp_sqr}) for $\eta\!=\!50$. Newton-Schulz iter. require $6k\!=\!120$ matrix-matrix mult. for $k\!=\!20$. Memory-wise, MaxExp and Newton-Schulz iter. need to store  $\sim\!2\log{\eta}$ and $2k$ matrices, respectively. 

Finally, the complexity of an SVD is $\bigoh(d^\omega)$ with  $2\!<\!\omega\!<\!2.376$. 
Non-spectral PNs are faster by the order of magnitude.}

\begin{algorithm}[b]
\caption{\revisedd{Exponentiation by Squaring in Fast Spectral MaxExp}}
\label{code:exp_sqr}
\vspace{-0.2cm}
\begin{tcolorbox}[width=1.01\linewidth, colframe=blackish,colback=beaublue, boxsep=0mm, arc=3mm, left=1mm, right=1mm, right=1mm, top=1mm, bottom=1mm]
{\bf Input:} $\mM$ for a given forward pass, $\eta\!\geq\!1$
\begin{algorithmic}[1]
\State{$\mG_1\!=\!\mIdent,\; \mM^*_1\!\!=\!\mIdent\!-\!\mM,\; n\!=\!\text{int}(\eta),\;t\!=\!1,\;q\!=\!1$}
\while{ $n\!\neq\!0$:}
	\ifff{ $n\&1$:}
	\State $\mG_{t+1}\!=\!\mG_{t}\mM^*_{q},\;\;\frac{\partial\mG_{t+1}}{\partial\mM}\!=\!\frac{\partial\mG_{t}}{\partial\mM}\mM^*_{q}\!+\!\mG_{t}\frac{\partial\mM^*_{q}}{\partial\mM}$
	%
	%
	\State $n\!\gets\!n\!-\!1,\;t\!\gets\!t\!+\!1$
	\endifff
	\State{$n\!\gets\!\text{int}(n/2)$}
	\ifff{ $n\!>\!0$:}
	\State $\mM^*_{q+1}\!=\mM^*_{q}\!\mM^*_{q},\;\;\frac{\partial\mM_{q+1}}{\partial\mM}\!=2\sym\left(\frac{\partial\mM^*_{q}}{\partial\mM}\mM^*_{q}\right)$
	\State $q\!\gets\!q\!+\!1$
	\endifff
\endwhile
\end{algorithmic}
{\bf Output:} $\mygthreeehat{MaxExp}{\,\mM}\!=\!\mIdent\!-\!\mG_t$ and $\mygthreeephat{MaxExp}{\,\mM}=\!-\frac{\partial\mG_{t}}{\partial\mM}$
\end{tcolorbox}
\vspace{-0.2cm}
\end{algorithm}


\vspace{-0.2cm}
\subsection{Spectral MaxExp is a Time-reversed Heat Diffusion. }
\label{sec:hdp} 
\revisedd{Our key theoretical contribution below shows that the Heat Diffusion Process{\color{red}\footnotemark[1]} (HDP) on a graph Laplacian is closely related to the Spectral Power Normalization (SPN) of the autocorrelation/covariance matrix (see second/fourth main taxonomy branch of Fig. \ref{fig:whatwedo}), whose inverse forms a loopy graph Laplacian.} 
To this end, we firstly explain the relation between an autocorrelation/covariance and a graph Laplacian matrices. Subsequently, we establish that for HDP with $t\!<\!1$, MaxExp and Gamma functions are tight upper bounds of HDP  for some parametrizations $\eta(t)$ and $\gamma(t)$. Finally, we reconsider the system of Ordinary Differential Equations describing HDP and show that MaxExp and Gamma correspond to modified heat diffusion ODEs. 
Figure \ref{fig:epn-dist} shows that HDP, spectral MaxExp and Gamma play the same role \ie, they boost or dampen the magnitudes of the eigenspectrum to concentrate eigenvalues around a single peak, thus 
reversing the diffusion of signal in autocorrelation/covariance matrices.

\vspace{-0.1cm}
\begin{theorem}\label{th:heatdiff}
A structured multivariate Gaussian distribution with a covariance matrix $\mM$ is associated
with a weighted graph such that the precision matrix $\!\mQ\!\equiv\!\mM^{\dagger}\!\!$ corresponds to the loopy graph Laplacian \cite{GMRF-graph}.
Let $\mM\!$ be trace-normalized so that $0\!<\!\sum_i\lambda_i\!\leq\!1$,
and $
\mygthreeehat{MaxExp}{\mM; \eta}\!=\!\mU\diag(1\!-\!(1\!-\!\vec{\lambda})^\eta)\mU^T$ be our spectral MaxExp operator. 
Let $\mK(\mQ; t)\!=\!\expm(-t\mQ)$ with $t\!>\!0$ be the Heat Diffusion Process on the graph. 
Then $\forall{\eta}>1$, $\mygthreeehat{MaxExp}{\mM; \eta}$ can be well approximated by some $\mK(\mQ; t)$, $0\!<\!t\!<\!1$. Similarly, $\mK(\mQ; t)$ can approximate Gamma.
\end{theorem}
\begin{proof}
\vspace{-0.1cm}
Appendix \hyperref[{app:graph}]{E} is the proof. 
\end{proof}

\revisedd{\vspace{-0.2cm}The loopy graph \cite{GMRF-graph} Laplacian (see Fig. \ref{fig:whatwedo}) represents each feature (that corresponds to a given CNN filter)  by a node while  edges quantify the similarity between features. The loopy graph is  characterized by: (i) symmetric positive definite matrix (\cf semi-definite Laplacian) due to self-loops of each node, (ii) dense connectivity (each node may connect with other nodes), (iii) the underlying multivariate Gaussian distribution, (iv) structural design (\ie, we augment the autocorrelation matrix with spatial coordinates making some nodes location-specific).
}

As the time parameter $t$ is in the range $(0,1)$, we call 
$\mK(\mQ; t)$ a time-reversed Heat Diffusion Process. This means that rather than diffusing the heat between the nodes ($t\!>\!1$), the model reverses the process in the direction of the identity matrix, that is $\lim_{t\to 0}\mK(\mQ; t)\!\rightarrow\!\mIdent$, $\lim_{\eta\to \infty}\mygthreeehat{MaxExp}{\mM; \eta}\!\rightarrow\!\mIdent$ and $\lim_{\gamma\to 0}\mM^\gamma\!\rightarrow\!\mIdent$. This coincides with so-called eigenspectrum whitening which prevents burstiness \cite{koniusz2018deeper}.
As Theorem \ref{th:heatdiff} does not state any  approximation results, we  have following theorems.

\begin{theorem}\label{th:param_bounds}
\vspace{-0.1cm}
$\forall\eta\!>\!1$, $\exists{}t(\eta)$ such that 
MaxExp function is an upper bound of HDP: $1\!-\!(1\!-\!\lambda)^{\eta}\!\geq\!\exp(-t(\eta)/\lambda)$,
$\forall\lambda\!\in\![0,1]$,
and gaps $\epsilon_1$ and $\epsilon_2$ between these two functions
at $\lambda\!=\!t(\eta)$ and $\lambda\!=\!\frac{1}{\eta+1}$, where the auxiliary bound $(\frac{1}{e}\!-\!1)\frac{t(\eta)}{\lambda}\!+\!1$ (Appendix \hyperref[{app:pr_bound}]{F}) touches HDP and MaxExp as in Fig. \ref{fig:bounds-zoom} (supp. mat.), resp.,
satisfy:
%
\vspace{-0.2cm}
\begin{equation}
\fontsize{8}{9}\selectfont
\!\!\!\frac{e\!-\!1}{e}-\left(1\!-\!\frac{e}{e\!-\!1}\frac{\eta^\eta}{(\eta\!+\!1)^{\eta+1}}\right)^\eta\!\!\!\!=\! \epsilon_1\!\leq\!\epsilon_2\!=\!1\!-\! {\left(\frac{\eta}{\eta\!+\!1}\right)^\eta}\!\!\!-\, e^{\!\!\displaystyle-\frac{e}{e\!-\!1}\left(\frac{\eta}{\eta\!+\!1}\right)^\eta}\!\!\!,
\label{eq:bounds1}
\end{equation}
%
%
One possible parametrization $t(\eta)$ satisfying the above condition is:
\vspace{-0.2cm}
\begin{equation}
t(\eta)\!=\!\frac{e}{e\!-\!1}\frac{\eta^\eta}{(\eta\!+\!1)^{\eta+1}},
\label{eq:eta_par}
\vspace{-0.3cm}
\end{equation}
and conversely:
\vspace{-0.2cm}
%
\begin{equation}
\eta(t)\!\approx\!0.5\sqrt{1\!+\!4/(t^2(e\!-\!1)^2)}\!-\!0.5.
\label{eq:t_par}
\end{equation}
\end{theorem}
\begin{proof}
\vspace{-0.1cm}
Appendix \hyperref[{app:pr_bound}]{F} is the proof. 
\vspace{-0.1cm}
\end{proof}

\begin{theorem}\label{th:gamma_bounds}
Gamma function is an upper bound of HDP, that is $\lambda^{\gamma(t)}\!\geq\!\exp(-t/\lambda)$, and there exist a direct point other than at $\lambda\!=\!0$ where Gamma and HDP touch. Moreover, the corresponding parametrizations are $\gamma(t)\!=\!et$ and $t(\gamma)\!=\!e^{-1}\gamma$.
\end{theorem}
\vspace{-0.1cm}
\begin{proof}
\vspace{-0.1cm}
Appendix \hyperref[{app:gamma_bound}]{G} is the proof. 
\vspace{-0.1cm}
\end{proof}

\Cref{th:param_bounds,th:gamma_bounds} give us a combined tighter bound $\min(1\!-\!(1\!-\!\lambda)^{\eta(t)},\lambda^{\gamma(t)})\!\geq\!\exp(-t/\lambda)$. \Cref{th:param_bounds,th:gamma_bounds} can now be used in the following theorems connecting MaxExp and Gamma to the Heat Diffusion Equation (HDE) \cite{heat_eqq} which is the system of the Ordinary Differential Equations describing HDP. The HDE is given as:
\vspace{-0.2cm}
\begin{equation}
\frac{\partial\vv(t)}{\partial t}+\mL\vv(t)=0,
\label{eq:heat_eq}
\end{equation}
where vector $\vv(t)\in\mbr{d}$ describes some heat quantity of graph nodes at a time $t$, where $\mL\!\in\!\semipd{d}$ (or $\spd{}$) is the graph Laplacian (or the loopy graph Laplacian).

\begin{theorem}\label{th:maxexp_ode}
MaxExp can be expressed as a modified Heat Diffusion Equation, where the largest eigenvalue of $\mM$ is assumed to be $\lambda_{max}\!\leq\!1$, and  $\mM\!=\!\mL^\dag$, thus we have:
\begin{equation}
\frac{\partial\vv(t)}{\partial t}+\frac{\partial\eta}{\partial t}\logm\left(\mIdent\!-\!\mM\right)\left(1-\vv(t)\right)=0.
\label{eq:heat_maxexp}
\end{equation}
\end{theorem}
\begin{proof}
\vspace{-0.1cm}
Appendix \hyperref[{app:maxexp_ode}]{H} is the proof. 
\vspace{-0.1cm}
\end{proof}

\begin{theorem}\label{th:gamma_ode}
\vspace{-0.1cm}
Gamma can be expressed as a modified Heat Diffusion Equation:
\vspace{-0.5cm}
\begin{equation}
\frac{\partial\vv(t)}{\partial t}+e\logm(\mL)\vv(t)=0,
\label{eq:heat_gamma}
\end{equation}
where $\logm(\mL)$ is a Log-Euclidean map of the (loopy) graph Laplacian \ie, $\mL\!=\!\mM^\dag$. Thus, Gamma is equal to HDP on a Log-Euclidean map of the loopy graph Laplacian, a theoretical connection between Power-Euclidean and Log-Euclidean metrics.
\end{theorem}
\begin{proof}
\vspace{-0.1cm}
Appendix \hyperref[{app:gamma_ode}]{I} is the proof. 
\vspace{-0.1cm}
\end{proof}

\revisedd{
\vspace{0.05cm}
\noindent{\textbf{Fast approximate HDP.}} For completeness of our theoretical discussions, we parametrize  MaxExp and Gamma \wrt time to devise HDP whose runtime scales sublinearly with $t$. MaxExp and Gamma are upper bounds of HDP and they both can approximately realize the time-reversed ($t\!\in\!(0,1)$) and time-forward ($t\!\geq\!1$) HDP. MaxExp in the time-reversed regime (and its derivative) can be evaluated very fast for integers $\eta\!\geq\!1$ via matrix-matrix multiplications (Sec. \ref{sec:fast_maxexp}, Alg. \ref{code:exp_sqr}). 
Moreover, Gamma $\mM^\gamma$ in the time-forward regime can be evaluated very fast for integers $\gamma\!\geq\!1$ via matrix-matrix multiplications. While the evaluation time of the derivative in Appendix \hyperref[{app:gamma_fast_spec}]{O} scales linearly \wrt $\gamma\!\geq\!1$, Gamma and its derivative can be computed in the sublinear time by modified Alg. \ref{code:exp_sqr} (modify  $\mM^*_1\!:=\!\mM$ (line 1), $\mygthreeehat{Gamma}{\,\mM}\!:=\!\mG_t$ and $\mygthreeephat{Gamma}{\,\mM}:=\frac{\partial\mG_{t}}{\partial\mM}$ (output), replace variable $\eta$ with $\gamma$). 
%
\begin{tcolorbox}[width=1.0\linewidth, colframe=blackish, colback=beaublue, boxsep=0mm, arc=3mm, left=1mm, right=1mm, right=1mm, top=1mm, bottom=1mm]
FAHDP uses MaxExp/Gamma for the fast approximation of the time-reversed/time-forward HDP, respectively. For the trace-normalized $\mM$, we have{\color{red}\footnotemark[2]}:
\vspace{-0.3cm}
\begin{align}
& \;\mygthreeehat{FAHDP}{\mM;t}\!=\!e^{-t}\!\cdot\!
\begin{cases}
\begin{array}{@{}cl}
\text{\fontsize{8}{9}\selectfont $\mIdent\!-\!\left(\mIdent\!-\!\mM\right)^{\hbar_t(\widetilde{\eta}(t))}$} & \!\!\!\!\!\text{\fontsize{8}{9}\selectfont if $t\!<\!1$}\\
\text{\fontsize{8}{9}\selectfont $\mM^{\hbar_t(\bar{\gamma}(t))}$} & \!\!\!\!\!\text{\fontsize{8}{9}\selectfont if $t\!\geq\!1$.}
\end{array}
\end{cases}\!\!\!\!\!\!\!\!\!\!\!\!
\label{eq:fahdp}
\end{align}
\end{tcolorbox}
\vspace{-0.1cm}
%
\footnotetext[2]{\label{foot:anypn}For the time-reverse regime of FAHDP, one could use $\mM^{\hat{\gamma}(t)},\,\hat{\gamma}(t)\!<\!1$, $0\!<\!t\!<\!1$ in place of $\mIdent\!-\!\left(\mIdent\!-\!\mM\right)^{\hbar_t(\widetilde{\eta}(t))}$ \eg, $\mM^{0.5}$ reverses $\mM^{2}$. However, Gamma and its back-propagation are slow to compute if $0\!<\!t\!<\!1$.}

\noindent{In} the above eq., the scaling factor $e^{-t}$ ensures the Fast Approximate HDP (FAHDP) and HDP have the same magnitude at $\lambda\!=\!1$ (trace-normalized $\mM$ has eigenvalues $0\!\leq\!\lambda\!\leq\!1$). Note that for $\lambda\!=\!1$ and $t\!\geq\!\log(10/9)\!\approx\!-0.105$, HDP ($e^{-t/\lambda}$) yields $e^{-t}\!\leq\!0.9$ (or less for larger $t$) while MaxExp ($1\!-\!(1\!-\!\lambda)^\eta$) and Gamma ($\lambda^\gamma$) yield $1$ for $\lambda\!=\!1$. 
The scaling factor has no impact on classification results as it is a constant that depends on $t$ fixed throughout an experiment. However, $e^{-t}$ makes FAHDP and HDP visually similar. Thus, we reparametrize $\eta$ and $\gamma$ of MaxExp and Gamma by $\widetilde{\eta}(t)\!=\!0.5\!+\!\sqrt{0.25\!+\!\left(1/(t(e\!-\!1))\!-\!1/e\right)^2}$ and $\bar{\gamma}(t)\!=\!t$. Appendix \hyperref[{app:fahdp_der}]{P} contains proofs and further expansions. 
Finally, $\hbar_t(\cdot)$ may be the `round' function to ensure that MaxExp/Gamma receive an integer parameter (for which forward/backward steps run fast). If $\hbar_t(x)\!=\!x$ or $\hbar_t(x)\!=\!\left\lceil{x}\right\rceil$ for $t\!<\!1$ and $\hbar_t(x)\!=\!\left\lfloor{x}\right\rfloor$ for $t\!\geq\!1$, FAHDP is an upper bound of HDP on $t\!\in\!(0,\infty)$.}

\vspace{0.05cm}
\noindent{\textbf{Discussion.}} We have shown in Theorems \ref{th:param_bounds} and \ref{th:gamma_bounds} that MaxExp and Gamma are tight upper bounds of time-reversed HDP thus having the same role as HDP. We note that MaxExp and Gamma do not require an inversion of spectrum thus they are natural choices for autocorrelation/covariance matrices while HDP is a natural choice for the graph Laplacian matrix. Moreover, in Theorems \ref{th:maxexp_ode} and \ref{th:gamma_ode}, we are the first to cast MaxExp and Gamma in the form of modified differential heat equations well-known for HDP. For instance, Gamma is equal to HDP on a Log-Euclidean map of the graph Laplacian, a first concrete result of this kind made in the literature showing how Gamma and HDP relate. The time reversal to a desired state is achieved by setting $t\!<\!1$ of HDP which simply redistributes the heat back to individual graph nodes \eg, all nodes become disconnected in the extreme case. Thus, applying MaxExp and Gamma to autocorrelation matrices has a similar effect, that is, it reduces the level of correlation between features. Thus, some features  will remain `untouched' or will be modified to a lesser degree during fine-tuning if they are not related to a new task. 
We believe this is a very useful property that reduces catastrophic forgetting during fine-tuning. As it reduces the correlation between features, it should also implicitly decorrelate CNN filters.
}



\begin{figure}[t]
\vspace{-0.3cm}
\centering
\begin{subfigure}[t]{0.62\linewidth}
\centering\includegraphics[trim=0 0 0 0, clip=true, height=2.7cm]{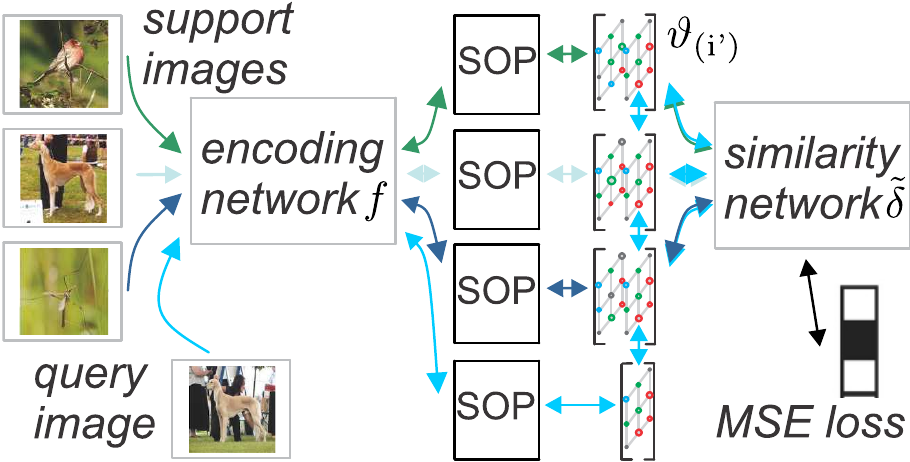}
\caption{\label{fig:princ2}}
\end{subfigure}
\begin{subfigure}[t]{0.36\linewidth}
\centering\includegraphics[trim=0 0 0 0, clip=true, height=2.7cm]{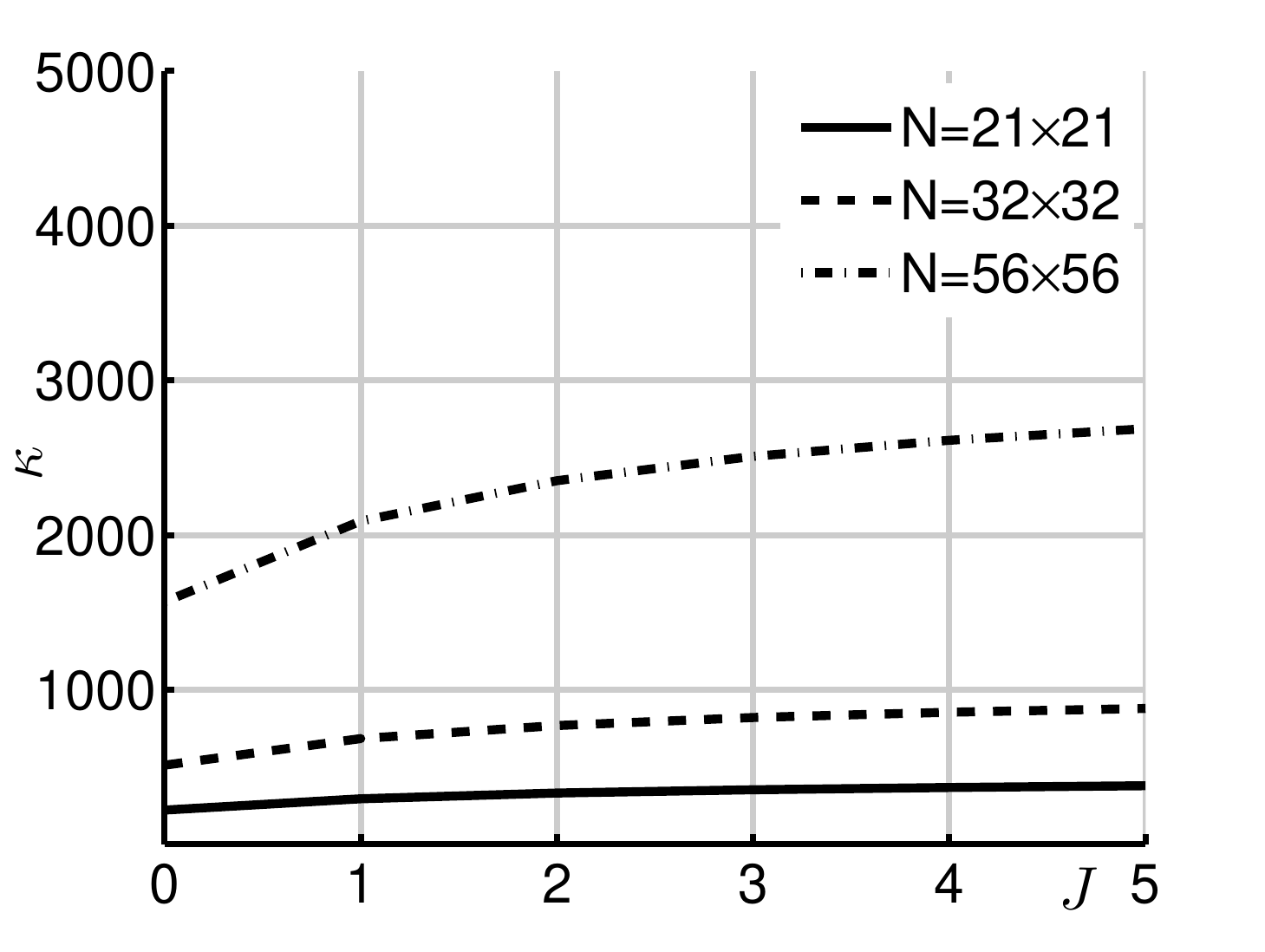}
\caption{\label{fig:poolunc}}
\end{subfigure}
\vspace{-0.2cm}
\caption{Figure \ref{fig:princ2} shows our few-shot pipeline. We use {\em feature encoding} and {\em similarity} networks. Power Normalized second-order ({\em SOP}) support-query pairs are formed and passed to the similarity net. Support-query pairs of the same class receive positive labels. 
Fig. \ref{fig:poolunc} is the $\kappa$ ratio \wrt the $J$-shot value (see Eq. \eqref{eq:support}). The  curves show that as $J$ grows (0 denotes the regular classification), the similarity learner has to memorize $\kappa\!\times$ more co-occurrence $(k,l)$ configurations if MaxExp is not used. The curves show for larger $N$ ($N\propto$ image size), not using MaxExp requires even more memorization.
}\vspace{-0.3cm}
\label{fig:fsp}
\end{figure}

\vspace{-0.1cm}
\section{Pipelines}
\label{sec:pipes}

\vspace{0.05cm}
{\noindent\textbf{Classification pipeline}} from Figure \ref{fig:princ1} has a straightforward implementation. We pass each image via ImageNet pre-trained ResNet-50, we extract feature vectors $\vPhi$ from the last conv. layer.
Formally, we have the feature encoding network  $f\!:(\mbr{3\!\times\!W\!\!\times\!H}; \mbr{|\tF|})\!\shortrightarrow\!\mbr{K\!\times\!N}$, where $W$ and $H$ are the width and height of an input image, $K$ is the length of feature vectors (number of filters), $N\!=\!N_W\!\cdot\!N_H$ is the total number of spatial locations in the last convolutional feature map. For brevity, we denote an image descriptor by $\mPhi\!\in\!\mbr{K\!\times\!N}$, where $\mPhi\!=\!f(\mX; \tF)$ for an image $\mX\!\in\!\mbr{3\!\times\!W\!\!\times\!H}$ and $\tF$ are the parameters-to-learn of the encoding network. Moreover, where stated, assume that spatial locations discussed in Section \ref{sec:cooc} are concatenated with $\mPhi$ to obtain $\bar{\mPhi}$. Subsequently, we form \revisedd{autocorrelation matrix} $\mM$ per image according to details of Section \ref{sec:cooc} which is then passed via pooling $\mygthreee{i}{\,\mM\,}$ or $\mygthreeehat{i}{\,\mM\,}$ to the classifier, in end-to-end setting.

\vspace{0.05cm}
{\noindent\textbf{Few-shot learning pipeline}}, called Second-order Similarity Network ({\em SoSN}), is shown in Figure \ref{fig:princ2}. It is  inspired by the end-to-end relationship-learning network \cite{sung2017learning} and consists of two major parts which are (i) feature encoding network and (ii) similarity network. The role of the feature encoding network is to generate convolutional feature vectors which are then used as image descriptors. The role of the similarity network is to learn the relation and compare so-called support and query image embeddings. Our work is different to the Relation Net \cite{sung2017learning} in that we apply second-order representations built from image descriptors followed by Power Normalizing functions. For instance, we construct the support and query second-order feature matrices followed by a non-linear Power Normalization unit. In SoSN, the feature encoding network remains the same as Relation Net \cite{sung2017learning}, however, the similarity network learns to compare from second- rather than the first-order statistics. %
%
%
SoSN is illustrated in Figures \ref{fig:princ1} and \ref{fig:blocks}.

We use the feature encoding network $f$ illustrated in Figure \ref{fig:blocks} (top) for which spatial locations may be concatenated with representations $\mPhi$ (see Section \ref{sec:cooc}) used by the similarity network. 

\revisedd{Figure \ref{fig:blocks} (bottom) shows the similarity network, which compares two datapoints encoded as $2\!\times\!K\!\times\!K$ dim. second-order representations, is denoted by $\tilde{\delta}\!:(\mbr{2\!\times\!K\!\times\!K}; \mbr{|\tS|})\!\shortrightarrow\!\mbr{}$, 
where $\tS$ are the parameters-to-learn of the similarity network.}


Next, let an operator $\vartheta_{\text{(i')}}\!:(\left\{\mbr{K\!\times\!N}\!\right\}^{J}\!\!\!,\,\mbr{K\!\times\!N})\!\shortrightarrow\!2\!\times\!\!K\!\times\!K$ encode a relationship between the descriptors built from the $J$-shot support images and a query image. This relationship is encoded via computing second-order statistics followed by Power Normalization and applying concatenation (inner-product, sum, subtraction, \etc are among other possible choices) to capture a relationship between features of two images. Finally, $i'\!\!\in\!\{\otimes, \otimes\text{+L}\}$ takes on one of specific operator variants defined below.  

For the $L$-way $J$-shot problem, assume some $J$ support images $\{\mX_n\}_{n\in\mathcal{J}}$ from some set $\mathcal{J}$ and their corresponding image descriptors $\{\mPhi_n\}_{n\in\mathcal{J}}$ which can be considered as a $J$-shot descriptor if stacked along the third mode. Moreover, we assume one query image $\mX^*\!$ with its image descriptor $\mPhi^*$. In general, we use `$^*\!$' to indicate query-related variables. Both the $J$-shot and the query descriptors belong to one of $L$ classes in the subset $\mathcal{C}^{\ddag}\!\equiv\!\{c_1,\cdots,c_L\}\!\subset\!\idx{C}\!\equiv\!\mathcal{C}$ chosen randomly per episode. Similarly to approach \cite{sung2017learning}, we employ the Mean Square Error (MSE) objective in our end-to-end SoSN model. Then, we perform the $L$-way $J$-shot learning by:  
%
%
%
%
%
\vspace{-0.3cm}
\begin{align}
&\!\!\argmin\limits_{\tF, \tS} \sum\limits_{c,c'\!\in\mathcal{C}^{\ddag}}
\!\left(\tilde{\delta}(\vartheta_{\text{(i')}}(\{\mPhi_n\}_{n\in\mathcal{J}_c},\mPhi^*_{q\in\mathcal{Q}: \ell(q)=c'});\tS)\! - \delta\!\left(c\!-\!c'\right)\right)^2\!\!\!,\nonumber\\
&\qquad\qquad\text{ where } \mPhi_n\!=\!f(\mX_n; \tF) \text{ and } \mPhi^*_q\!=\!f(\mX^*_q; \tF).
\end{align}
$\mathcal{J}_c$ is a randomly chosen set of $J$ support image descriptors of class $c\!\in\!\mathcal{C}^{\ddag}$, $\mathcal{Q}$ is a randomly chosen set of $L$ query image descriptors so that its consecutive elements belong to the consecutive classes in $\mathcal{C}^{\ddag}\!\equiv\!\{c_1,\cdots,c_L\}$. Lastly, $\ell(q)$ corresponds to the label of $q\!\in\!\mathcal{Q}$ \revisedd{while $\delta\!\left(x\right)\!=\!1$ if $x\!=\!0$, otherwise $\delta\!\left(x\right)\!=\!0$ (note that $c\!-\!c'\!=\!0$ if class labels $c$ and $c'$ are the same).}

\revisedd{
\vspace{0.05cm}
\noindent{\textbf{Relationship Descriptor (operator $\vartheta$). }} We consider two choices for the operator $\vartheta_{\text{(i')}}\!\left(\{\mPhi_n\}_{n\in\mathcal{J}},\mPhi^*\!\right)\!\in\!\mbr{2\!\times\!K\!\times\!K}\!$ whose role is to capture/summarize the information held in support/query image representations to pass it to the similarity network $\tilde{\delta}$ for learning similarity. Below we detail two  operators $\vartheta$ used by us.$\!\!$

\vspace{0.05cm}
\noindent{\textbf{Relationship Descriptor ($\otimes$)}} averages $J$ feature maps of support images per class followed by the outer product on the mean support and query vectors, Power Normalization and concatenation. This strategy, beneficial for $84\!\times\!84$ images, is defined as:
\begin{equation}
\!\!\!\!\!\!\!\fontsize{7}{8}\selectfont\vartheta_{\text{($\otimes$)}}\!\left(\{\mPhi_n\}_{n\in\mathcal{J}},\mPhi^*\!\right)\!=\!\left[\tG_{\text{i}}\left(\frac{1}{N}\bar{\mPhi}\bar{\mPhi}^T\!\right)\!;_1 \tG_{\text{i}}\!\left(\frac{1}{N}\mPhi^*\!\mPhi^{*T}\right)\right]\!
,\,\bar{\mPhi}\!\!=\!\!\frac{1}{J}\!\!\sum_{n\in\mathcal{J}}\!\!\mPhi_n\!,\!\!
\label{eq:concat_best}
\end{equation}
\vspace{-0.2cm}

\noindent{where} `$;_1$' is the concatenation along the channel mode, that is, $[\mX;_1\mY]\!\equiv\!\text{cat}(1,\mX,\mY)$ in the Matlab notation, $N\!=\!W\!H$, and $i\!\in\!\{\text{Gamma, MaxExp, AsinhE, SigmE, HDP}\}$.

\vspace{0.05cm}
\noindent{\textbf{Relationship Descriptor $\otimes$+L}} denotes the outer product of feature vectors per support image followed by Power Normalization of each matrix and then the average of $J$ such obtained support matrices and concatenation with the query matrix. This strategy, beneficial for large resolution images, is defined as:
\begin{equation}
\vspace{-0.1cm}
\!\!\!\!\fontsize{7}{8}\selectfont\vartheta_{\text{($\otimes$+L)}}\!\left(\{\mPhi_n\}_{n\in\mathcal{J}},\mPhi^*\!\right)\!=\!\left[\frac{1}{J}\!\sum_{n\in\mathcal{J}}\!\tG_{\text{i}}\left(\frac{1}{N}\mPhi_n\mPhi_n^T\right)\!;_1 \tG_{\text{i}}\left(\frac{1}{N}\mPhi^*\!\mPhi^{*T}\right)\right]\!.\!
\label{eq:concat_new}
\end{equation}
\vspace{-0.2cm}
}

\begin{figure}[t]
\vspace{-0.3cm}
	\centering
	\includegraphics[height=2.5cm]{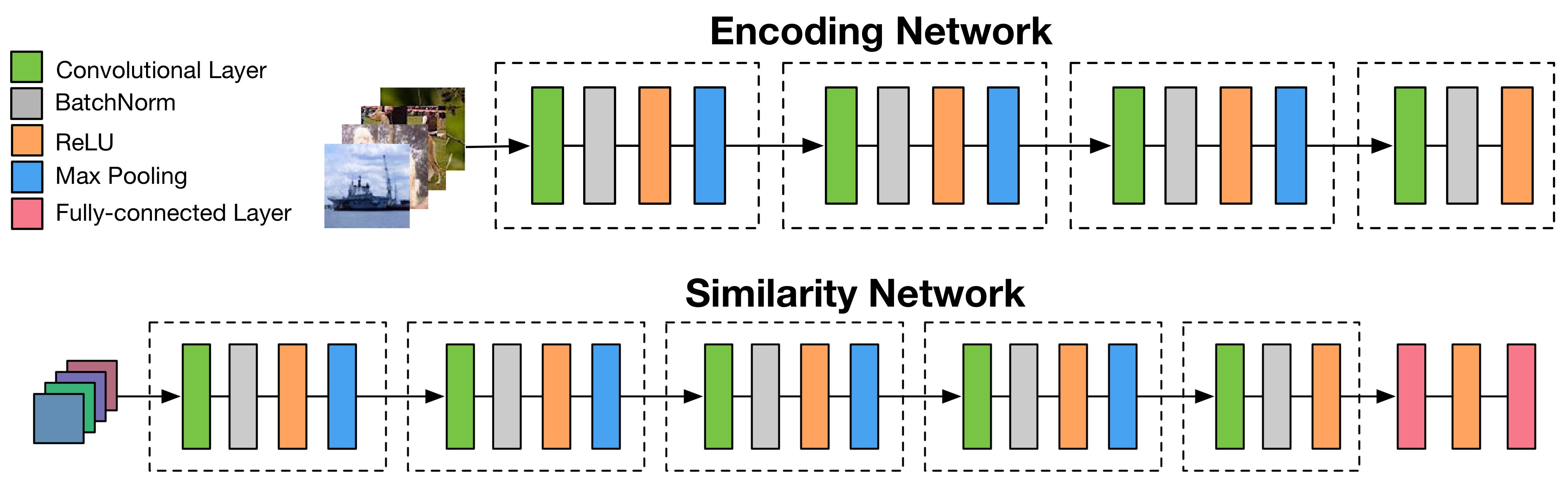}
    \vspace{-0.3cm}
	\caption{The network architecture used in our SoSN model.}
	\label{fig:blocks}
    \vspace{-0.4cm}
\end{figure}

\revisedd{
\vspace{0.05cm}
\noindent{\textbf{MaxExp in few-shot learning.}}\label{sec:max_exp_mot} Our final analysis shows that MaxExp (second-order element-wise) reduces the burstiness which is otherwise exacerbated in few-shot learning (compared to the regular classification) due to the concatenation operation in relation descriptors as explained below. 
Theorem \ref{pr:cooc} states that MaxExp (Power Normalizations in general) performs a co-occurrence detection rather than counting. For classification problems, assume a probability mass function $p_{X_{kl}}(x)\!=\!1/(N\!+\!1)$ if $x\!=\!0,\cdots,N$, $p_{X_{kl}}(x)\!=\!0$ otherwise, that tells the probability that co-occurrence between $\phi_{kn}$ and $\phi_{ln}$ given an image happened $x\!=\!0,\cdots,N$ times. Note that  classification often  depends on detecting a co-occurrence (\eg, is there a flower co-occurring with a pot?) rather than counts (\eg, how many flowers and pots co-occur?). Using second-order pooling without MaxExp requires a classifier to observe $N\!+\!1$ tr. samples of {\em flower and pot} co-occurring in quantities $0,\cdots,N$ to memorize all possible co-occurrence configurations. For $1$-shot learning, our $\vartheta$ stacks pairs of samples to compare, thus a similarity learner now has to deal with a probability mass function of ${R_{kl}}\!=\!{X_{kl}}\!+\!{Y_{kl}}$ capturing co-occurrence configurations of {\em flowers and pots} whose $\text{support}(p_{R_{kl}})\!=\!2N\!+\!1\!>\!\text{support}(p_{X_{kl}})\!=\!N\!+\!1$ as random variable $X\!=\!Y$ (same class). The same is reflected by variances \ie, $\text{var}(p_{R_{kl}})\!>\!\text{var}(p_{X_{kl}})$. For $J$-shot learning, ${R'_{kl}}\!=\!{X^{(1)}_{kl}}\!+\!\cdots\!+\!{X^{(J)}_{kl}}\!+\!{Y_{kl}}$, $X^{(j)}\!=\!Y, \forall j\!\in\!\idx{J}$, we have $\text{support}(p_{R'_{kl}})\!=\!(J\!+\!1)N\!+\!1$ and the variance grows further indicating that the similarity learner has to memorize more configurations of co-occurrence $(k,l)$ as $J$ grows. 
However, this situation is alleviated by MaxExp  whose  probability mass function yields $p_{X^\text{MaxExp}_{kl}}(x)\!=\!1/2$ if $x\!=\!\{0,1\}$, $p_{X^\text{MaxExp}_{kl}}(x)\!=\!0$ otherwise, as MaxExp detects a co-occurrence (or its lack). For $J$-shot learning, $\text{support}(p_{{R'}^\text{MaxExp}_{kl}})\!=\!J\!+\!2\!\ll\!\text{support}(p_{R'_{kl}})\!=\!(J\!+\!1)N\!+\!1$.

\begin{tcolorbox}[width=1.0\linewidth, colframe=blackish, colback=beaublue, boxsep=0mm, arc=3mm, left=1mm, right=1mm, right=1mm, top=1mm, bottom=1mm]
The  ratio
\vspace{-0.2cm}
\begin{align}
\kappa\!=\!\frac{\text{support}(p_{R'_{kl}})}{\text{support}(p_{{R'}^\text{MaxExp}_{kl}})}\!=\!\frac{(J\!+\!1)N\!+\!1}{J\!+\!2}
\label{eq:support}
\end{align}
\vspace{-0.2cm}

\noindent{shows} that the similarity learner has to memorize  more configurations for $(k,l)$ if no pooling is used relative to configurations if MaxExp is used. As $J$ and/or $N$ increase ($N\!=\!WH$ of encoder feature maps), this effect becomes more prominent.
\end{tcolorbox}
\vspace{-0.1cm}

Figure \ref{fig:poolunc} shows how $\kappa$ varies \wrt $J$ and $N$. Our modeling assumptions are very basic \eg, we use mass functions with uniform probabilities and their set support rather than variances to model the variability of co-occurrences $(k,l)$. More sophisticated choices \ie, Binomial PMF and variance-based modeling in Appendix \hyperref[{app:maxexp_fsl}]{M} lead to the same theoretical conclusions that: (i) MaxExp (and PN) benefits few-shot learning ($J\!\geq\!1$) even more than it benefits the regular classification ($J\!=\!0$) in terms of reducing possible configurations of $(k,l)$ to memorize, and (ii) for larger images (large $N$), MaxExp (and PN) must reduce a larger number of configurations of $(k,l)$ than for smaller images (smaller $N$). While classifiers and similarity learners do not memorize all configurations of $(k,l)$ thanks to their generalization ability, they learn quicker  if the number of configurations of $(k,l)$  is reduced. 
}

\vspace{-0.2cm}
\section{Experiments}
\label{sec:expts}

Below we demonstrate experimentally merits of our second-order pooling via Power Normalization functions.

\vspace{0.05cm}
\noindent{\textbf{Datasets.}} For the standard classification setting, we use five publicly available datasets and report the mean top-$1$ accuracy on them. The Flower102 dataset \cite{nilsback_flower102} is a fine-grained category recognition dataset that contains 102 categories of various flowers. 
Each class consists of between 40 and 258 images. 
The MIT67 dataset \cite{quattoni_mitindoors} contains a total of 15620 images belonging to 67 indoor scene classes. 
We follow the standard evaluation protocol, which uses a train and test split of 80\% and 20\% of images per class.
The FMD dataset contains in total 100 images per category belonging to 10 categories of materials (\eg, glass, plastic, leather) collected from the Flickr website. 
The Food-101 dataset \cite{food101}, a fine-grained collection of food images from 101 classes, has 101000 images in total and 1000 images per category. 
Finally, we report top-$1$ and -$5$ error on the ImageNet 2012 dataset \cite{ILSVRC15} with 1000 object categories. The dataset contains 1.28M images for training, 50K images for validation and 100K images for testing. As testing labels are withheld,  we follow the common practice  \cite{resnet,peihua_fast} and report the results on the validation set.

\ifdefined\arxiv
\newcommand{\SrcImgWW}{0.13}
\newcommand{\SrcImgHHH}{1.8cm}
\newcommand{\SrcImgWWW}{1.8cm}
\else
\newcommand{\SrcImgWW}{0.215}
\newcommand{\SrcImgHHH}{1.8cm}
\newcommand{\SrcImgWWW}{1.8cm}
\fi

\begin{figure}[t]
\centering
%
\comment{
\begin{subfigure}[b]{\SrcImgWW\linewidth}
\centering\includegraphics[trim=0 0 0 0, clip=true,width=\SrcImgWWW, height=\SrcImgHHH]{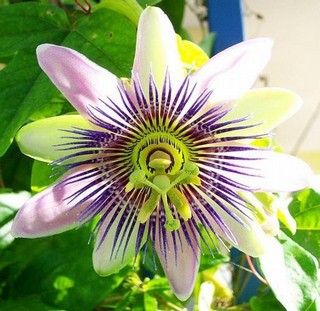}
\end{subfigure}
\begin{subfigure}[b]{\SrcImgWW\linewidth}
\centering\includegraphics[trim=0 0 0 0, clip=true,width=\SrcImgWWW, height=\SrcImgHHH]{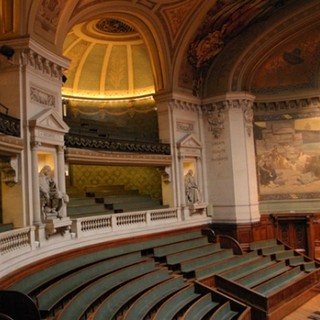}
\end{subfigure}
\begin{subfigure}[b]{\SrcImgWW\linewidth}
\centering\includegraphics[trim=0 0 0 0, clip=true,width=\SrcImgWWW, height=\SrcImgHHH]{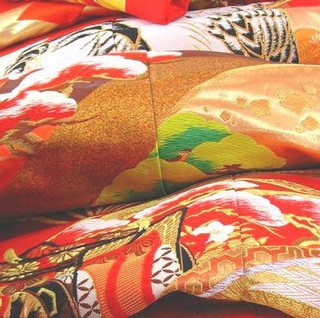}
\end{subfigure}
\begin{subfigure}[b]{\SrcImgWW\linewidth}
\centering\includegraphics[trim=0 0 0 0, clip=true,width=\SrcImgWWW, height=\SrcImgHHH]{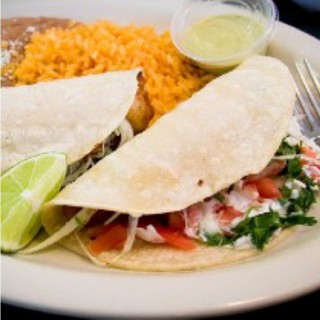}
\end{subfigure}
\\
}
\vspace{0.052cm}
\begin{subfigure}[b]{\SrcImgWW\linewidth}
\centering\includegraphics[trim=0 0 0 0, clip=true,width=\SrcImgWWW, height=\SrcImgHHH]{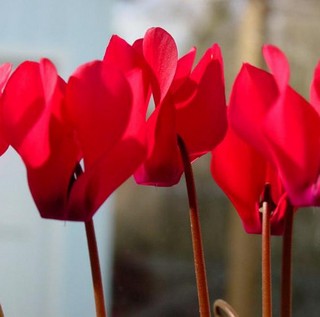}
\end{subfigure}
\begin{subfigure}[b]{\SrcImgWW\linewidth}
\centering\includegraphics[trim=0 0 0 0, clip=true,width=\SrcImgWWW, height=\SrcImgHHH]{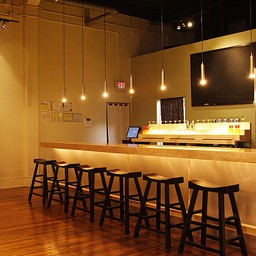}
\end{subfigure}
\begin{subfigure}[b]{\SrcImgWW\linewidth}
\centering\includegraphics[trim=0 0 0 0, clip=true,width=\SrcImgWWW, height=\SrcImgHHH]{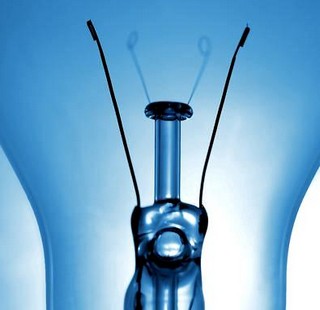}
\end{subfigure}
\begin{subfigure}[b]{\SrcImgWW\linewidth}
\centering\includegraphics[trim=0 0 0 0, clip=true,width=\SrcImgWWW, height=\SrcImgHHH]{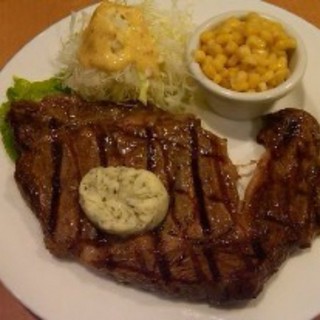}
\end{subfigure}
\\
\vspace{0.052cm}
\begin{subfigure}[b]{\SrcImgWW\linewidth}
\centering\includegraphics[trim=0 0 0 0, clip=true,width=\SrcImgWWW, height=\SrcImgHHH]{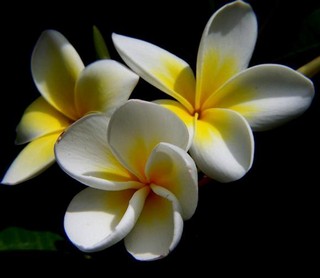}
\end{subfigure}
\begin{subfigure}[b]{\SrcImgWW\linewidth}
\centering\includegraphics[trim=0 0 0 0, clip=true,width=\SrcImgWWW, height=\SrcImgHHH]{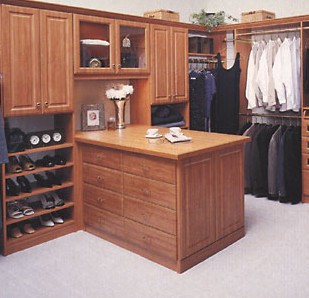}
\end{subfigure}
\begin{subfigure}[b]{\SrcImgWW\linewidth}
\centering\includegraphics[trim=0 0 0 0, clip=true,width=\SrcImgWWW, height=\SrcImgHHH]{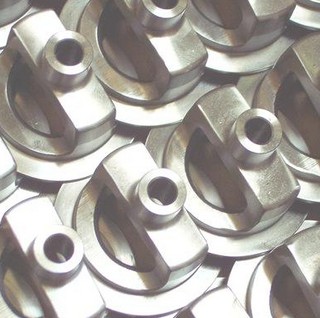}
\end{subfigure}
\begin{subfigure}[b]{\SrcImgWW\linewidth}
\centering\includegraphics[trim=0 0 0 0, clip=true,width=\SrcImgWWW, height=\SrcImgHHH]{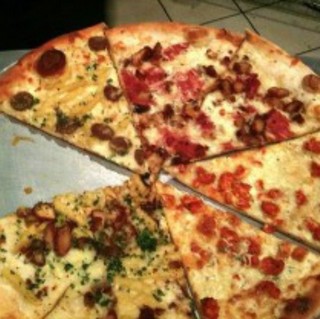}
\end{subfigure}
%
%
\caption{Each column shows examples of images from the Flower102, MIT67, FMD and Food-101 dataset, respectively.
}\vspace{-0.2cm}
\label{fig:datasets}
\end{figure}

\begin{figure}[t]
\centering
%
\begin{subfigure}[b]{\SrcImgWW\linewidth}
\centering\includegraphics[trim=0 0 0 0, clip=true,width=\SrcImgWWW, height=\SrcImgHHH]{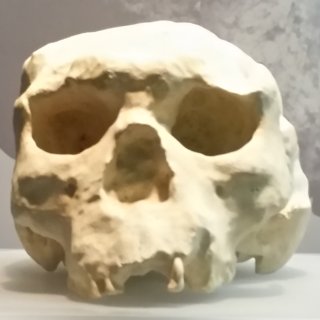}
\end{subfigure}
\begin{subfigure}[b]{\SrcImgWW\linewidth}
\centering\includegraphics[trim=0 0 0 0, clip=true,width=\SrcImgWWW, height=\SrcImgHHH]{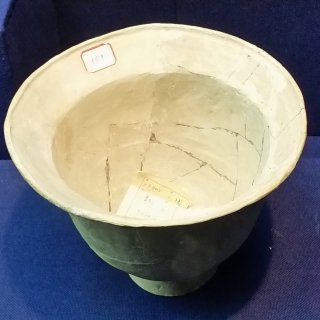}
\end{subfigure}
\begin{subfigure}[b]{\SrcImgWW\linewidth}
\centering\includegraphics[trim=0 0 0 0, clip=true,width=\SrcImgWWW, height=\SrcImgHHH]{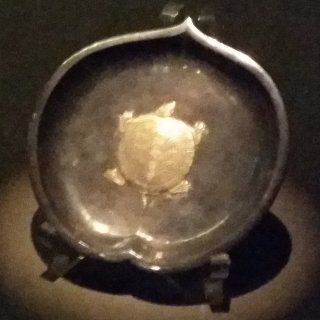}
\end{subfigure}
\begin{subfigure}[b]{\SrcImgWW\linewidth}
\centering\includegraphics[trim=0 0 0 0, clip=true,width=\SrcImgWWW, height=\SrcImgHHH]{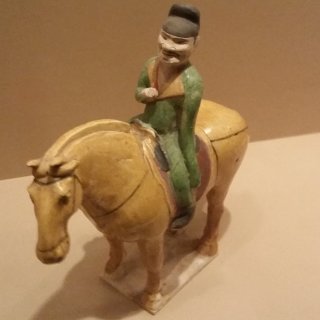}
\end{subfigure}
\\
\vspace{0.052cm}
\begin{subfigure}[b]{\SrcImgWW\linewidth}
\centering\includegraphics[trim=0 0 0 0, clip=true,width=\SrcImgWWW, height=\SrcImgHHH]{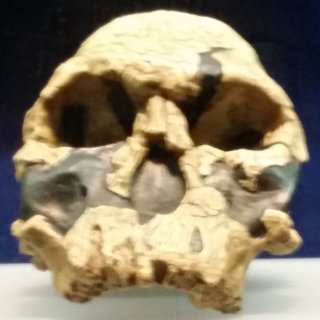}
\end{subfigure}
\begin{subfigure}[b]{\SrcImgWW\linewidth}
\centering\includegraphics[trim=0 0 0 0, clip=true,width=\SrcImgWWW, height=\SrcImgHHH]{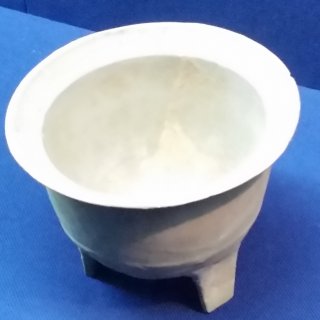}
\end{subfigure}
\begin{subfigure}[b]{\SrcImgWW\linewidth}
\centering\includegraphics[trim=0 0 0 0, clip=true,width=\SrcImgWWW, height=\SrcImgHHH]{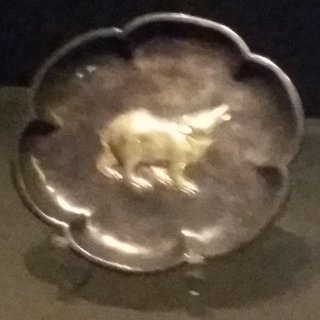}
\end{subfigure}
\begin{subfigure}[b]{\SrcImgWW\linewidth}
\centering\includegraphics[trim=0 0 0 0, clip=true,width=\SrcImgWWW, height=\SrcImgHHH]{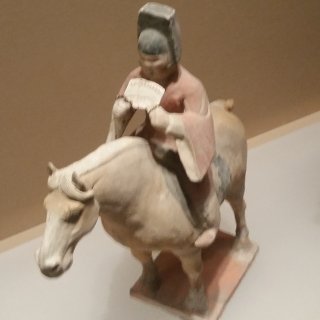}
\end{subfigure}
%
%
\caption{Each column shows examples of fine-grained objects from the Open MIC dataset which look similar but belong to  different classes.
}\vspace{-0.3cm}
\label{fig:openmicfg}
\end{figure}

For few-shot classification problems,  we also use four publicly available datasets and report the mean top-$1$ accuracy for so-called $L$-way $J$-shot problems \cite{sung2017learning}. The \textit{mini}ImageNet dataset \cite{vinyals2016matching} is a standard benchmark for evaluating few-shot learning approaches. It consists of 60000 RGB images from 100 classes. 
We follow 
\cite{vinyals2016matching} and use 64 classes for training, 16 classes for validation, remaining 20 classes for testing, and we use images of size $84\!\times\!84$. 
We also investigate larger sizes, \eg~$224\!\times\!224$, as our few-shot learning SoSN model can use richer spatial information from larger images to obtain high-rank autocorrelation matrices without a need to modify the similarity network to larger feature inputs. Moreover, we investigate the fine-grained Flower102 and Food-101 datasets in the few-shot learning scenario. We take the first 80 categories of each dataset for training/validation and the remaining 21 and 22 for testing, respectively. Lastly, we introduce few-shot learning protocols on a fine-grained Open MIC dataset \cite{me_museum} detailed next.

\vspace{0.05cm}
\noindent\textbf{Open MIC}. The Open Museum Identification Challenge \cite{me_museum} contains photos of various exhibits \eg, paintings, timepieces, sculptures, glassware, relics, science exhibits, natural history pieces, ceramics, pottery, tools and indigenous crafts, captured from 10 museum exhibition spaces according to which this dataset is divided into 10 sub-problems. In total, it has 1--20 images per class and 866 diverse classes, many of which are fine-grained \eg, fossils, jewelery, cultural relics, as shown in Figure \ref{fig:openmicfg}. The within-class images undergo various geometric and photometric distortions as the data was captured by wearable cameras. Thus, Open MIC challenges one-shot learning algorithms. We combine ({\em shn+hon+clv}), ({\em clk+gls+scl}), ({\em sci+nat}) and ({\em shx+rlc}) into sub-problems {\em p1}, $\!\cdots$, {\em p4}. We form 12 possible pairs in which sub-problem $x$ is used for training and $y$ for testing (x$\rightarrow$y). Our first protocol aims at the generalization from one task to another task, thus we use the target part of Open MIC with the 12 above sub-problems. Our second protocol aims at the generalization from one domain to another domain, thus we use the source and target parts of Open MIC for training/testing on 10 original sub-problems.

\begin{table}[t]
\centering
\begin{tabular}{l l|c|c|}
Method & & \multicolumn{2}{c|}{top-$1$ accuracy}  \\ 
\hline
{\em Second-order Bag-of-Words}\kern-0.6em  &\cite{me_tensor} & \multicolumn{2}{c|}{90.2} \\
{\em Factors of Transferability}\kern-0.6em &\cite{carlson_cnn} & \multicolumn{2}{c|}{91.3} \\
{\em Reversal-inv. Image Repr.}\kern-0.6em  &\cite{rrir} & \multicolumn{2}{c|}{94.0} \\
{\em Optimal two-stream fusion}\kern-0.6em  &\cite{two-stream} & \multicolumn{2}{c|}{94.5} \\
{\em Neural act. constellations}\kern-0.6em &\cite{nacc} & \multicolumn{2}{c|}{95.3} \\
\end{tabular}\\
%
\vspace{0.2cm}
\begin{tabular}{l|c|c|}
Method       & \kern-0.3em AlexNet\kern-0.3em & \kern-0.3em ResNet-50\kern-0.3em \\
\hline
{\em Baseline}      & 82.00  & 94.06   \\ 
{\em FOP}         	& 85.40  & 94.08   \\ 
{\em FOP+AsinhE}    & 85.64  & 94.60   \\ 
{\em SOP}         	& 87.20  & 94.70   \\ 
{\em SOP+AsinhE}    & 88.40  & 95.12   \\ 
{\em SOP+SC+AsinhE} & 90.70  & 95.74   \\ 
{\em SOP+SC+SigmE}  & 91.71  & \textbf{96.78}\\
\hline
{\em SOP+SC+Spec. Gamma} & -  & 96.88   \\ 
{\em SOP+SC+Spec. HDP}  & -  & 97.05\\
{\em SOP+SC+Spec. MaxExp}  & -  & \textbf{97.28}\\
{\em SOP+SC+Spec. MaxExp(F)}  & -  & \textbf{97.62}
\end{tabular}
\caption{The Flower102 dataset. The bottom part shows our results for AlexNet and ResNet-50. The horizontal lines separate first- and second-order element-wise pooling, and the spectral pooling. The top part of the table lists state-of-the-art results from the literature.}
\label{tab:flower102}
\end{table}
\begin{table}[t]
\centering
\begin{tabular}{l l c || c c|}
Method  && acc. & Method  & acc.\\
\hline
{\em Baseline} 											&									 & 81.9  & {\em SOP+SC+SigmE} & 87.5 \\
{\em SOP} 													&       					& 83.0  & {\em SOP+SC+Spec. MaxExp} & 87.8 \\
{\em Kern. Pool.}\kern-0.6em&\cite{cui2017kernel}       & 85.5 & \kern-0.3em{\em SOP+SC+Spec. MaxExp(F)}\kern-0.3em & \textbf{88.4}\\
\end{tabular}
\caption{The Food-101 dataset. Our (right) \vs other methods (left).}
\label{tab:food-101}
\vspace{-0.35cm}
\end{table}

\revised{
\vspace{0.05cm}
\noindent\textbf{Graph datasets.} We use seven popular graph benchmarks MUTAG, PTC, PROTEINS, NCI1, COLLAB, REDDIT-BINARY and REDDIT-MULTI-5K \cite{pytorch_geom}. 
MUTAG contains mutable molecules, 188 chemical compounds, and 7 node labels. PTC includes a number of carcinogenicity  tasks  for  toxicology  prediction, it contains 417 compounds from four species, and 18 node labels. PROTEINS are sets  of  proteins from  the  BRENDA database \cite{brenda} with 3 node labels. NCI is a collection of datasets for anticancer activity prediction with 37 node labels. Finally,  COLLAB, REDDIT-BINARY and REDDIT-MULTI-5K represent social networks.
}

\vspace{0.05cm}
\noindent{\textbf{Experimental setup (classification setting).}} 
For Flower102 \cite{nilsback_flower102}, we extract 12 cropped 224$\times$224 patches per image and use mini-batch of size 5 to fine-tune the ResNet-50 model \cite{resnet} pre-trained on ImageNet 2012 \cite{ILSVRC15}. We obtain 2048 dim. $12\!\times\!7\!\times\!7$ conv. feature vectors from the last conv. layer for our second-order pooling layer.
For MIT67 \cite{quattoni_mitindoors}, we resize original images to 336$\times$336 and use mini-batch of size 32, then fine-tune it on the ResNet-50 model \cite{resnet} pre-trained on the Places-205 dataset \cite{places_dataset}. With $336\!\times\!336$ image size, we obtain 2048 dim. $11\!\times\!11$ conv. feature vectors from the last conv. layer for our second-order pooling layer. 
For FMD \cite{fmd} and Food-101 \cite{food101}, we resize images to $448\!\times\!448$, use mini-batch of size 32 and fine-tune ResNet-50 \cite{resnet} pre-trained on ImageNet 2012 \cite{ILSVRC15}. We use the 2048 dim. $14\!\times\!14$ conv. feature vectors from the last conv. layer. 
For ImageNet 2012 \cite{ILSVRC15}, we crop $224\!\times\!224$ patches and allow left-right flip. We obtain 2048 dim. $12\!\times\!7\!\times\!7$ conv. feature vectors. 
For ResNet-50, we fine-tune all layers for $\sim$20 epochs with learning rates 1e-6--1e-4. We use  
RMSprop \cite{rmsprop} with the moving average $0.99$. 
%
 Where stated, we use AlexNet \cite{krizhevsky_alexnet} with fine-tuned last two conv. layers. We use $256$ dim. $6\!\times\!6$ conv. feature vectors from the last conv. layer.

\vspace{0.05cm}
\noindent{\textbf{Experimental setup (few-shot setting).}} For {\em mini}ImageNet, we use standard 5-way 1-shot and 5-way 5-shot protocols. For every training/testing episode, we randomly select 5/3 query samples per class. We average over 600 episodes to obtain results. 
We use the initial learning rate $1e\!-\!3$ and train the model with $200K$ episodes. For Flower102 and Food-101, we follow the same setting and train models with $40K$ and $200K$ episodes.
For Open MIC, we mean-center images per sub-problem. As some classes have less than 5 images, we use the 5- to 90-way 1-shot learning protocol. During training, to form an episode, we select 1 image for the support set and another 2 images for the query set per class. During testing, we use the same number of support/query samples in every episode and compute the accuracy over 1000 episodes. 
We use the initial learning rate $1e\!-\!4$ and train over $15K$ episodes. 
For all datasets, we resize images to $84\!\times\!84$ or $224\!\times\!224$ where stated.

\revised{
\vspace{0.05cm}
\noindent{\textbf{Experimental setup (graph classification).}} We use 
the Graph Isomorphism Network ({\em GIN0}) from package \cite{pytorch_geom}.
We remove the classifier to produce covariances. We tune the neighborhood size, hidden units and the number of layers between 10--50, 16--128 and  2--5. We use the Adam optimizer with learning rate $1e\!-\!2$. 
}

\vspace{0.05cm}
\noindent{\textbf{Our methods.}} We evaluate  the generalizations of MaxExp and Gamma, 
that is Sigmoid ({\em SigmE}) and  Arcsin hyperbolic ({\em AsinhE}) pooling functions. We focus mainly on our second-order representation ({\em SOP}) but we also occasionally report results for the first-order approach ({\em FOP}). For the baseline, we use the classifier on top of the {\em fc} layer ({\em Baseline}). The hyperparameters 
are selected via cross-validation. The use of spatial coordinates and spectral operators is indicated by ({\em SC}) and ({\em Spec.}), resp. For few-shot setting, we evaluate second-order similarity network variants ({\em SoSN}($\otimes$)) and ({\em SoSN}($\otimes$+L)) defined in Eq. \eqref{eq:concat_best} and \eqref{eq:concat_new}.

\vspace{-0.1cm}
\subsection{Evaluations}
\label{sec:eval}

Below we investigate  element-wise and spectral pooling in fine-grained classification, few-shot learning and graph classification.

\vspace{0.05cm}
\noindent{\textbf{Fine-grained datasets.}} The Flower102 dataset is evaluated in Table \ref{tab:flower102} which shows that AlexNet performs worse than ResNet-50, which is consistent with the literature. For the standard ResNet-50 fine-tuned on Flower102, we obtain 94.06\% accuracy. The first-order Average and AsinhE pooling ({\em FOP}) and ({\em FOP+AsinhE}) score 94.08 and 94.6\% accuracy. The second-order pooling ({\em SOP+AsinhE}) outperforms ({\em FOP+AsinhE}). For element-wise operators, we obtain the best result of \textbf{96.78}\% for the second-order representation combined with spatial coordinates and SigmE ({\em SOP+SC+SigmE}), which is $\sim$2.7\% higher than our baseline. In contrast, a recent more complex method  \cite{nacc} obtained 95.3\% accuracy. Our scores highlight that capturing co-occurrences of visual features and passing them via a well-defined Power Normalization function such as SigmE works well for our fine-grained problem. We attribute the good performance of SigmE to its ability to act as a detector of co-occurrences. The role of the Hyperbolic Tangent non-linearity (popular in deep learning) may be explained by its similarity to SigmE.

\begin{table}[t]
\centering
\begin{tabular}{l l|c|}
Method && top-$1$ accuracy\\
\hline
{\em CNNs with Deep Supervision} & \cite{places_mit_more} & 76.1\\
{\em Places-205}\kern-0.6em&\cite{places_mit} & 80.9 \\
{\em Deep Filter Banks}\kern-0.6em &\cite{cimpoi2015deep} & 81.0 \\
{\em Spectral Features}\kern-0.6em&\cite{khan2017scene} & 84.3 \\
\hline
{\em Baseline}  && 84.0\\
\hline
{\em SOP+AsinhE} && 85.3 \\
{\em SOP+SigmE} && 85.6\\ 
{\em SOP+SC+AsinhE} && 85.9 \\
{\em SOP+SC+SigmE} && \textbf{86.3}\\
\hline
{\em SOP+SC+Spec. Gamma} && 86.4\\
{\em SOP+SC+Spec. HDP} && 86.3\\
{\em SOP+SC+Spec. MaxExp} && 86.5\\
{\em SOP+SC+Spec. MaxExp(F)} && \textbf{86.8}
\end{tabular}
\caption{The MIT67 dataset. The bottom part shows our results for ResNet-50 pre-trained on the Places-205 dataset. The horizontal lines separate first- and second-order element-wise pooling, and the spectral pooling. The top part are state-of-the-art results from the literature. }
\label{tab:mit67}
\vspace{-0.35cm}
\end{table}

\ifdefined\arxiv
\newcommand{\PlotWW}{0.245}
\newcommand{\PlotHHH}{4cm}
\newcommand{\PlotWWW}{3.8cm}
\else
\newcommand{\PlotWW}{0.494}
\newcommand{\PlotHHH}{4cm}
\newcommand{\PlotWWW}{4.45cm}
\fi

\ifdefined\arxiv
\begin{figure}[!b]
\else
\begin{figure}[b]
\fi
\vspace{-0.4cm}
\centering
\begin{subfigure}[t]{\PlotWW\linewidth}
\centering\includegraphics[trim=0 0 0 0, clip=true,width=\PlotWWW]{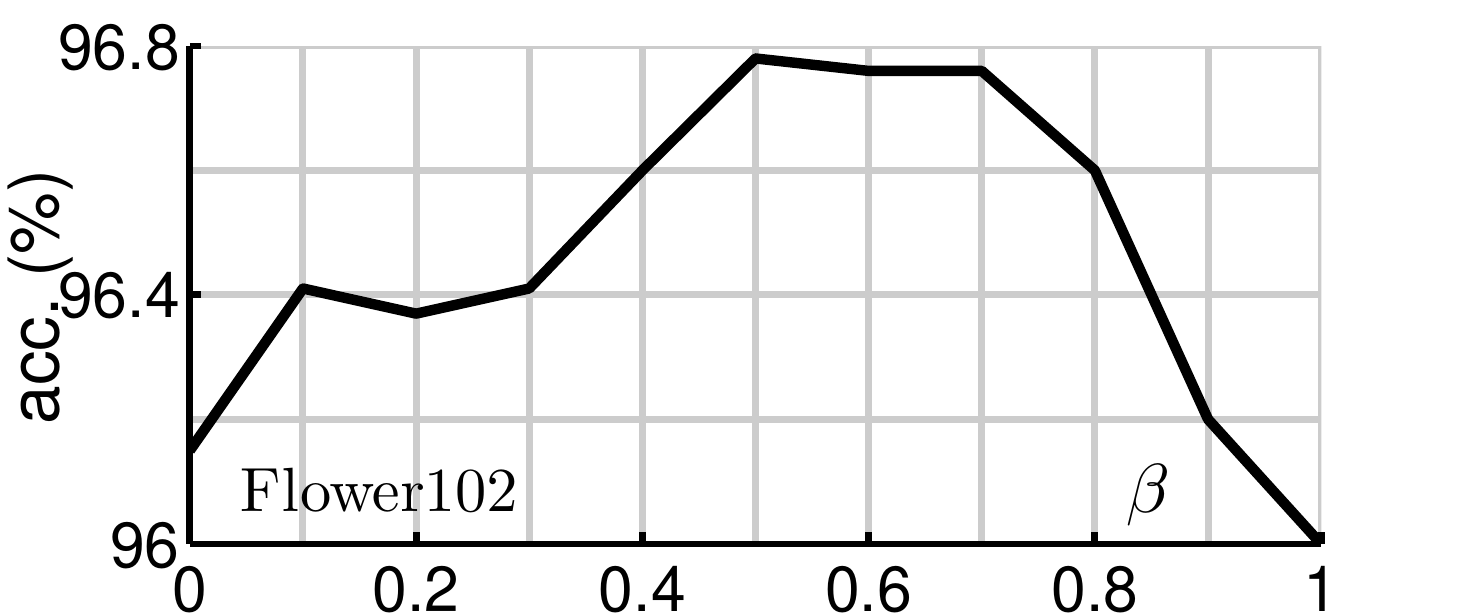}\vspace{-0.2cm}
\caption{\label{fig:eval1}}
\end{subfigure}
\begin{subfigure}[t]{\PlotWW\linewidth}
\centering\includegraphics[trim=0 0 0 0, clip=true,width=\PlotWWW]{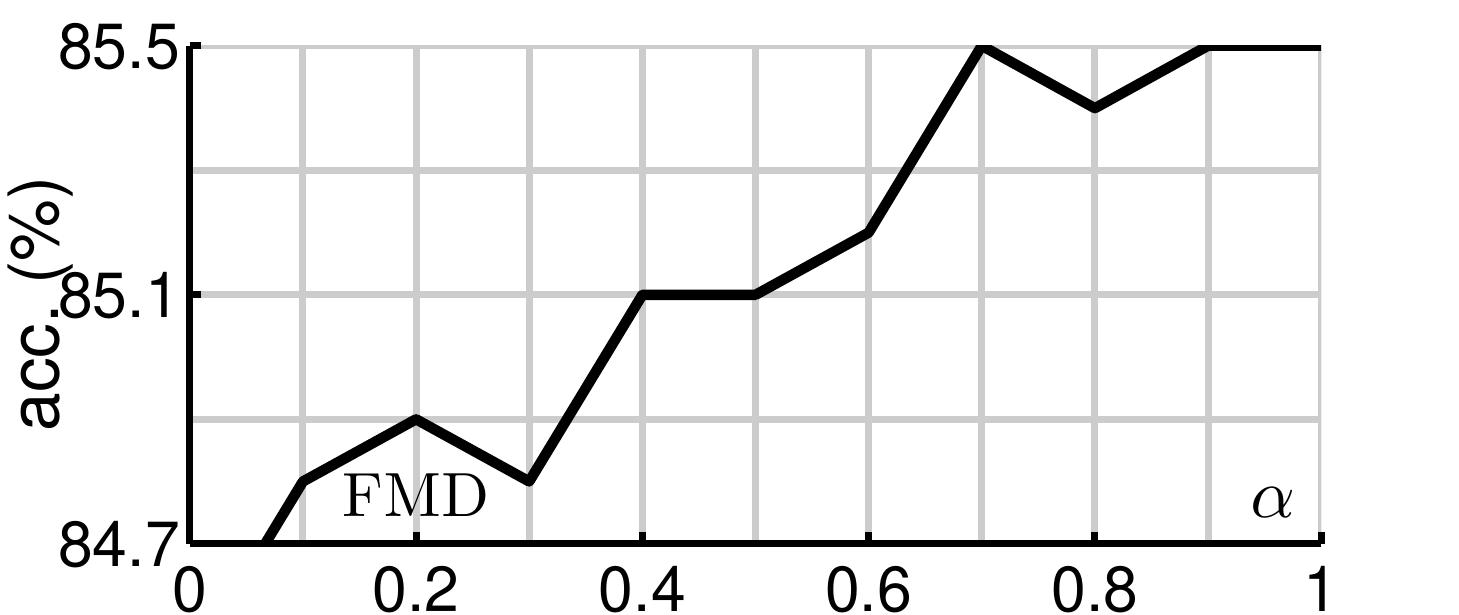}\vspace{-0.2cm}
\caption{\label{fig:eval2}}
\end{subfigure}
\ifdefined\arxiv\else\\\fi
\begin{subfigure}[t]{\PlotWW\linewidth}
\centering\includegraphics[trim=0 0 0 0, clip=true,width=\PlotWWW]{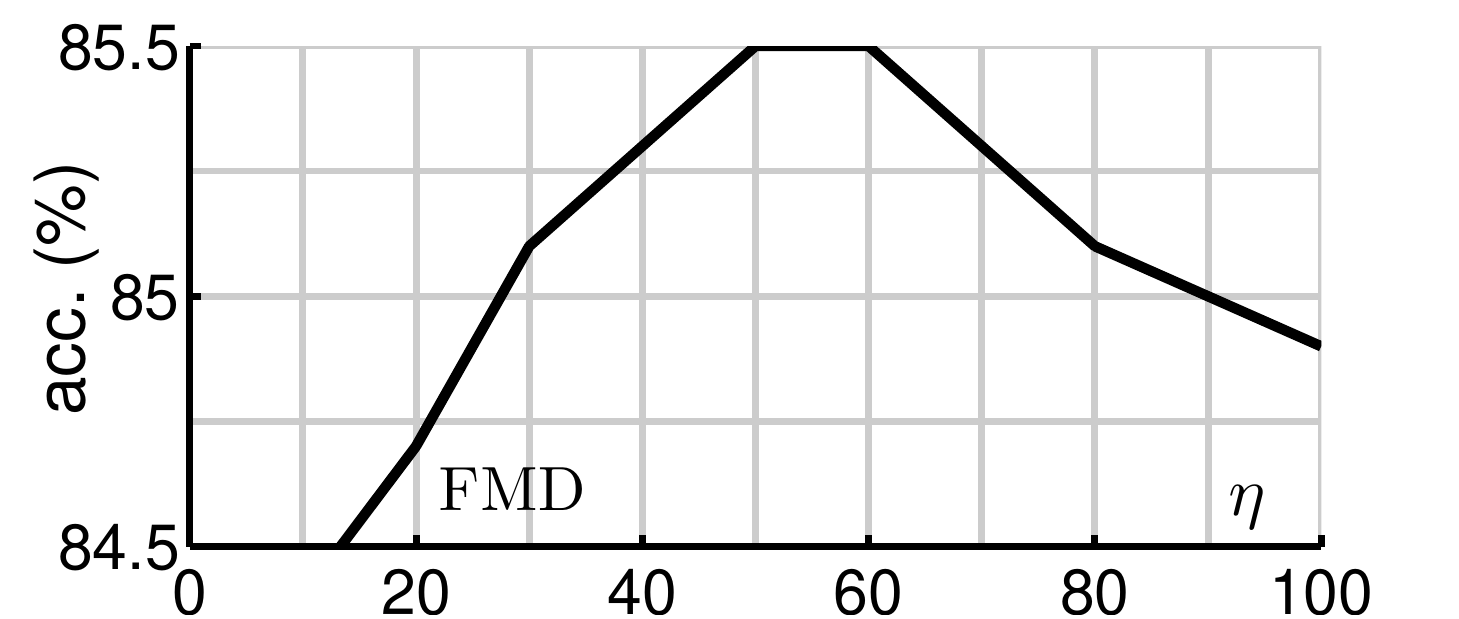}\vspace{-0.2cm}
\caption{\label{fig:eval3}}
\end{subfigure}
\begin{subfigure}[t]{\PlotWW\linewidth}
\centering\includegraphics[trim=0 0 0 0, clip=true,width=\PlotWWW]{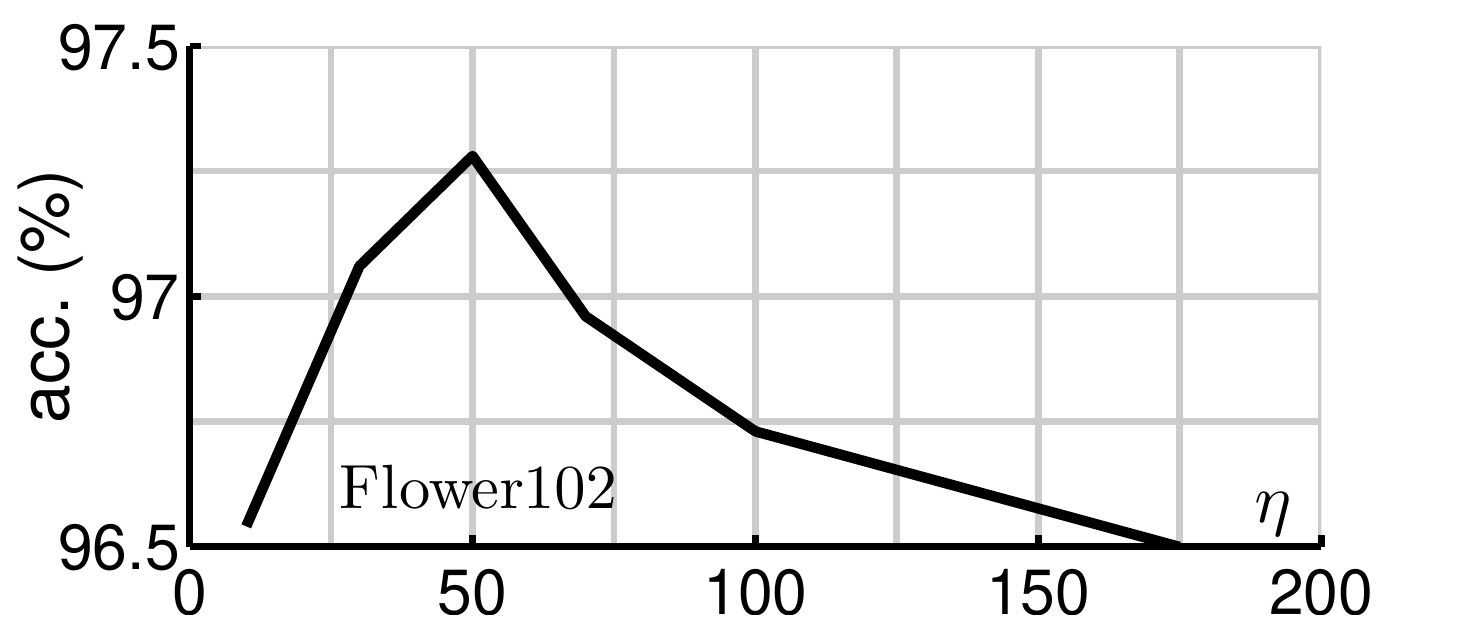}\vspace{-0.2cm}
\caption{\label{fig:eval4}}
\end{subfigure}
\vspace{-0.3cm}
\caption{Performance w.r.t. hyperparameters. Figures \ref{fig:eval1} and \ref{fig:eval2}: $\beta$-centering on Flower102 and $\alpha$ for spatial coordinate encoding on FMD. Figures \ref{fig:eval3} and \ref{fig:eval4}: the accuracy w.r.t. the $\eta'\!$ and $\eta\!$ parameters given SigmE and the spectral MaxExp.
}
\label{fig:sens_pars}
\vspace{-0.3cm}
\end{figure}

\revised{
Furthermore, we note that the spectral approaches ({\em Spec.}) outperform element-wise second-order pooling. However, the differences are not drastically large.
Spectral Gamma and MaxExp, both with spatial coordinates, denoted as ({\em SOP+SC+Spec. Gamma}) and ({\em SOP+SC+Spec. MaxExp}), perform similarly to the time-reversed Heat Diffusion Process (HDP) which experimentally validates their similarity to HDP.
Finally,  our fast spectral MaxExp ({\em SOP+SC+Spec. MaxExp(F)}) yields \textbf{97.62}\% accuracy. We expect that the back-prop. through the fast spectral MaxExp is more stable than the back-prop. via SVD which we use for other methods.
}

For Food-101, we apply our  second-order representations ({\em SOP+SC+SigmE}) and ({\em SOP+SC+Spec. MaxExp}) 
and obtain \textbf{87.5}\% and \textbf{87.8}\% accuracy. \revised{Furthermore, the fast spectral MaxExp ({\em SOP+SC+Spec. MaxExp(F)}) yields \textbf{88.4}\% accuracy.} In contrast, a recent more involved kernel pooling \cite{cui2017kernel} reports 85.5\% accuracy while the baseline approach scores only 81.9\% in the same testbed. This shows the strength of our approach on fine-grained problems.

\noindent{\textbf{Scene recognition.}} Next, we validate our approach on MIT67--a larger dataset for scene recognition.  Table \ref{tab:mit67} shows that all second-order approaches ({\em SOP}) outperform the standard ResNet-50 network ({\em Baseline}) pre-trained on the Places-205 dataset and fine-tuned on MIT67. Moreover, ({\em SigmE})  yields marginally better results than ({\em AsinhE}). Using spatial coordinates ({\em SC}) also results in additional gain in the classification performance. The second-order representation combined with  spatial coordinates and SigmE pooling ({\em SOP+SC+SigmE}) yields \textbf{86.3}\% accuracy and outperforms our baseline and \cite{khan2017scene} by 2.3\% and 2\%, respectively. 

\revised{
For spectral operators, we observe a similar trend to results on Flower-102. Spectral Gamma with spatial coordinates ({\em SOP+SC+Spec. Gamma}) marginally outperforms HDP ({\em SOP+SC+Spec. HDP}). We expect this is due to the inversion of eigenvalues in HDP which makes it unstable for rank-deficient autocorrelation matrices. We note that the fast spectral MaxExp ({\em SOP+SC+Spec. MaxExp}) outperforms other spectral operators  due to a more stable back-propagation it enjoys.
}

\noindent{\textbf{Material classification.}} For the FMD dataset (material/texture recognition), 
 Table \ref{tab:fmd} demonstrates that our second-order representation ({\em SOP+SC+SigmE}) scores \textbf{85.5}\% accuracy and outperforms our baseline approach by 2.1\%. \revised{Moreover, using the fast spectral MaxExp ({\em SOP+SC+Spec. MaxExp}) yields \textbf{86.4}\% accuracy.} We note that our approach and the baseline use the same testbed, that is, the only difference is the addition of our second-order representations, spatial coordinates and Power Normalization. 

\revised{\noindent{\textbf{ImageNet 2012.}} Our fast spectral MaxExp ({\em SOP+SC+Spec. MaxExp(F)}) is shown to outperform ({\em SOP+SC+Spec. Gamma}) in Table \ref{tab:in}. This is expected due to instabilities in back-propagation through SVD. Our method is also comparable with the recent approaches while enjoying strong theoretical connections to HDP.}

\begin{table}[t]
\centering
\begin{tabular}{l l c || c c|}
Method  && acc. & Method  & acc.\\
\hline
{\em IFV+DeCAF}\kern-0.6em&\cite{cimpoi2014describing} & 65.5  & {\em Baseline} & 83.4 \\
{\em FV+FC+CNN}\kern-0.6em&\cite{cimpoi2015deep}       & 82.2  & {\em SOP+SC+AsinhE} & 85.0 \\
{\em SMO Task}\kern-0.6em&\cite{zhang2016integrating}  & 82.3  & {\em SOP+SC+SigmE} & \textbf{85.5}\\
													&													   &       & {\em \fontsize{7}{8}\selectfont SOP+SC+Spec. MaxExp(F)} & \textbf{86.4}
\end{tabular}
\vspace{-0.2cm}
\caption{The FMD dataset. Our (right) \vs other methods (left).}
\label{tab:fmd}
\vspace{-0.1cm}
\end{table}
\begin{table}[t]
\centering
\begin{tabular}{l l|c|c|}
Method && top-$1$ err& top-$5$ err\\
\hline
{\em ResNet-50} & \cite{resnet} & 24.7& 7.8 \\
{\em MPN-COV} & \cite{peihua_fast} & 22.73& 6.54 \\
{\em Newton-Schulz}\kern-0.9em&\kern-0.8em\cite{lin2017improved,peihua_fast}\kern-0.3em & 22.14 & 6.22 \\
\hline
{\em Baseline}  && 25.0 & 8.1\\
\hline
{\em SOP+SC+Spec. Gamma} && 22.51 & 6.85\\
{\em SOP+SC+Spec. MaxExp(F)}\kern-0.6em && \textbf{22.05} & \textbf{6.04}
\end{tabular}
\caption{\revised{The ImageNet 2012 dataset. The top part of the table lists state-of-the-art results from the literature.}}
\label{tab:in}
\vspace{-0.35cm}
\end{table}
\begin{table}[t]
\vspace{-0.1cm}
\centering
\hspace{-0.3cm}
\setlength{\tabcolsep}{0.10em}
\renewcommand{\arraystretch}{0.70}
\fontsize{8.5}{9}\selectfont
\begin{tabular}{l c|c|c|c|c}
\multirow{2}{*}{Model} & & Image & Fine & \multirow{2}{*}{1-shot} & \multirow{2}{*}{5-shot} \\
 & & Res. & Tune & &  \\ \hline
\textit{Meta-Learn LSTM} & \kern-3.6em\cite{METALEARNLSTM}\kern-0.6em & \multirow{4}{*}{$84\!\times\!84$} & N & $43.44 \!\pm\! 0.7$ & $60.60 \!\pm\! 0.7$ \\
\textit{Prototypical Net} & \kern-3.6em\cite{snell2017prototypical}\kern-0.6em & & N & $49.42 \!\pm\! 0.7$ & $68.20 \!\pm\! 0.6$ \\ 
\textit{MAML} & \kern-3.6em\cite{finn2017model}\kern-0.6em & & Y & $48.70 \!\pm\! 1.8$ & $63.11 \!\pm\! 0.9$ \\ 
\textit{Relation Net} & \kern-3.6em\cite{sung2017learning}\kern-0.6em & & N & $50.44 \!\pm\! 0.8$ & $65.32 \!\pm\! 0.7$  \\ 
\hline
\multicolumn{2}{l|}{\textit{SoSN($\otimes$)} (no Power Norm.)} & \multirow{8}{*}{$84\!\times\!84$} & \multirow{8}{*}{N} & $50.88 \!\pm\! 0.8$ & $66.71 \!\pm\! 0.6$  \\
\multicolumn{2}{l|}{\textit{SoSN($\otimes$)+AsinhE}} &  &  & $52.10 \!\pm\! 0.8$ & $67.79 \!\pm\! 0.6$ \\
\multicolumn{2}{l|}{\textit{SoSN($\otimes$)+SigmE}} &  &  & ${ 52.96\!\pm\! 0.8}$ & $68.63 \!\pm\! 0.6$  \\ 
\multicolumn{2}{l|}{\textit{SoSN($\otimes$+L)+AsinhE}} &  &  & $54.36\!\pm\! 0.7$ & ${ 70.80 \!\pm\! 0.6}$  \\ 
\multicolumn{2}{l|}{\textit{SoSN($\otimes$+L)+SC+AsinhE}} &  &  & $55.01\!\pm\! 0.6$ & ${ 71.23 \!\pm\! 0.6}$  \\ 
\multicolumn{2}{l|}{\textit{SoSN($\otimes$+L)+SigmE}} &  &  & $54.19\!\pm\! 0.6$ & ${ 70.94 \!\pm\! 0.6}$  \\ 
\multicolumn{2}{l|}{\textit{SoSN($\otimes$+L)+SC+SigmE}} &  &  & $55.36\!\pm\! 0.7$ & ${ 71.23 \!\pm\! 0.6}$  \\ 
%
%
\hline
\multicolumn{2}{l|}{\textit{SoSN($\otimes$)+SigmE}} & \multirow{11}{*}{$224\!\times\!224$} & \multirow{8}{*}{N} & $60.35 \!\pm\! 0.7$ & $74.01 \!\pm\! 0.6$  \\
\multicolumn{2}{l|}{\textit{SoSN($\otimes$+L)+AsinhE}} &   &  & $60.32 \!\pm\! 0.6$ & $75.10 \!\pm\! 0.5$  \\
\multicolumn{2}{l|}{\textit{SoSN($\otimes$+L)+SC+AsinhE}} &   &  & $60.42 \!\pm\! 0.6$ & $75.89 \!\pm\! 0.6$  \\
\multicolumn{2}{l|}{\textit{SoSN($\otimes$+L)+SigmE}} &  &  & $60.35 \!\pm\! 0.7$ & $75.02 \!\pm\! 0.5$  \\
\multicolumn{2}{l|}{\textit{SoSN($\otimes$+L)+SC+SigmE}} &   &  & $60.49 \!\pm\! 0.6$ & $75.54 \!\pm\! 0.5$  \\\cline{1-2}\cline{5-6}
\multicolumn{2}{l|}{{\fontsize{7}{8}\selectfont\textit{SoSN($\otimes$+L)+SC+Spec. Gamma}}} &   &  & $60.38 \!\pm\! 0.5$ & $75.70 \!\pm\! 0.5$  \\
\multicolumn{2}{l|}{{\fontsize{7}{8}\selectfont\textit{SoSN($\otimes$+L)+SC+Spec. MaxExp(F)}}} &   &  & $60.40 \!\pm\! 0.6$ & $75.75 \!\pm\! 0.5$  \\\cline{4-4}
\multicolumn{2}{l|}{{\fontsize{7}{8}\selectfont\textit{SoSN($\otimes$+L)+Pretr.+SC+SigmE}}} &   & \multirow{3}{*}{\kern-0.4em Y \kern-0.4em} & $60.51 \!\pm\! 0.5$ & $75.56 \!\pm\! 0.6$  \\
\multicolumn{2}{l|}{{\fontsize{7}{8}\selectfont\textit{SoSN($\otimes$+L)+Pretr.+SC+Spec. Gamma}}} &   &  & $60.95 \!\pm\! 0.6$ & $76.20 \!\pm\! 0.5$  \\
\multicolumn{2}{l|}{{\fontsize{7}{8}\selectfont\textit{SoSN($\otimes$+L)+Pretr.+SC+Spec. MaxExp(F)}}} &   &  & $61.32 \!\pm\! 0.6$ & $76.45 \!\pm\! 0.5$  \\
%
\hline
\end{tabular}
\caption{Evaluations on the \textit{mini}ImageNet dataset (5-way acc. given). Refer to \cite{sung2017learning} for references to baselines.}
\label{table2}
\vspace{-0.35cm}
\end{table}

\vspace{0.05cm}
\noindent{\textbf{Performance w.r.t. hyperparameters.}} Figure \ref{fig:eval1} demonstrates that $\beta$-centering has a positive impact on image classification with ResNet-50. This strategy, detailed in Section \ref{sec:cooc}, is trivial to combine with our pooling. Figure \ref{fig:eval2} shows that setting non-zero $\alpha$, which lets encode spatial coordinates according to Eq. \eqref{eq:encode_sc}, brings additional gain in accuracy at no extra cost. Figure \ref{fig:eval3} demonstrates that over 1\% accuracy can be gained by tuning our SigmE pooling. Moreover, Figure \ref{fig:eval4} shows that the spectral MaxExp can yield further gains over element-wise SigmE and MaxExp for carefully chosen $\eta$. Lastly, we observed that our spectral and element-wise MaxExp converged in 3--12 and 15--25 iterations, respectively. This shows that both spectral and element-wise pooling have their strong and weak points.

\revisedd{
\vspace{0.05cm}
\noindent{\textbf{Timing and variance in SPN.}} Below we present timing experiments performed on a TitanX GPU with the use of autograd profiler of PyTorch. To evaluate the forward runtime $t_{fw\!+SPN}$ of Fast Spectral MaxExp, Newton-Schulz iter. (the approximate matrix square root) and the Generalized Spectral Power Normalization, we applied the {\em record\_function()} subroutine of the profiler. To time the autograd-based back-propagation runtime through each of these operators, we firstly recorded the total GPU time $t_{tot\!+SPN}$ per method before removing these operators from the code and recording the total GPU time $t_{tot}$. Thus, we obtain the backward runtime $t_{bw\!+SPN}\!=\!t_{tot\!+SPN}\!-\!t_{tot}\!-\!t_{fw\!+SPN}$. We normalize results by the number of batches and datapoints per mini-batch.

\vspace{-0.1cm}
\begin{tcolorbox}[width=1.0\linewidth, colframe=blackish, colback=beaublue, boxsep=0mm, arc=3mm, left=1mm, right=1mm, right=1mm, top=1mm, bottom=1mm]
Figure \ref{fig:tim1} shows the speed of forward/backward passes of Fast Spectral MaxExp from Alg. \ref{code:exp_sqr} \wrt $\eta$. The plot shows that the runtime grows sublinearly \wrt $\eta$ which is the major advantage over the Generalized Spectral Power Normalization ({\em GSPN}) that uses SVD (runtime scales with $d^\omega$ where $2\!<\!\omega\!<\!2.376$) and the Newton-Schulz iterations which realize only the approximate matrix square root ($\gamma\!=\!0.5$) whose quality depends on the number of iterations $k$ (runtime scales linearly \wrt $k$). 
\end{tcolorbox}
\vspace{-0.2cm}

Figures \ref{fig:tim2} and \ref{fig:tim3} compare the forward and backward speeds which show that our Fast Spectral MaxExp is faster than the Newton-Schulz iter. and GSPN. Notably, the backward pass appears $\sim\!2\!\times$ more costly than the forward pass in all cases. We suspect this is due to the autograd recomputing intermediate variables from the forward pass in the backward pass. Thus, the runtime of optimized backward pass can be halved. Another downside of the Newton-Schulz iter. and GSPN compared to the Fast Spectral MaxExp was their larger memory footprint.}

\begin{table}[t]
\centering
\setlength{\tabcolsep}{0.10em}
\renewcommand{\arraystretch}{0.70}
\begin{tabular}{l|c|c}
Model & 1-shot & 5-shot \\ \hline
\textit{Relation Net}  & $68.52 \pm 0.94\%$ & $81.11 \pm 0.66\%$  \\ 
\hline
%
\textit{SoSN+SigmE}   & $77.62 \pm 0.88\%$ & $88.60 \pm 0.53\%$  \\
\textit{SoSN+SC+SigmE}   & ${\bf 78.50 \pm 0.75\%}$ & ${\bf 89.95 \pm 0.62\%}$  \\
\textit{\fontsize{7}{8}\selectfont SoSN+SC+Spec. Gamma}  & $78.45 \pm 0.72\%$ & $89.75 \pm 0.55\%$  \\
\textit{\fontsize{7}{8}\selectfont SoSN+SC+Spec. MaxExp(F)}  & $78.51 \pm 0.68\%$ & $89.70 \pm 0.58\%$  \\
\textit{\fontsize{7}{8}\selectfont SoSN+Pretr.+SC+ SigmE}   & $78.55 \pm 0.70\%$ & $89.99 \pm 0.80\%$  \\
\textit{\fontsize{7}{8}\selectfont SoSN+Pretr.+SC+Spec. Gamma}   & $79.35 \pm 0.72\%$ & $90.65 \pm 0.84\%$  \\
\textit{\fontsize{7}{8}\selectfont SoSN+Pretr.+SC+Spec. MaxExp(F)}   & ${\bf 80.01 \pm 0.74\%}$ & ${\bf 91.50 \pm 0.65\%}$  \\

\hline
\end{tabular}
\caption{Evaluations on the Flower102 dataset (5-way acc. given). For SoSN, we evaluate only our ({\em ($\otimes$+L)+SigmE}) aggregator (equiv. to ({\em ($\otimes$)+SigmE}) for 1-shot problems) given images of res. $224\!\times\!224$.}
\label{tab:fewshot-flower}
\end{table}
\begin{table}[t]
\centering
\setlength{\tabcolsep}{0.10em}
\renewcommand{\arraystretch}{0.70}
\begin{tabular}{l|c|c}
Model & 1-shot & 5-shot \\ \hline
\textit{Relation Net}  & $36.89 \pm 0.72\%$ & $49.07 \pm 0.65\%$  \\ 
\hline
%
\textit{SoSN+SigmE}   & $42.44 \pm 0.75\%$ & $60.70 \pm 0.65\%$  \\
\textit{SoSN+SC+SigmE}   & $42.80 \pm 0.72\%$ & $60.95 \pm 0.68\%$  \\
\textit{\fontsize{7}{8}\selectfont SoSN+SC+Spec. Gamma}  & $42.30 \pm 0.72\%$ & $60.71 \pm 0.60\%$  \\
\textit{\fontsize{7}{8}\selectfont SoSN+SC+Spec. MaxExp(F)}  & $42.40 \pm 0.75\%$ & $60.70 \pm 0.59\%$  \\
\textit{\fontsize{7}{8}\selectfont SoSN+Pretr.+SC+SigmE}   & $42.60 \pm 0.69\%$ & $60.85 \pm 0.70\%$  \\
\textit{\fontsize{7}{8}\selectfont SoSN+Pretr.+SC+Spec. Gamma}   & $43.72 \pm 0.52\%$ & $62.02 \pm 0.68\%$  \\
\textit{\fontsize{7}{8}\selectfont SoSN+Pretr.+SC+Spec. MaxExp(F)}   & ${\bf 45.21 \pm 0.62\%}$ & ${\bf 64.50 \pm 0.61\%}$  \\
\hline
\end{tabular}
\caption{Evaluations on the Food-101 dataset (5-way acc. given). For SoSN, we evaluate only our ({\em ($\otimes$+L)+SigmE}) aggregator (equiv. to ({\em ($\otimes$)+SigmE}) for 1-shot problems) given images of res. $224\!\times\!224$.}
\label{tab:fewshot-food}
\vspace{-0.35cm}
\end{table}
\begin{table}[t]
\hspace{-0.3cm}
\setlength{\tabcolsep}{0.10em}
\renewcommand{\arraystretch}{0.70}
\begin{tabular}{l l|c|c|c|c|}
Method && \kern-0.6em MUTAG\kern-0.6em & \kern-0.6em PTC\kern-0.6em & \kern-0.0em PROTEINS \kern-0.3em & \kern-0.6emNCI1\kern-0.6em\\
\hline
{\em S\textsuperscript{2}GC} \kern-0.6em&\cite{zhu2021simple} & 85.1$\!\pm\!$7.4 &  - & 75.5$\!\pm\!$4.1 & - \\
{\em DGCNN} \kern-0.6em&\cite{DGCNN} & 85.8$\!\pm\!$1.7 & 58.6$\!\pm\!$2.5& 75.5$\!\pm\!$0.9 & 74.4$\!\pm\!$0.5 \\
{\em GCAPS-CNN}\kern-0.6em&\cite{capsule0} & - & 66.0$\!\pm\!$5.9 & 76.4$\!\pm\!$4.2 & 82.7$\!\pm\!$2.4 \\
{\em BC+CAPS} & \cite{capsule1} & 88.9$\!\pm\!$5.5 & 69.0$\!\pm\!$5.0 & 74.1$\!\pm\!$3.2 & 65.9$\!\pm\!$1.1 \\
\hline
\hline
{\em GIN0}  &\cite{pytorch_geom}& 86.1$\!\pm\!$5.8 & 56.5$\!\pm\!$6.8 & 72.2$\!\pm\!$4.9 & 77.9$\!\pm\!$2.5 \\
\hline
{\em SOP} && 86.2$\!\pm\!$5.2 & 57.5$\!\pm\!$10.1 & 71.2$\!\pm\!$4.9 & 78.3$\!\pm\!$3.0\\
{\em SOP+AsinhE} && 86.2$\!\pm\!$6.2 & 58.2$\!\pm\!$5.8 & 72.0$\!\pm\!$3.8 & 79.5$\!\pm\!$2.0 \\
{\em SOP+SigmE} && 87.8$\!\pm\!$6.1 & 58.4$\!\pm\!$5.5 & 71.8$\!\pm\!$3.6 & 79.6$\!\pm\!$1.9\\
\kern-0.6em{\em \fontsize{8}{9}\selectfont SOP+Newton-Schulz}\kern-0.6em & \kern-1.5em\cite{lin2017improved,peihua_fast}\kern-0.6em& 86.2$\!\pm\!$6.1  & 59.3$\!\pm\!$5.8 & 75.3$\!\pm\!$2.8 & 79.9$\!\pm\!$2.3 \\
{\em SOP+Spec. Gamma}\kern-2.2em && 86.5$\!\pm\!$6.1 & 61.5$\!\pm\!$3.8 & 75.7$\!\pm\!$4.0 & 80.0$\!\pm\!$2.1 \\
{\em SOP+Spec. HDP}\kern-2.2em && 86.2$\!\pm\!$7.9 & 61.2$\!\pm\!$6.4 & 75.5$\!\pm\!$2.8 & 79.6$\!\pm\!$2.0 \\
{\em SOP+Spec. MaxExp}\kern-2.2em && 86.8$\!\pm\!$6.6 & 61.9$\!\pm\!$2.4 & \textbf{76.8}$\!\pm\!$2.9 & 79.8$\!\pm\!$2.4 \\
{\em SOP+Spec. MaxExp(F)}\kern-1.8em && \textbf{88.9}$\!\pm\!$5.8 & \textbf{68.3}$\!\pm\!$9.3 & 76.2$\!\pm\!$2.8 & \textbf{80.3}$\!\pm\!$2.4 \\
\end{tabular}

\vspace{0.1cm}
%
%
\begin{tabular}{l l|c|c|c|c|}
Method && \kern-0.1em MUTAG\kern-0.1em & \kern-0.04em COLLAB\kern-0.04em & \kern0.00em {\fontsize{8}{9}\selectfont REDDIT-B}\kern-0.0em & \kern0.00em {\fontsize{8}{9}\selectfont REDDIT-5K}\kern0.0em\\
\hline
{\em S\textsuperscript{2}GC} \kern-0.6em&\cite{zhu2021simple} & 85.1$\!\pm\!$7.4 & 80.2$\!\pm\!$1.2 & - & - \\
{\em DGCNN} \kern-0.6em&\cite{DGCNN} & 85.8$\!\pm\!$1.7 & 73.8$\!\pm\!$0.5 & 76.0$\!\pm\!$1.7 & 48.7$\!\pm\!$4.5 \\
{\em GCAPS-CNN}\kern-0.6em&\cite{capsule0} & -          & 77.7$\!\pm\!$2.5 & 87.6$\!\pm\!$2.5 & 50.1$\!\pm\!$1.7 \\
{\em AWE} & \cite{pmlr-v80-ivanov18a} & - & 71.0$\!\pm\!$1.5 & 83.0$\!\pm\!$2.7 & 54.7$\!\pm\!$2.9 \\
\hline
\hline
{\em GIN0}  &\cite{pytorch_geom}& 89.2$\!\pm\!$4.8 & 79.9$\!\pm\!$1.7 & 92.1$\!\pm\!$1.9 & 55.5$\!\pm\!$2.1 \\
\hline
{\em SOP} && 89.9$\!\pm\!$9.5 & 80.6$\!\pm\!$1.2 & 91.8$\!\pm\!$2.1 & 55.6$\!\pm\!$3.0\\
{\em SOP+AsinhE} && 92.0$\!\pm\!$4.9 & 81.2$\!\pm\!$1.6 & 92.5$\!\pm\!$1.7 & 56.8$\!\pm\!$2.1 \\
{\em SOP+SigmE} && 91.5$\!\pm\!$5.6 & 81.4$\!\pm\!$1.8 & 92.4$\!\pm\!$1.9 & \textbf{57.1}$\!\pm\!$1.8\\
\kern-0.6em{\em \fontsize{8}{9}\selectfont SOP+Newton-Schulz}\kern-0.6em & \kern-1.5em\cite{lin2017improved,peihua_fast}\kern-0.6em& 93.7$\!\pm\!$8.5  & 81.2$\!\pm\!$8.9 & 92.3$\!\pm\!$1.8 & \textbf{57.1}$\!\pm\!$1.9 \\
{\em SOP+Spec. MaxExp(F)}\kern-1.8em && \textbf{94.7}$\!\pm\!$5.0 & \textbf{81.7}$\!\pm\!$1.7 & \textbf{92.6}$\!\pm\!$1.6 & \textbf{57.1}$\!\pm\!$1.9 \\
\end{tabular}
\caption{\revisedd{Classification with GIN0 and various pooling methods on ({\em top}) MUTAG, PTC, PROTEINS and NCI1 (the validation split was used only for validation) and ({\em bottom}) MUTAG, COLLAB, REDDIT-BINARY and REDDIT-MULTI-5K (after obtaining hyperparameters, the train and validation splits were combined for retraining). State-of-the-art results from the literature are listed directly below `Method'.}}
\label{tab:graph}
\vspace{-0.35cm}
\end{table}
%

\revisedd{Finally, we investigate if SPN reduces the correlation between features that $\mM$ represents. It is known that an isotropic Gaussian can be thought of as being constructed from uncorrelated features. Thus, we measure the variance of $j\!=\!5$ leading eigenvalues of $\mM$ passed via MaxExp. Figure \ref{fig:var1} shows that as $\eta$ increases, the variance of the leading eigenvalues decreases. However, while leading eigenvalues for smaller $\eta$ become equalized and pulled towards the value of one, non-leading eigenvalues may remain unaffected which increases the variance. This behavior is desired as only a certain $j$ leading eigenvalues correspond to the signal and the remaining non-leading eigenvalues represent the noise (by analogy to the Principal Component Analysis).}

\begin{figure}[t]
\centering
\hspace{-0.3cm}
\begin{subfigure}[t]{0.495\linewidth}
\centering\includegraphics[trim=0 0 0 0, clip=true, height=\PowH]{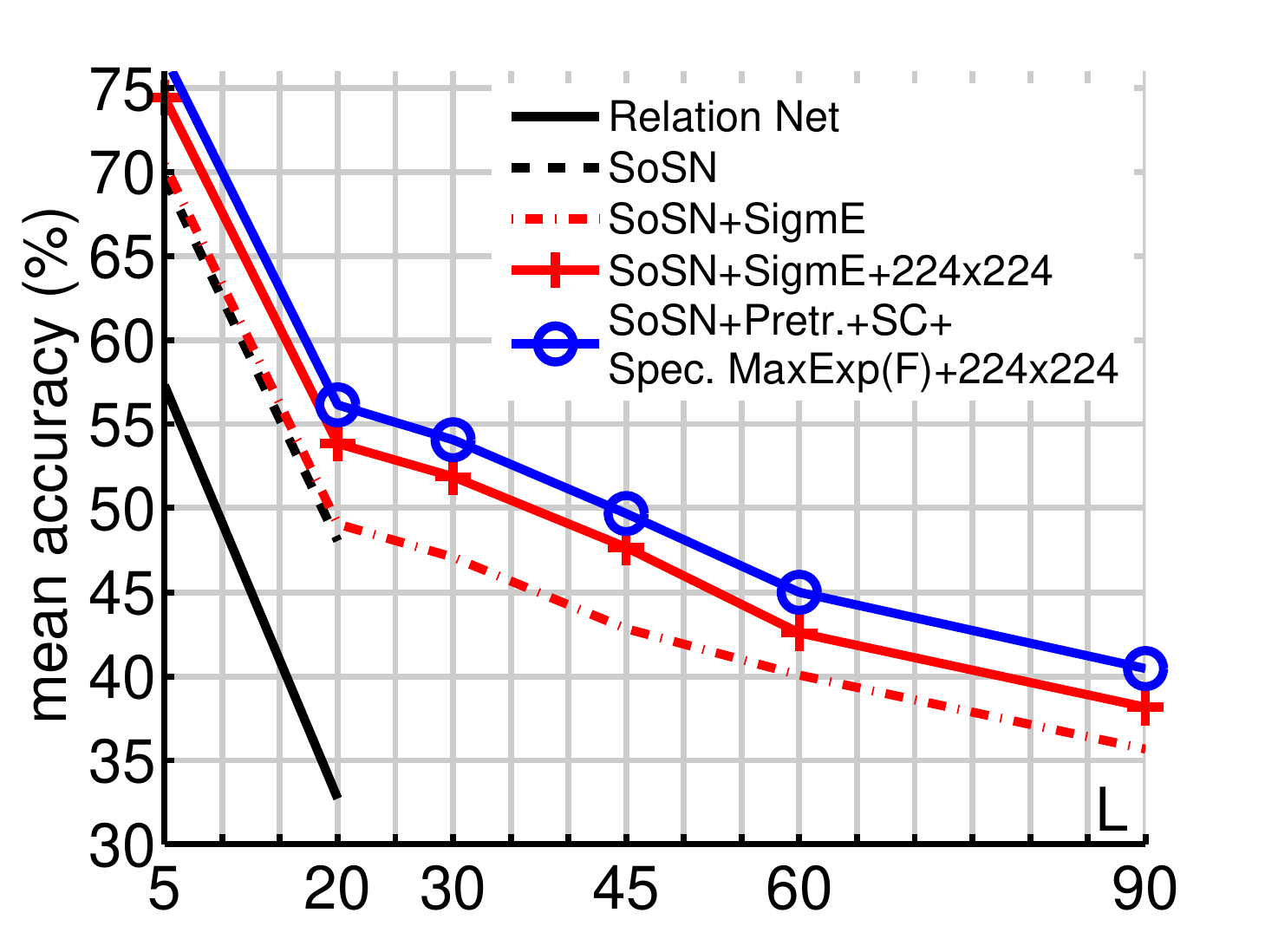}
\caption{Protocol I$\!\!\!\!\!\!\!\!$}\label{fig:openmic1}
\end{subfigure}
\begin{subfigure}[t]{0.495\linewidth}
\centering\includegraphics[trim=0 0 0 0, clip=true, height=\PowH]{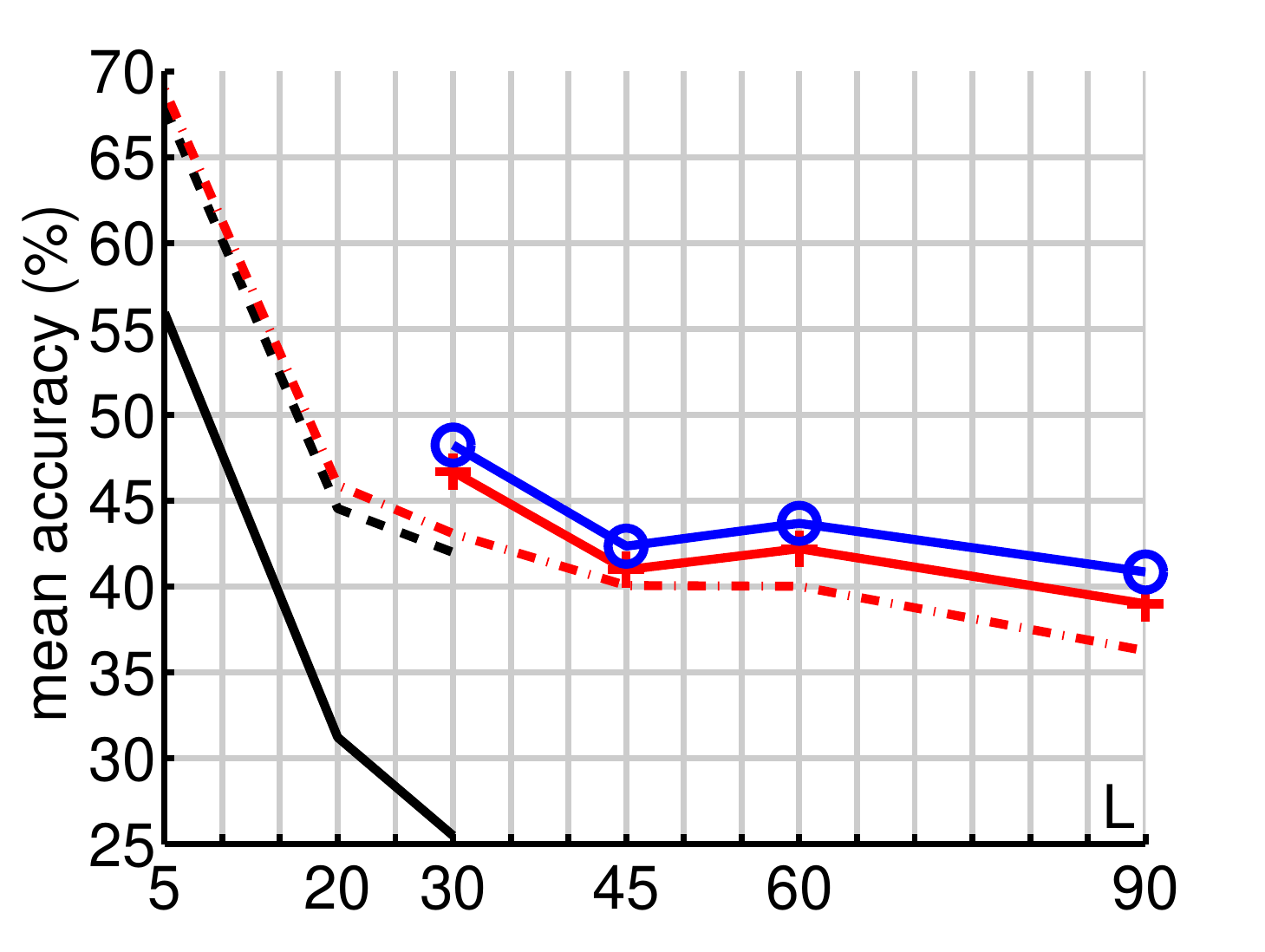}
\caption{Protocol II$\!\!\!\!\!\!\!\!$}\label{fig:openmic2}
\end{subfigure}
\caption{Evaluations on the Open MIC dataset. Fig. \ref{fig:openmic1} shows 1-shot mean accuracy on Protocol I. Each point is averaged over 12 possible testing results from train-test pairs {\em x$\!\rightarrow$y} where we use p1: shn+hon+clv, p2: clk+gls+scl, p3: sci+nat, p4: shx+rlc. Fig. \ref{fig:openmic2} shows 1-shot mean accuracy on Protocol II. Training is performed on source images and testing on target images for every exhibition. Then every point in the plot is average over results on 10 exhibitions.}
\vspace{-0.35cm}
\label{fig:openmic}
\end{figure}

\vspace{0.05cm}
\noindent{\textbf{Few-shot learning.}} Below we evaluate our SoSN model and compare it against state-of-the-art models \eg, Relation Net \cite{sung2017learning}.

For {\em mini}ImageNet, 
Table \ref{table2} shows that our method outperforms others on 5-way 1- and 5-shot learning. For experiments with image size of $84\!\times\!84$, our SoSN model ({\em SoSN($\otimes$)+SigmE}) achieved $\sim\!2.5\%$ and $\sim\!3.3\%$  higher accuracy than Relation Net \cite{sung2017learning}. 
Our SoSN models also outperformed Prototypical Net by $0.43$--$0.75\%$ accuracy on the 5-way 5-shot protocol. For $224\!\times\!224$ images on ({\em SoSN($\otimes$+L)+SigmE}), the accuracies on both protocols increase by $\sim\!5.5\%$ and $\sim\!6\%$ over $84\!\times\!84$ counterpart, which shows that SoSN benefits from larger image sizes as second-order matrices used in Eq. \eqref{eq:concat_new} become full-rank (the similarity network needs no modifications). \revised{Finally, Table \ref{table2} shows gains on spectral methods ({\em Spec.})  pre-trained on Food-101 ({\em Pretr.}) \ie, ({\em Pretr.+SC+Spec. Gamma}) and ({\em Pretr.+SC+Spec. MaxExp(F)}) outperforms non-spectral ({\em SC+SigmE}). However, spectral methods without pre-training ({\em SC+Spec. Gamma}) and ({\em SC+Spec. MaxExp(F)}) and even pre-trained non-spectral ({\em Pretr.+SC+SigmE}) fail to bring any benefits over non-spectral ({\em SC+SigmE}). 

\vspace{-0.1cm}
\begin{tcolorbox}[width=1.0\linewidth, colframe=blackish, colback=beaublue, boxsep=0mm, arc=3mm, left=1mm, right=1mm, right=1mm, top=1mm, bottom=1mm]
We believe this is an important finding--we benefit from pre-training only when we use spectral Power Norms. As spectral Power Norms act as time-reversed HDP, the level of correlation between co-occurring features is reduced. This reduces so-called catastrophic forgetting on the dataset used for pre-training. 
\end{tcolorbox}
\vspace{-0.2cm}

Our best results for ({\em Pretr.+SC+Spec. MaxExp(F)}) gave \textbf{61.32}\% and \textbf{76.45}\% accuracy on 5-way 1- and 5-shot learning.

\vspace{0.05cm}
\noindent{\textbf{Fine-grained few-shot learning.}} 
%
For Open MIC, Figure \ref{fig:openmic1} introduces results for the protocol that tests the generalization from
one task to another task ({\em Protocol I}). As only 1-shot protocol can be applied to this dataset, we use ({\em SoSN($\otimes$)}) which is equivalent to ({\em SoSN($\otimes$+L)}) for 1-shot problems and we denote it by ({\em SoSN}). All our methods outperform the Relation Net \cite{sung2017learning}. 
For $5$- and $20$-way, Relation Net scores 55.45 and 31.58\%. In contrast, our ({\em SoSN}) scores 68.23 and 45.31\%, resp. Our ({\em SoSN+SigmE}) scores 69.06 and 46.53\%, resp. Increasing resolution to $224\!\times\!224$ on ({\em SoSN+SigmE}) yields $\bf{74.43\%}$ and $\bf{53.84\%}$ accuracy. 

\ifdefined\arxiv
\begin{figure}[!t]
\else
\begin{figure}[t]
\fi
\vspace{-0.2cm}
\centering
\begin{subfigure}[t]{\PlotWW\linewidth}
\centering\includegraphics[trim=0 0 0 0, clip=true,width=\PlotWWW]{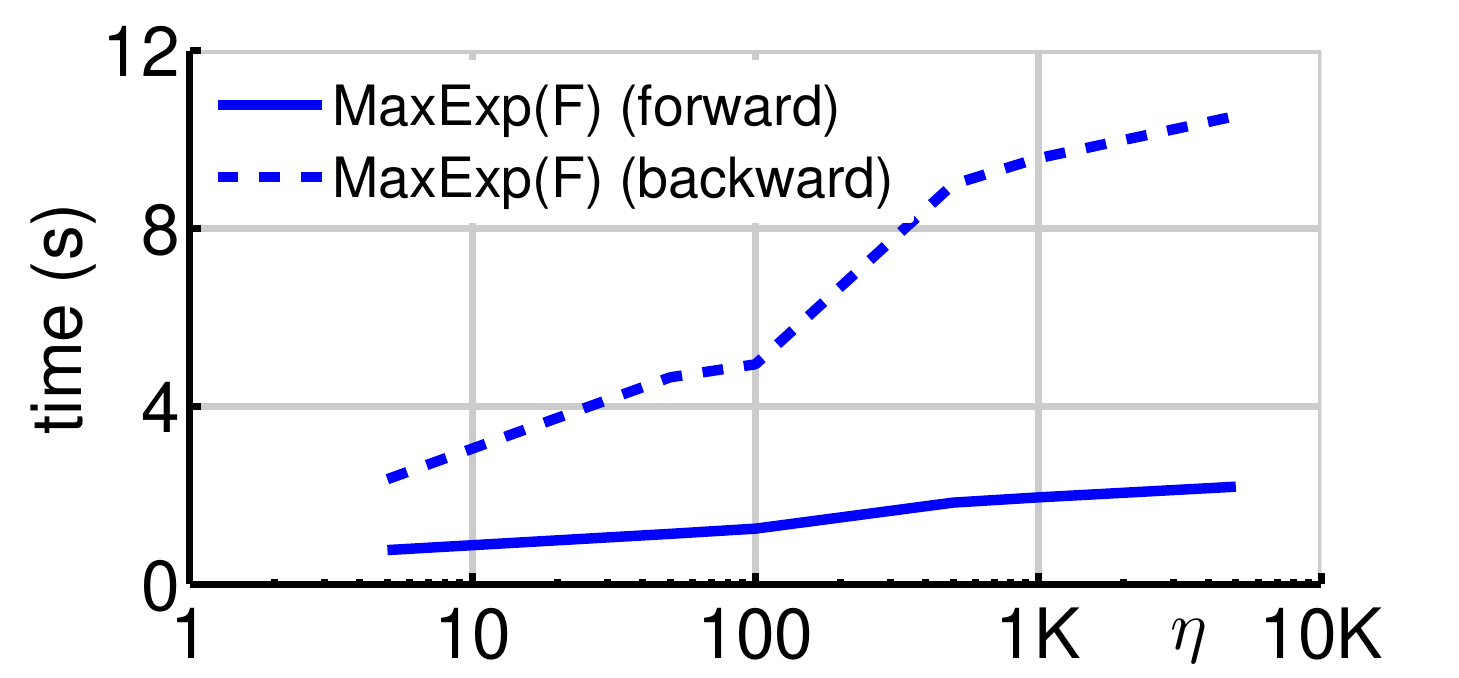}\vspace{-0.2cm}
\caption{\label{fig:tim1}}
\end{subfigure}
\begin{subfigure}[t]{\PlotWW\linewidth}
\centering\includegraphics[trim=0 0 0 0, clip=true,width=\PlotWWW]{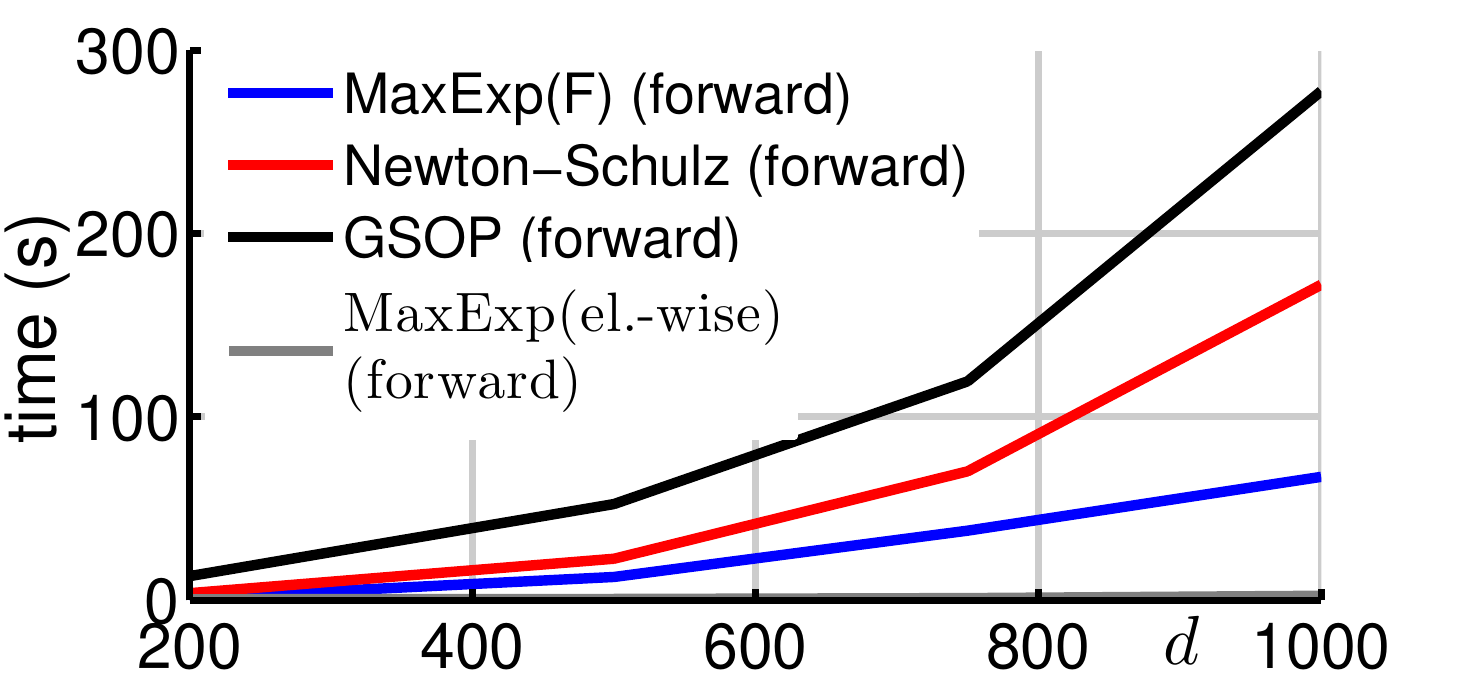}\vspace{-0.2cm}
\caption{\label{fig:tim2}}
\end{subfigure}
\ifdefined\arxiv\else\\\fi
\begin{subfigure}[t]{\PlotWW\linewidth}
\centering\includegraphics[trim=0 0 0 0, clip=true,width=\PlotWWW]{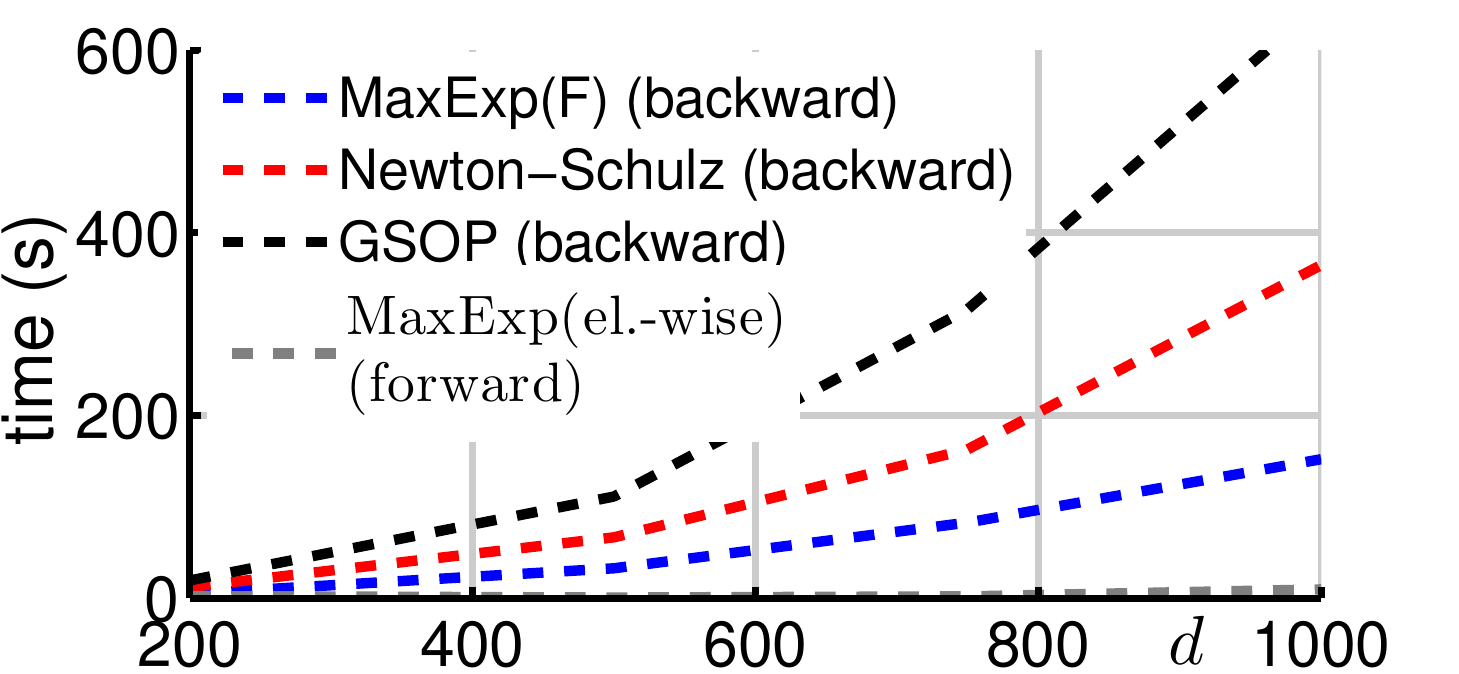}\vspace{-0.2cm}
\caption{\label{fig:tim3}}
\end{subfigure}
\begin{subfigure}[t]{\PlotWW\linewidth}
\centering\includegraphics[trim=0 0 0 0, clip=true,width=\PlotWWW]{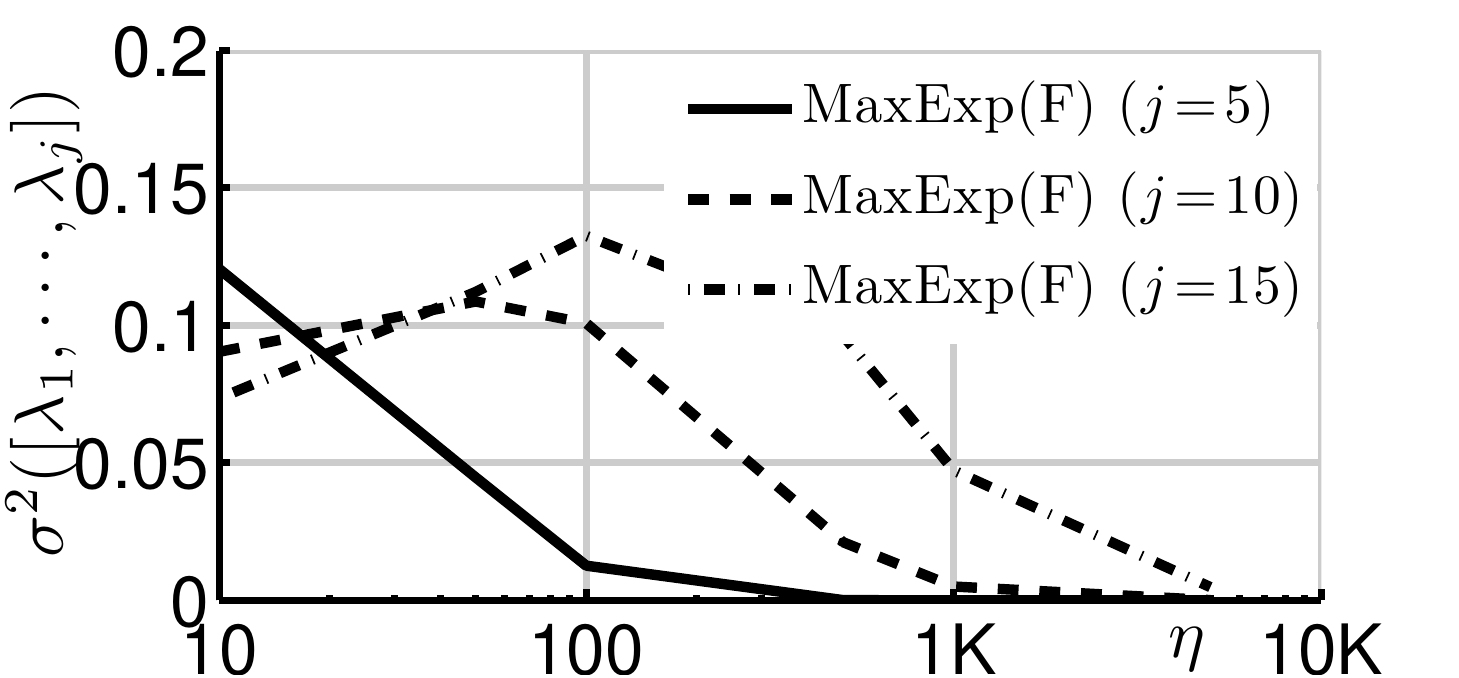}\vspace{-0.2cm}
\caption{\label{fig:var1}}
\end{subfigure}
\vspace{-0.3cm}
\caption{\revisedd{Timing and variance of operators. Figure \ref{fig:tim1} evaluates the speed (per 1000 images) of forward and backward passes our Fast Spectral MaxExp (Alg. \ref{code:exp_sqr}) as a function of $\eta$ given the side size $d\!=\!200$ of $\mM$. Figures \ref{fig:tim2} and \ref{fig:tim3} compare the speed of forward and backward passes of the proposed Fast Spectral MaxExp ({\em MaxExp(F)}), the Newton-Schulz iterations algorithm, the Generalized Spectral Power Normalization ({\em GSPN}) that uses SVD and the Element-wise MaxExp ({\em MaxExp(el.-wise)}). We vary the side size $200\!\leq\!d\!\leq\!1000$ of $\mM$. Figure \ref{fig:var1} shows how the variance over $5\!\leq\!j\!\leq\!15$ leading eigenvalues passed via MaxExp(F) varies as a function of $\eta$. As $\eta$ increases the variance of the leading eigenvalues decreases. 
}}
\label{fig:timing}
\vspace{-0.35cm}
\end{figure}

}

\revised{
Moreover, our ({\em SoSN+Pretr.+SC+Spec. MaxExp(F)}) with $224\!\times\!224$ res. pre-trained  on \textit{mini}ImageNet yielded a $\sim\!2.0\%$ increase over non-spectral ({\em SoSN+SigmE}). A similar trend is observed on 5- to 90-way protocols which test the stability of our idea across matching testing queries each with support images from 90 distinct classes. A realistic setting with a large `way' number is frequently avoided in few-shot learning community.
}

%

Figure \ref{fig:openmic2} introduces results for the protocol that tests the
generalization from one domain to another domain ({\em Protocol II}) which is also often avoided in the few-shot learning community but helps ascertain how well the algorithm generalizes between different domains.
Our ({\em SoSN+SigmE}) outperforms Relation Net by up to $\sim\!\bf{35\%}$. As this protocol measures the generalization ability of few-shot learning methods under a large domain shift, the average scores are $\sim$20\% below scores from Figure \ref{fig:openmic1}. However, our second-order relationship descriptor with SigmE pooling is beneficial for similarity learning. \revised{Moreover, our ({\em SoSN+Pretr.+SC+Spec. MaxExp(F)}) with $224\!\times\!224$ res. pre-trained  on \textit{mini}ImageNet obtained the best performance.}

To conclude our experiments on fine-grained few-shot learning, Tables \ref{tab:fewshot-flower} and \ref{tab:fewshot-food} introduce results on 5-way 1-shot and 5-way 5-shot evaluation protocols on Flower102 and Food-101. Both tables demonstrate that our ({\em SoSN+SigmE}) model outperforms Relation Net by $\sim\!\!\bf{5\%}$ to $\sim\!\!\bf{11\%}$ accuracy.

\revised{Spectral operators on SoSN pre-trained with \textit{mini}ImageNet performed well. On 1-/5-shot protocols, our ({\em SoSN+Pretr.+SC+Spec. MaxExp(F)}) scored $1.5\%$/$1.55\%$ and $2.4\%$/$3.5\%$ accuracy over ({\em SoSN+SC+SigmE}) on Flower102 and Food-101, resp. In contrast, pre-trained non-spectral ({\em SoSN+Pretr.+SC+SigmE}) and spectral ({\em SoSN.+SC+Spec. Gamma}) and ({\em SoSN+SC+Spec. MaxExp(F)}) without pre-training fail to bring further benefits over ({\em SoSN+SC+SigmE}) without pre-training. This supports our hypothesis about the connection of spectral operators to time-reversed HDP which reduces correlation between co-occurrences thus limiting catastrophic forgetting.}

%

\revisedd{
\vspace{0.05cm}
\noindent{\textbf{Graph classification.}} Table \ref{tab:graph} shows results for spectral and element-wise PN operators on covariance matrices employed on top of the Graph Isomorphism Network ({\em GIN0}) \cite{pytorch_geom}. Element-wise PN ({\em SOP+SigmE}) outperforms ({\em SOP}) and the first-order average pooling ({\em GIN0}), one of the strongest baselines. 
Moreover, the Fast Spectral MaxExp ({\em SOP+MaxExp(F)}) typically outperforms the Newton-Schulz inter. (approx. matrix square root) and element-wise operators. 
As package \cite{pytorch_geom} uses the validation split only for validation (in contrast to other packages), we retrain on train+validation splits (see the bottom of  Table \ref{tab:graph}).  
}

\vspace{-0.2cm}
\section{Conclusions}
\label{sec:conclude}

We have studied Power Normalizations in the context of element-wise co-occurrence representations and demonstrated their theoretical role which is to `detect' co-occurrences. We have shown that different assumptions on distributions from which features are drawn result in similar non-linearities \eg, MaxExp \vs SigmE. Thus, we have proposed surrogate functions SigmE and AsinhE with well-behaved derivatives for end-to-end training which can handle so-called negative evidence.  Moreover, we have proposed a fast spectral MaxExp which can be computed as faster than the matrix square root via iterative Newton-Schulz while enjoying the adjustable parameter. Finally, we have shown that Spectral Power Normalizations fulfill a similar role to the time-reversed Heat Diffusion Process well-known from the graph theory, thus paving a strong theoretical foundation for further studies of SPNs. 

\vspace{0.05cm}
\noindent{\textbf{Acknowledgements.}} We thank Dr. Ke Sun for brainstorming, Ondrej Hlinka/Garry Swan for help with HPC, Hao Zhu for checks of some SOP codes, and Lei Wang for quick checks of text.



{\small
\vspace{-0.3cm}
\bibliographystyle{IEEEtran}
\bibliography{pn}
}

\vspace{-1.2cm}
\begin{IEEEbiography}[{\includegraphics[width=1in,height=1.25in,clip,keepaspectratio]{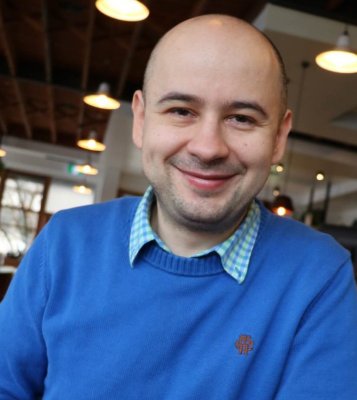}}]{Piotr Koniusz.}
A Senior Researcher in Machine Learning Research Group at Data61/CSIRO (NICTA), and a Senior Honorary Lecturer at the Australian National University (ANU). He was a postdoctoral researcher in the team LEAR, INRIA, France. He received his BSc in Telecommunications and Software Engineering in 2004 from the Warsaw University of Technology, Poland, and completed his PhD in Computer Vision in 2013 at CVSSP, University of Surrey, UK. 
\end{IEEEbiography}

\vspace{-1.5cm}
\begin{IEEEbiography}[{\includegraphics[width=1in,height=1.25in,clip,keepaspectratio]{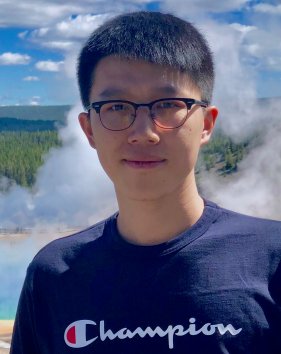}}]{Hongguang Zhang.}
From 2016, Hongguang Zhang is a PhD student in computer vision and machine learning at the Australian National University and Data61/CSIRO, Canberra, Australia. He received the BSc degree in electrical engineering and automation from Shanghai Jiao Tong University, Shanghai, China in 2014. He received his MSc degree in electronics science and technology from National University of Defense Technology, Changsha, China in 2016. His interests include fine-grained image classification, zero-shot learning, few-shot learning and deep learning methods.
\end{IEEEbiography}

\clearpage

\newcommand{\widsup}{-0.2cm}

\renewcommand{\appendixname}{Supplementary Material}
\appendix


Below we provide  proofs and derivations.

\vspace{\widsup}
\vspace{-0.1cm}
\subsection*{A. Proof of Eq. \eqref{eq:my_maxexp33}}
\label{app:eq14}
Let the following  difference of probabilities:
\begin{align}
&\psi\!=\!(1\!-\!q)^{N}\!-\!(1\!-\!p)^{N}.
\label{eq:my_maxexp33}
\end{align}

\begin{proof}
One can derive Eq. \eqref{eq:my_maxexp33} by directly applying the Multinomial calculus as follows:
%
\begin{align}
& \!\!\!\!\textstyle\sum\limits_{n=1}^{N}\sum\limits_{n'=0}^{N\!-n}\!\!\binom{N}{n,n'\!,N\!-n-n'\!-n''\!}\!\!\left(p^nq^{n'\!}\!-\!p^{n'\!}q^{n}\right)\!(1\!\!-\!\!p\!\!-\!\!q)^{N\!-n-n'}\!.
\label{eq:my_maxexppr22}
\end{align}
%
One can verify algebraically/numerically that Eq. \eqref{eq:my_maxexppr22} and \eqref{eq:my_maxexp33} are equivalent.
\end{proof}

\vspace{\widsup}
\vspace{-0.3cm}
\subsection*{B. Deriv. of Average, Gamma and MaxExp}
\label{app:der}
{\noindent Let} $\mPhi\!=\![\vphi_1,\cdots,\vphi_N]\!\in\!\mbr{d\times N}$\!, $\mC\!=\![\vc_1,\cdots,\vc_N]\!\in\!\mbr{Z'\!\times N}$\!, $\vphibar_n\!=\![\vphi_n; \vc_n], n\!\in\!\idx{N}$. 
 Let some class. loss $\ell(\mPsi,\mW), \mPsi\!\in\!\semipd{d+Z'}\!$ (or $\spd{}$) and $\mW$ are our descriptor and a hyperplane. Observe that:
%
\begin{align}
& \frac{\partial\sum_n\!\vphibar_n\vphibar_n^T}{\partial \phi_{kl}}\!=\!
\left[\begin{array}{cc}
\vj_{k}\vphi_{l}^T\!+\!\vphi_{l}\vj_{k}^T & \vj_{k}\vc_{l}^T \\
\vc_{l}\vj_{k}^T & [0]_{Z'\!\times Z'\!}  \\
\end{array}\right],
\label{eq:der_auto}
\end{align}
where $[0]_{Z'\!\times Z'\!}$ denotes an array of size $Z'\!\times Z'\!$ filled with zeros. 

\vspace{0.05cm}
{\noindent\textbf{Average pooling}} is set by $\mPsi\!=\!\mygthreee{Avg}{\mM}\!=\!\mM$ and $\mD\!=\!\vOnes\vOnes^T$ so that $\mPsi\!=\!\frac{1}{N}\sum_n\!\vphibar_n\vphibar_n^T$. Thus, the full derivative becomes:
\begin{align}
& \!\!\!\text{\scriptsize $\sum\limits_{k,l}\frac{\partial \ell(\mPsi,\mW)}{\partial  \Psi_{kl}}\frac{\partial \Psi_{kl}}{\partial \mPhi}=\frac{2}{N}\sym\Big(\frac{\partial \ell(\mPsi,\mW)}{\partial  \mPsi}\!\odot\!\mD\Big)_{(1:d,:)}
\left[\!\!\begin{array}{c}
\mPhi\\
\mC
\end{array}\!\!\right]$}.
\label{eq:der_auto2}
\end{align}
$\mX_{(1:d,:)}$ returns $1,\cdots,d$ rows/all col. as Matlab oper. $\mX(1\!:\!d,:)$.

\vspace{0.05cm}
{\noindent\textbf{Gamma pooling}} is set by $\mPsi\!=\!\mygthreee{Gamma}{\mM;\gamma}\!=\!(\mM\!+\!\varepsilon)^\gamma$, where rising $\mM$ to the power of $\gamma$ is element-wise and $\varepsilon$ is a reg. constant. Thus, we obtain:
\begin{align}
&\!\!\!\frac{\partial\mPsi}{\partial\phi_{kl}} =\frac{1}{N}\gamma\big(\mM\!+\!\varepsilon\big)^{\gamma-1}\!\odot\frac{\partial\sum_n\!\vphibar_n\vphibar_n^T}{\partial \phi_{kl}}.\!\!
\end{align}
The derivative is given by Eq. \eqref{eq:der_auto2} if $\mD\!=\!\gamma\big(\mM\!+\!\varepsilon\big)^{\!\gamma-1}$. 
\comment{
\begin{align}
&\!\!\!\!\!\!\sum\limits_{k,l}\frac{\partial \ell(\mPsi,\mW)}{\partial  \Psi_{kl}}\frac{\partial \Psi_{kl}}{\partial \mPhi}=\\
&\quad\frac{2\gamma}{N}\sym\left(\frac{\partial \ell(\mPsi,\mW)}{\partial  \mPsi}\!\odot\!\big(\mM\!+\!\varepsilon\big)^{\!\gamma-1}\right)_{(1:d,:)}\!
\left[\!\!\begin{array}{c}
\mPhi\\
\mC
\end{array}\!\!\right],\nonumber
\end{align}
where $\mM_{(1:d,1:n)}$ denotes a MATLAB style operator selecting sub-matrix $\mM'\!\in\!\mbr{d\times n}$ from $\mM$ such that $\mM'_{d'n'}\!=\!\mM_{d'n'}, \forall d'\!\!=\!1,\cdots,d,\;n'\!\!=\!1,\cdots,n$.
}

\vspace{0.05cm}
{\noindent\textbf{MaxExp pooling}} $\mPsi\!=\!\mygthreee{MaxExp}{\mM;\eta}\!=\!1\!-\!(1\!-\!\mM/(\trace(\mM)+\varepsilon))^\eta$ has the derivative given by  Eq. \eqref{eq:der_auto2} with the following $\mD$:
%
\comment{
\begin{align}
& \!\!\!\!\frac{\partial\mPsi}{\partial \eta}\!=\!-\!\left(1\!-\!\frac{\mM}{\trace(\mM)+\varepsilon}\right)^\eta\!\!\odot\left(\log\left(1\!-\!\frac{\mM}{\trace(\mM)+\varepsilon}\right)\right)^{-1}\!\!\!\!,\\
& \!\!\!\!\left[\frac{\partial\Psi_{kl}}{\partial M_{kl}}\right]_{
\begin{array}{c}
\!\!\!\scriptscriptstyle(k\!,l)\in\!\!\!\!\\[-5pt]
\!\!\scriptscriptstyle\idx{d}\times\idx{d}\!\!\!\!
\end{array}
}\!\!\!\!\!\!\!\!=\eta\left(1\!-\!\frac{\mM}{\trace(\mM)+\varepsilon}\right)^{\eta-1}\nonumber\\
&\qquad\qquad\qquad\odot\left(\frac{1}{\trace(\mM)\!+\!\varepsilon}\!-\!\frac{\mM\!\odot\!\mIdent}{\left(\trace(\mM)\!+\!\varepsilon\right)^2}\right),
\end{align}
while the matrix form of this derivative becomes:
}
\comment{
\begin{align}
& \!\!\!\!\frac{\partial\mPsi}{\partial\phi_{kl}}\!=\!-\frac{\eta}{N}\left(1\!-\!\frac{\mM}{\trace(\mM)+\varepsilon}\right)^{\eta-1}\\
&\odot\left(\frac{1}{\trace(\mM)\!+\!\varepsilon}\!-\!\frac{\mM\!\odot\!\mIdent}{\left(\trace(\mM)\!+\!\varepsilon\right)^2}\right)\!\odot\!\,\frac{\partial\sum_n\!\vphibar_n\vphibar_n^T}{\partial \phi_{kl}},\nonumber
\end{align}
}
\begin{align}
& \!\!\!\!\!\!\!\!\!\text{\scriptsize $\mD\!=\!\eta\left(1\!-\!\frac{\mM}{\trace(\mM)+\varepsilon}\right)^{\eta-1}\!\!\!\!\!\!\!\odot\mT\text{ and }\; \mT\!=\!\left(\frac{1}{\trace(\mM)\!+\!\varepsilon}\!-\!\frac{\mM\!\odot\!\mIdent}{\left(\trace(\mM)\!+\!\varepsilon\right)^2}\right)$},\!\!\!
\end{align}
%
%
\comment{
\begin{align}
&\!\!\!\!\!\!\!\!\sum\limits_{k,l}\frac{\partial \ell(\mPsi,\mW)}{\partial  \Psi_{kl}}\frac{\partial \Psi_{kl}}{\partial \mPhi}\!=\!-\frac{2\eta}{N}\sym\!\left(\frac{\partial \ell(\mPsi,\mW)}{\partial  \mPsi}\!\odot\! \Big(\!1\!-\!\frac{\mM}{\trace(\mM)\!+\!\varepsilon}\!\Big)^{\eta-1} \right.\nonumber\\
&\quad\left.\odot\Big(\frac{1}{\trace(\mM)\!+\!\varepsilon}\!-\!\frac{\mM\!\odot\!\mIdent}{\left(\trace(\mM)\!+\!\varepsilon\right)^2}\Big)\right)_{(1:d,:)}\!
\left[\!\!\begin{array}{c}
\mPhi\\
\mC
\end{array}\!\!\right],
\end{align}
}
where multiplication $\odot$, division, rising to the power {\em etc.} are all element-wise operations.

\vspace{\widsup}
\subsection*{C. Derivatives of SigmE and AsinhE pooling}
\label{app:der2}

\vspace{0.05cm}
{\noindent\textbf{SigmE pooling}} is set by $\mPsi\!=\!\mygthreee{SigmE}{\mM;\eta'}\!=\!\frac{2}{1\!+\!\expl{-\eta'\mM}}\!-\!1$ or trace-normalized $\frac{2}{1\!+\!\expl{\frac{-\eta'\mM}{\trace(\mM)+\varepsilon}}}\!-\!1$. The first expression yields: 
\begin{align}
& \!\!\!\!\frac{\partial\mPsi}{\partial\phi_{kl}}\!=\!\frac{1}{N}\frac{2\eta'\expl{-\eta'\mM}}{(1+\expl{-\eta'\mM})^2}\odot(\vj_{k}\vphi_{l}^T\!+\!\vphi_{l}\vj_{k}^T),
\end{align}
where multiplication $\odot$, division, and exponentiation are all element-wise operations.

\vspace{0.05cm}
{\noindent\textbf{AsinhE pooling}} is set by $\mPsi\!=\!\mygthreee{AsinhE}{\mM;\gamma'}\!=\!\arcsinh(\gamma'\!\mM)\!=\!\log(\gamma'\!\mM+\sqrt{1+{\gamma'}^2\!\mM^2})$ which yields the following:
\begin{align}
& \!\!\!\!\!\frac{\partial\mPsi}{\partial\phi_{kl}}\!=\!\frac{1}{N}\frac{\gamma'}{\sqrt{{\gamma'}^2\mM^2+1}}\odot(\vj_{k}\vphi_{l}^T\!+\!\vphi_{l}\vj_{k}^T),
\end{align}
where multiplication $\odot$, division, square root and the square are all element-wise operations.

For SigmE, trace-normalized SigmE and AsinhE pooling, 
\comment{
Then, for $\mPhi=[\vphi_1,\cdots,\vphi_N]\!\in\!\mbr{d\times N}$ and $\mC=[\vc_1,\cdots,\vc_N]\!\in\!\mbr{Z'\!\times N}$, we obtain the following expression:
\begin{align}
& \!\!\sum\limits_{k,l}\frac{\partial \ell(\mPsi,\mW)}{\partial  \Psi_{kl}}\frac{\partial \Psi_{kl}}{\partial \mPhi}=\frac{2}{N}\sym\left(\frac{\partial \ell(\mPsi,\mW)}{\partial  \mPsi}\!\odot\! 
\mD
\right)_{(1:d,:)}
\left[\!\!\begin{array}{c}
\mPhi\\
\mC
\end{array}\!\!\right],
\label{eq:general_der_add_pn}
\end{align}
}
final derivatives are given by Eq. \eqref{eq:der_auto2} with the following $\mD$, resp.: 
\begin{align}
&\!\!\!\!\!\!\!\!\!\!\!\!\!\!\text{\scriptsize $\mD\!=\!\frac{2\eta'\expl{-\eta'\mM}}{(1\!+\!\expl{-\eta'\mM})^2}$}\text{ or } \text{\scriptsize $\mD\!=\!\frac{2\eta'\expl{\frac{-\eta'\mM}{\trace(\mM)+\varepsilon}}}{\big(1\!+\!\expl{\frac{-\eta'\mM}{\trace(\mM)+\varepsilon}}\big)^2}\!\odot\mT$} \text{ and } \text{\scriptsize $\mD\!=\!\frac{\gamma'}{\sqrt{{\gamma'}^2\mM^2\!+\!1}}$}.\!\!
\label{eq:additional_pn_matrices_d}
\end{align}
Moreover, for  SigmE and AsinhE we allow $\beta$-centering so 
its derivative has to be included in the chain rule.

\vspace{\widsup}
\vspace{-0.05cm}
\subsection*{D. Derivative of Spectral Gamma}
\label{app:der3}

\vspace{0.05cm}
{\noindent\textbf{Gamma pooling}} has derivative which can be solved by the SVD back-propagation or the Sylvester equation if $\gamma\!=0.5$:
\begin{align}
&\!\!\text{\scriptsize $2\res\!\Big(\vectorise\Big(\sym\Big(\frac{\partial \ell(\mPsi,\mW)}{\partial  \mPsi}\Big)\Big)^T\!\!\mM^*\!\Big)_{d\!+\!Z'\!\times d\!+\!Z'\!} \text{ and } \mM^*\!\!\!\!=\!(\mIdent\!\otimes\!\mM^{\frac{1}{2}}\!\!+\!\!\mM^{\frac{1}{2}}\!\otimes\!\mIdent)^{\dagger}$},
\label{eq:sylv}
\end{align}
where $\otimes$ and $\dagger$ are the Kronecker product and the pseudo-inverse. Matrix reshaping to the size $m\!\times\!n$ is by  $\res(\mX)_{m\!\times\!n}$.

\comment{
\vspace{0.05cm}
{\noindent\textbf{MaxExp}} has a closed-form derivative which requires the following chain rule:
\begin{align}
& \!\!\!\!\text{\scriptsize $\frac{\partial\mygthree{\mM}}{\partial\ M_{kl}}=\!\frac{1}{\trace(\mM)}\sum\limits_{n\!=\!0}^{{\eta}-1}\left(\mIdent\!-\!\frac{\mM}{\trace(\mM)}\right)^n\!\left(\mJ_{kl}-\frac{\mM}{\trace(\mM)}\sIdent_{kl}\right)\left(\mIdent\!-\!\frac{\mM}{\trace(\mM)}\right)^{{\eta}-1-n}\!\!\!\!\!\!\!\!\!\!\!\!\!$}.
\label{eq:binom_mat_der}
\end{align}
}


\vspace{\widsup}
\vspace{-0.05cm}
\subsection*{E. Outline Proof of \Cref{th:heatdiff}}
\label{app:graph}
\begin{proof}
The connection of GMRF and $\mathcal{E}$ follows \cite{GMRF-graph}. For the reminder of the proof, a simple visual inspection of profiles $g_{\text{MaxExp}}(\lambda)\!=\!1\!-\!(1\!-\!\lambda)^\eta$ and $g_{\text{HDP}}(\lambda)\!=\!\exp(-t/\lambda)$ shows that $g_{\text{MaxExp}}(\lambda)\!\approx\!g_{\text{HDP}}(\lambda)$, or even $0\!\leq\!g_{\text{MaxExp}}(\lambda)\!-\!g_{\text{HDP}}(\lambda)\!<\!\epsilon$ for some sufficiently small $\epsilon\!>\!0$, which shows that $g_{\text{MaxExp}}(\lambda)$ is an upper bound of $g_{\text{HDP}}(\lambda)$ on the interval $\lambda\!\in\![0,1]$. \Cref{{fig:pow5}} shows $g_{\text{MaxExp}}(\lambda)$, $g_{\text{Gamma}}(\lambda)$ and $g_{\text{HDP}}(\lambda)$ for MaxExp, Gamma and HDP. The plot also shows that $g_{\text{Gamma}}(\lambda)\!\approx\!g_{\text{HDP}}(\lambda)$.
\end{proof}

\vspace{\widsup}
\vspace{-0.2cm}
\subsection*{F. Proof of \Cref{th:param_bounds}}
\label{app:pr_bound}
Working with MaxExp and HDP according to their original parametrization as used in  \Cref{fig:bounds-zoom} is difficult/intractable. Thus,
for this bound, we start by a parametrization $y\!=\!\frac{t}{\lambda}$ and we note that $t$ and $\eta$ can be tied together, that is $t\eta\!=\!\alpha$. We obtain:
\begin{equation}
    e^{-y}\leq\left(\frac{1}{e}\!-\!1\right)y\!+\!1\leq1\!-\!\left(1\!-\!\frac{t}{y}\right)^\eta\!=1\!-\!\left(1\!-\!\frac{\alpha}{\eta y}\right)^\eta,
\end{equation}
where $\left(\frac{1}{e}\!-\!1\right)y\!+\!1$ is an upper bound of $e^{-y}$ on $y\in(0,1)$ and a lower bound of $1\!-\!\left(1\!-\!\frac{\alpha}{\eta y}\right)^\eta$. Moreover, the latter equation can be tightened (as in `lowered down') to touch $\left(\frac{1}{e}\!-\!1\right)y\!+\!1$ on $y\in(0,1)$. This process is illustrated in \Cref{fig:bounds2}. To this end, we need to solve for the system of equations to obtain $(y,\alpha)$:

\ifdefined\arxiv
\renewcommand{\PowH}{3.0cm}
\renewcommand{\PowHB}{2.875cm}
\renewcommand{\PowW}{3.65cm}
\else
\renewcommand{\PowH}{4.5cm}
\renewcommand{\PowHB}{3cm}
\renewcommand{\PowW}{4.1cm}
\fi
\begin{figure*}[!b]
\centering
\vspace{-0.3cm}
\hspace{-0.6cm}
\begin{subfigure}[t]{0.325\linewidth}
\centering\includegraphics[trim=0 0 0 0, clip=true, height=\PowH]{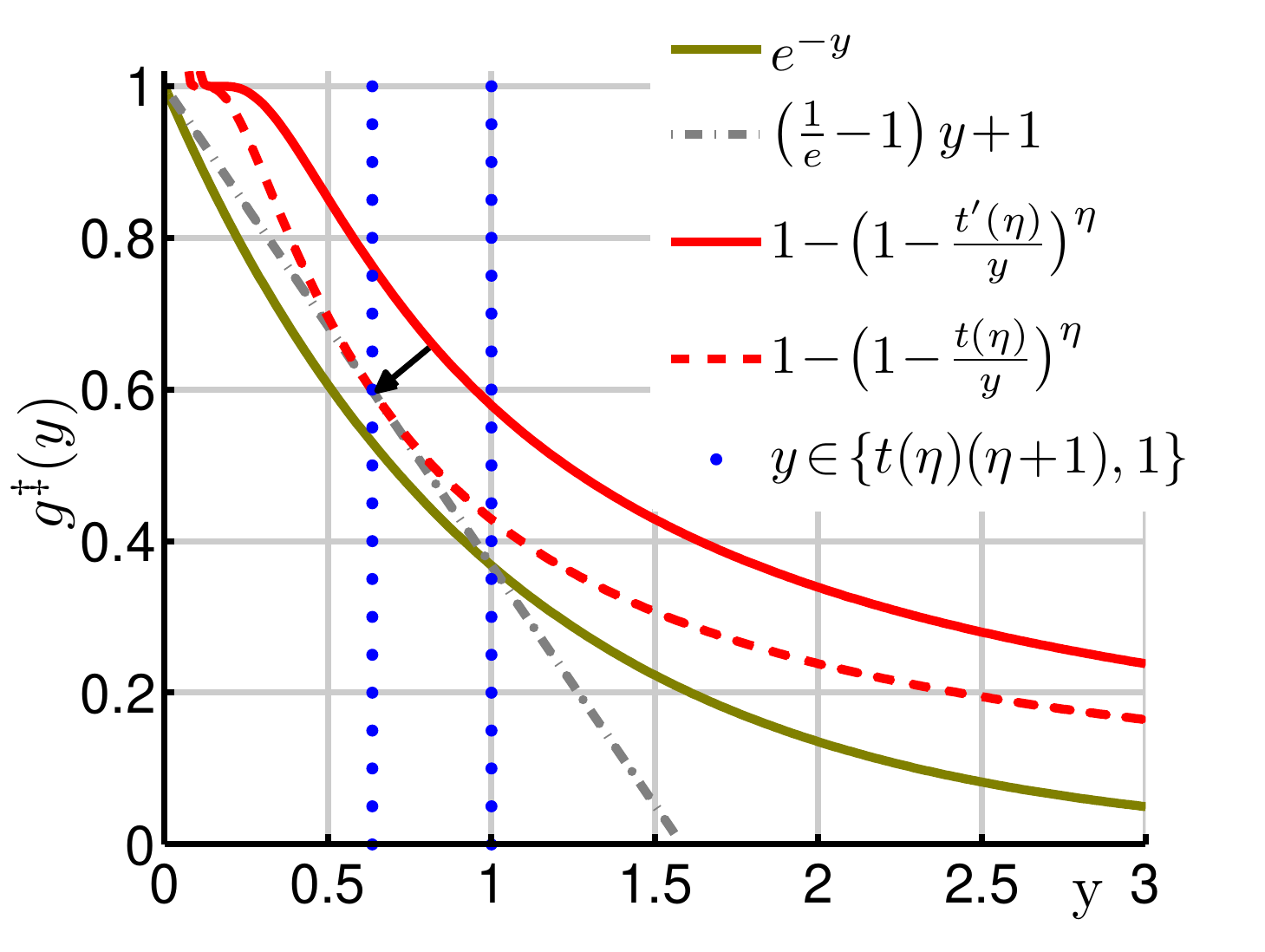}
\caption{$\!\!\!\!\!\!\!\!$}\label{fig:bounds2}
\end{subfigure}
\begin{subfigure}[t]{0.325\linewidth}
\centering\includegraphics[trim=0 0 0 0, clip=true, height=\PowH]{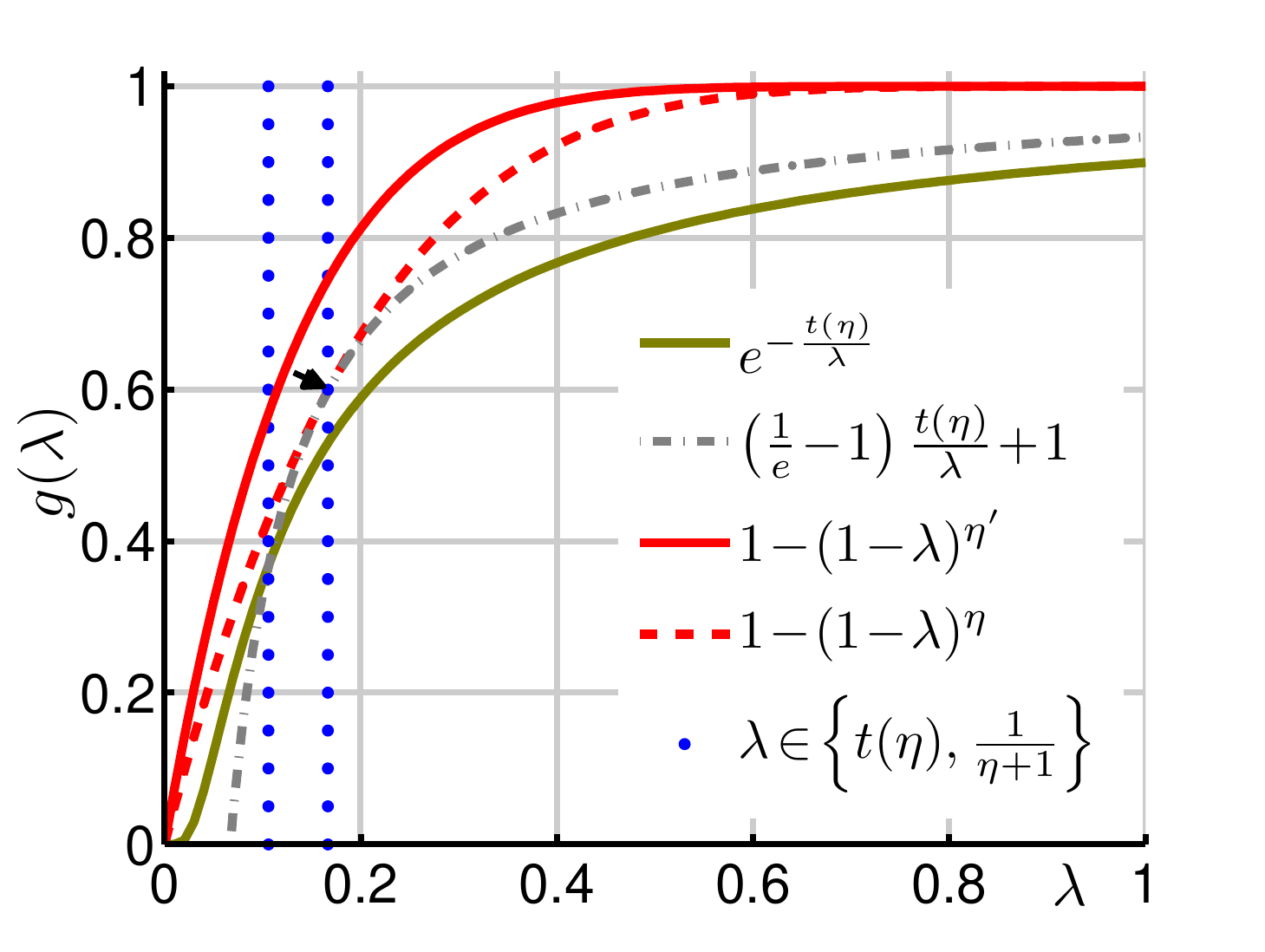}
\caption{$\!\!\!\!\!\!\!\!$}\label{fig:bounds-zoom}
\end{subfigure}
\begin{subfigure}[t]{0.325\linewidth}
\centering\includegraphics[trim=0 0 0 0, clip=true, height=\PowH]{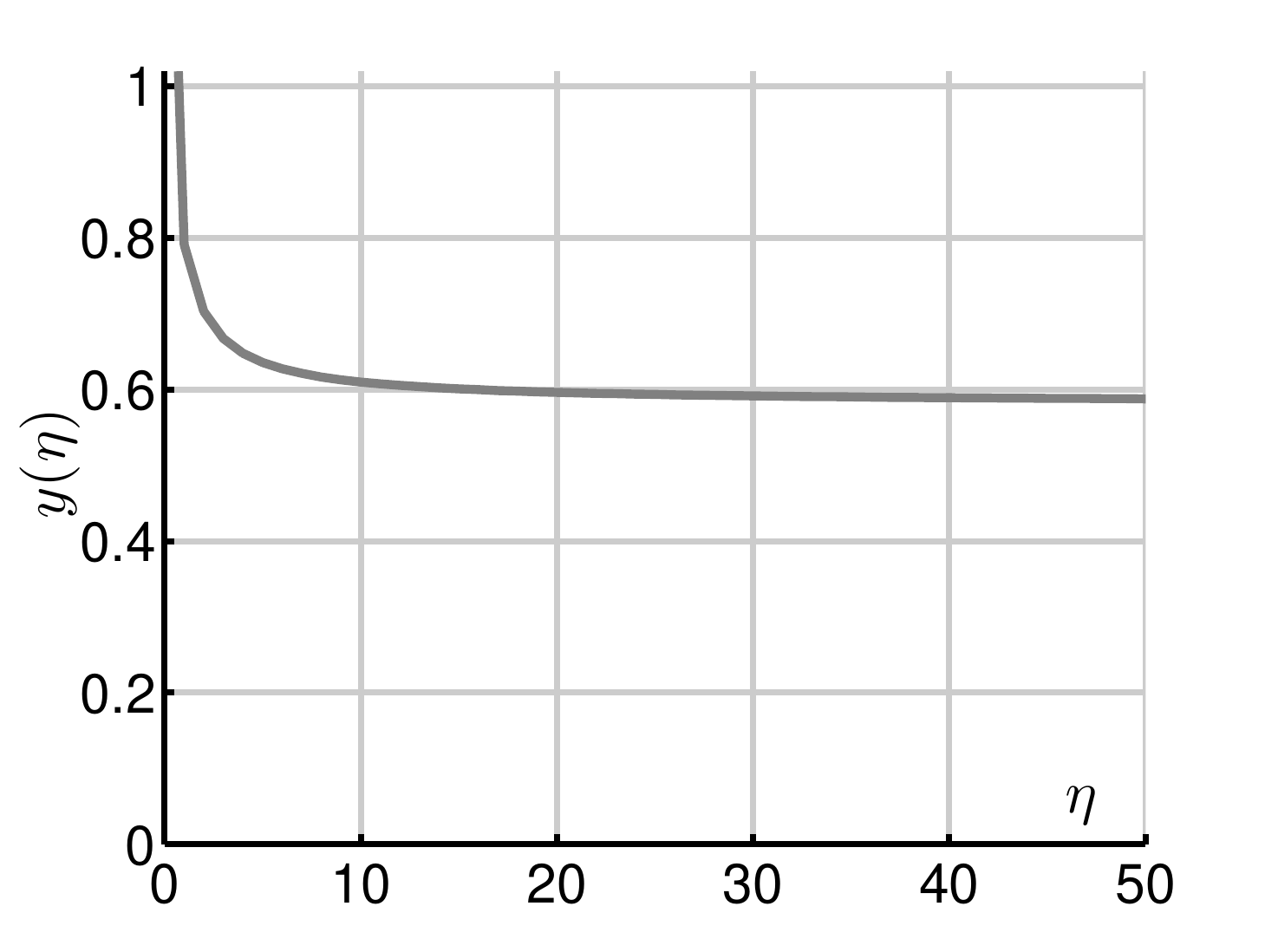}
\caption{$\!\!\!\!\!\!\!\!$}\label{fig:y-eta}
\end{subfigure}
%
\caption{In Fig. \ref{fig:bounds2}, we show pooling functions reparametrized according to $y\!=\!\frac{t}{\lambda}$. Specifically, we have HDP given by $e^{-y}$, its upper bound $\left(\frac{1}{e}\!-\!1\right)y\!+\!1$ for $y\!\in\!(0,1)$, and MaxExp given by $1\!-\!\left(1\!-\!\frac{t'(\eta)}{y}\right)^\eta$ which we `lower down' onto $\left(\frac{1}{e}\!-\!1\right)y\!+\!1$ as illustrated by the black arrow. As we tighten the bound, some initial MaxExp with $t'(\eta)$ becomes MaxExp with $t(\eta)$. Blue vertical lines indicate $y$ at which we measure $\epsilon_1$ and $\epsilon_2$. Fig. \ref{fig:bounds-zoom} illustrates the same pooling operations as in Fig. \ref{fig:bounds2} but without the reparametrization, that is, we show $g(\lambda)$ rather than $g^\ddag(y)$. Note the corresponding $y$ and $\lambda$ ranges in both figures indicated by the blue dashed lines. Fig. \ref{fig:y-eta} shows that $y(\eta)$ is monotonically decreasing on $\eta\!\in\!(0,\infty)$.}
\vspace{-0.3cm}
\label{fig:repar}
\end{figure*}

\vspace{-0.3cm}
\begin{equation}
\!\!\!\!\begin{cases}
\left(\frac{1}{e}\!-\!1\right)y\!+\!1=1\!-\!\left(1\!-\!\frac{\alpha}{\eta y}\right)^\eta\\
\frac{\partial \left(\frac{1}{e}\!-\!1\right)y\!+\!1}{\partial y}=-\frac{\partial\left(1\!-\!\frac{\alpha}{\eta y}\right)^\eta}{\partial y}\Rightarrow-\frac{\alpha\eta\left(1-\frac{\alpha}{\eta y}\right)^\eta}{y\left(\alpha-\eta y\right)} = \left(\frac{1}{e}\!-\!1\right),\!
\label{eq:syst1}
\end{cases}
\end{equation}
which simply says that we search for $(y,\alpha)$ for which both functions on the left- and right-hand side touch and their slopes/tangents (thus derivatives \wrt $y$) are equal. We could try directly `lower down' $1\!-\!\left(1\!-\!\frac{\alpha}{\eta y}\right)^\eta$ onto $e^{-y}$ but such an approach yields an intractable system of equations requiring numerical approximations and the use of the LambertW function. 
By solving Eq. \eqref{eq:syst1} we get:
%
\vspace{-0.5cm}
\begin{equation}
\begin{cases}
\alpha(\eta)=\frac{e}{e\!-\!1}\left(\frac{\eta}{\eta\!+\!1}\right)^{\eta+1}\\
y(\eta)=\frac{e}{e\!-\!1}\left(\frac{\eta}{\eta\!+\!1}\right)^{\eta}\!\!.
\label{eq:yalpha}
\end{cases}
\vspace{-0.1cm}
\end{equation}
Recall that we have assumed parametrization $t(\eta)\!=\!\frac{\alpha(\eta)}{\eta}$, thus $t(\eta)\!=\!\frac{e}{e\!-\!1}\frac{\eta^\eta}{(\eta\!+\!1)^{\eta+1}}$ (Eq. \eqref{eq:eta_par}). Furthermore, we have to check if $y\!\in(0,1)$ for $1\!\leq\!\eta\!\leq\!\infty$ in order for the bound to hold as $\left(\frac{1}{e}\!-\!1\right)y\!+\!1$ is an upper bound of $e^{-y}$ only for $y\!\in\!(0,1)$. To this end, we firstly notice that $y(\eta)$ is monotonically decreasing on $0\leq\eta\leq\infty$ as shown in \Cref{fig:y-eta}. Therefore, it suffices to check extremes of $\eta$ for $1\!\leq\!\eta\!\leq\!\infty$, that is $y(1)\!=\!\frac{1}{2}\frac{e}{e-1}$ and $\lim_{\eta\rightarrow\infty}y(\eta)\!=\!\frac{1}{e-1}$ which verifies that $y(1\!\leq\!\eta\!\leq\!\infty)\!\subset\!(0,1)$.
%
%
%
%
This completes the first part of the proof.

The next part of the proof requires solving:
    $\epsilon_2(\eta)=1\!-\!\left(1\!-\!\frac{\alpha(\eta)}{\eta y}\right)^\eta\!\!\!-\!e^{-y}$
%
which can be solved by plugging $(y,\alpha)$ from Eq. \eqref{eq:yalpha} into it. After a few of algebraic manipulations we have $\epsilon_2(\eta)=1\!-\! \left(\frac{\eta}{\eta\!+\!1}\right)^\eta \!\!-\! e^{{\textstyle-\frac{e}{e\!-\!1}\left(\frac{\eta}{\eta\!+\!1}\right)^\eta}}$ (the right part of Eq. \eqref{eq:bounds1}).
 
 We also note that, by design, $\epsilon_1(\eta)=1\!-\!\left(1\!-\!\frac{\alpha}{\eta y}\right)^\eta\!\!\!-\!\left(\frac{1}{e}\!-\!1\right)y\!+\!1$ where $\left(\frac{1}{e}\!-\!1\right)y\!+\!1$ touches $e^{-y}$ at  $y\!=\!1$. Thus, we readily obtain  $\epsilon_1(\eta)=\frac{e\!-\!1}{e}\!-\!\left(1\!-\!\frac{e}{e\!-\!1}\frac{\eta^\eta}{(\eta\!+\!1)^{\eta+1}}\right)^\eta$ (the left part of Eq. \eqref{eq:bounds1}).

Finally, obtaining  $\eta(t)$ (an inverse of $t(\eta)$) follows from simple algebraic manipulations based on the Stirling approximation.

\vspace{\widsup}
\vspace{-0.1cm}
\subsection*{G. Proof of \Cref{th:gamma_bounds}}
\label{app:gamma_bound}
As in \hyperref[{app:pr_bound}]{F}, let parametrization $y\!=\!\frac{t}{\lambda}$ which yields a set of following equations:
\vspace{-0.5cm}
\begin{equation}
\label{eq:touch_gamma}
\begin{cases}
e^{-y}=(\frac{t}{y})^\gamma\\
\frac{\partial e^{-y}}{\partial y}=\frac{\partial (\frac{t}{y})^\gamma}{\partial y}  \Rightarrow -e^{-y}=-\frac{\gamma}{\lambda}\left(\frac{t}{y}\right)^\gamma\!\!\!.
\end{cases}
\end{equation}
We again seek a parametrization $t(\gamma)$ for which the Gamma and HDP functions touch. However, this time the strict bound can be achieved analytically (\eg, the analytical solution to \eqref{eq:touch_gamma} exists) and thus an intermediate bounding function is not needed. After a few of algebraic manipulations, we obtain a candidate solution $y\!=\gamma$, which, if combined with the intermediate equation $e^{y}\left(\frac{t}{y}\right)^\gamma\!\!=\!1$, readily yields $\gamma(t)\!=\!et$, which completes the  proof.

\vspace{\widsup}
\vspace{-0.1cm}
\subsection*{H. Proof of \Cref{th:maxexp_ode}}
\label{app:maxexp_ode}
To obtain the proof, we note that the system of ODE from Eq. \eqref{eq:heat_eq} and thus also Eq. \eqref{eq:heat_maxexp} can be written in the span of eigenvectors of graph Laplacian \cite{heat_eqq}. Thus, we write the standard Heat Diffusion Equation as:
\vspace{-0.2cm}
\begin{equation}
    \frac{\partial\vv'(t)}{\partial t}+\lambda\vv'(t)=0,
    \label{eq:gam_eig}
\end{equation}
where $\vv'(t)$ is now expressed in the span of new bases ($\vv'(t)$ is not a derivative).

Now simply write MaxExp as $1\!-\!\left(1\!-\!\lambda^{-1}\right)^{\eta(t)}$. In this parametrization, we use $\lambda^{-1}$ as we start from the eigenvectors of the graph Laplacian rather than an autocorrelation/covariance matrix, and we use the parametrization $\eta(t)$ derived earlier.
Thus $\frac{\partial\vv'(t)}{\partial t}\!=\!-\log\left(1\!-\!\lambda^{-1}\right)\left(1\!-\!\lambda^{-1}\right)^{\eta(t)}\frac{\partial \eta(t)}{\partial t}$. Plugging this result into Eq. \eqref{eq:gam_eig}, we obtain: 
\begin{equation}
    -\log\left(1\!-\!\lambda^{-1}\right)\frac{\partial \eta(t)}{\partial t}\left(1\!-\!\lambda^{-1}\right)^{\eta(t)}+f(\lambda,t)\lambda\vv'(t)-h(\lambda,t)=0.
    \label{eq:dermaxheat}
\end{equation}
After simple algebraic manipulations we find $f(\lambda,t)\!=\!-\lambda^{-1}\log\left(1\!-\!\lambda^{-1}\right)\frac{\partial \eta(t)}{\partial t}$ and $h(\lambda,t)\!=\!-\log\left(1\!-\!\lambda^{-1}\right)\frac{\partial \eta(t)}{\partial t}$ such that Eq. \eqref{eq:dermaxheat} holds. Putting these results together we obtain  the set of ODE given as:
\begin{equation}
    \frac{\partial\vv'(t)}{\partial t}+\frac{\partial \eta(t)}{\partial t}\log\left(1\!-\!\lambda^{*}\right)\left(1-\vv'(t)\right)=0.
\end{equation}
which is equivalent to Eq. \eqref{eq:heat_maxexp} as, in the above equation, eigenvalues $\lambda^*\!\!=\!\lambda^{-1}$ correspond to the autocorrelation/covariance and graph Laplacian, respectively, which completes the proof.

\vspace{\widsup}
\vspace{-0.0cm}
\subsection*{I. Proof of \Cref{th:gamma_ode}}
\label{app:gamma_ode}
As above, we note that the system of ODE from Eq. \eqref{eq:heat_eq} and thus also Eq. \eqref{eq:heat_gamma} can be rewritten in the span of eigenvectors of the graph Laplacian \cite{heat_eqq}, that is Eq. \eqref{eq:gam_eig}. 
%

Now simply write Gamma as $\left(\lambda^{-1}\right)^{et}$. In this parametrization, we use $\lambda^{-1}$ as we start from the eigenvectors of the graph Laplacian rather than an autocorrelation/covariance matrix, and we use the previous result stating that $\gamma\!=\!et$. Thus $\frac{\partial\vv'(t)}{\partial t}\!=\!e\left(\lambda^{-1}\right)^{et}\log\left(\lambda^{-1}\right)$. Plugging this result into Eq. \eqref{eq:gam_eig}, we obtain:
\begin{equation}
    -e\log\left(\lambda\right)\left(\lambda^{-1}\right)^{et} +f(\lambda)\lambda\vv'(t)=0,
    \label{eq:gam_eig2}
\end{equation}
where $f(\lambda)$ must be equal $e\lambda^{-1}\!\log(\lambda)$ for Eq. \eqref{eq:gam_eig2} to hold. Putting together these results we obtain:
\begin{equation}
    -e\log\left(\lambda\right)\left(\lambda^{-1}\right)^{et} +e\log(\lambda)\lambda^{-1}\!\lambda\vv'(t)=0,
    \label{eq:gam_eig3}
\end{equation}
which simply tells us that the desired result is yielded by the set of ODE:
\vspace{-0.1cm}
\begin{equation}
    \frac{\partial\vv'(t)}{\partial t}+e\log(\lambda)\vv'(t)=0.
    \label{eq:gam_eig4}
\end{equation}
which is equivalent to Eq. \eqref{eq:heat_gamma} which completes the proof.

\vspace{\widsup}
\subsection*{J. Proof of \Cref{pr:sigmoids}}
\label{app:mix_proof}
Probabilities $\psi^{(-)}$ and $\psi^{(+)}$ follow a simple calculus for the probability of selecting a component in a mixture model given a sample. The rest follows simple algebraic manipulations.

\vspace{\widsup}
\vspace{-0.1cm}
\subsection*{K. Proof of \Cref{re:maxexp_sigmoid}}
\label{app:sigme_align_maxexp}
We note that the concavity of SigmE on interval $[0,1]$ is at its maximum for a point $p''(\eta')\!=\!\log(\sqrt{3}\!+\!2)/\eta'$ which we obtain as a solution to:
\begin{equation}
    \frac{\partial^3}{\partial p^3}\frac{2}{1\!+\!e^{-\eta'p}}\!-\!1=0.
    \label{eq:der_sigme_maxexp}
\end{equation}

Subsequently, we formulate the square loss between SigmE and MaxExp at $p''\!$, and take its derivative to solve it for $\eta$ (or $\eta'\!$):
\vspace{-0.1cm}
\begin{equation}
\!\!\!\!\!\!\!\frac{\partial}{\partial\eta}\bigg(\frac{2}{1\!+\!\expl{-\eta'\frac{\log(\sqrt{3}\!+\!2)}{\eta'}\!}}\!-\!1\!-\!\bigg(1\!-\!\Big(1\!-\!\frac{\log(\sqrt{3}\!+\!2)}{\eta'}\Big)^\eta\bigg)\bigg)^2\!\!\!\!=\!0,\!
\end{equation}
%

Regarding the maximum error between SigmE and MaxExp, one should technically find the maximum square difference between SigmE and MaxExp parametrized by $\eta'(\eta)$. However, such an equation has no closed form. Thus, an easier approximate measure is to consider the difference between SigmE and MaxExp at $2p''(\eta')$ or $p'''(\eta')\!=\!\log(\sqrt{26\sqrt{105} + 270}/2 + \sqrt{105}/2 + 13/2)/\eta'$, which is the solution to the fifth derivative of MaxExp: 
\begin{equation}
    \frac{\partial^5}{\partial p^5}\frac{2}{1\!+\!e^{-\eta'p}}\!-\!1=0.
    \label{eq:sigme_maxexp_bound}
\end{equation}

\revisedd{
\vspace{-0.2cm}
\vspace{\widsup}
\subsection*{L. Derivation of the kernel linearization in Eq. \eqref{eq:gauss_lin2}}
\label{app:kern_linear}

Let $G_{\sigma}(\vx\!-\!\vy)\!=\!\exp(-\!\enorm{\vx\!-\!\vy}^2/{2\sigma^2})$ be a Gaussian RBF kernel with a bandwidth $\sigma$. Kernel linearization refers to rewriting $G_{\sigma}$ as an inner-product of two (in)finite-dimensional feature maps. 
Specifically, we employ the inner product of $d'$-dimensional isotropic Gaussians cantered at $\vx,\vy\!\in\!\mbr{d'}\!$: 
\begin{align}
\vspace{-0.8cm}
&\!\!\!\!\!\!\!G_{\sigma}\!\left(\vx\!-\!\vy\right)\!\!=\!\!\left(\frac{2}{\pi\sigma^2}\right)^{\!\!\frac{d'}{2}}\!\!\!\!\!\!\int\limits_{\vzeta\in\mbr{d'}}\!\!\!\!G_{\sigma/\sqrt{2}}\!\!\left(\vx\!-\!\vzeta\right)G_{\sigma/\sqrt{2}}(\vy\!\!-\!\vzeta)\,\mathrm{d}\vzeta,
\label{eq:gauss_integral}
\vspace{-0.5cm}
\end{align}

Eq. \eqref{eq:gauss_integral} can be thought of as a convolution of function $G_{\sigma/\sqrt{2}}\!\!\left(\vx\!-\!\vzeta\right)$ with $G_{\sigma/\sqrt{2}}(\vy\!\!-\!\vzeta)$ centered at $\vx$ and $\vy$, respectively. Both functions  are isometric multivariate Normal distributions if normalized by $\left(\frac{2}{\pi\sigma^2}\right)^{\!\!\frac{d'}{4}}$. To prove that Eq. \eqref{eq:gauss_integral} holds, we consider the sum (denoted as $S$) of two i.i.d. random variables distributed according to two Normal distributions (variance $\sigma/\sqrt{2}$). We note that $S$ is then also distributed according to the Normal distribution (variance $\sigma$). 

Eq. \eqref{eq:gauss_integral} is approximated by replacing the integral with the so-called Riemann sum over $Z$ pivots $\vzeta_1,\cdots,\vzeta_Z$ which represent centers of so-called approximating rectangles in the sum: 
\begin{align}
\vspace{-0.3cm}
&\!\!\!\!\!\!\!G_{\sigma}\!\left(\vx\!-\!\vy\right)\!\approx\!c\!\sum\limits_{i\in\idx{Z}}\varphi_i(\vx; \{\vzeta_i\}_{i\in\idx{Z}})\!\cdot\!\varphi_i(\vy; \{\vzeta_i\}_{i\in\idx{Z}}),
\label{eq:gauss_sumint}
\end{align}
\vspace{-0.2cm}

\noindent{where} $\vvarphi\left(\vx; \{\vzeta_i\}_{i\in\idx{Z}}\right)\!=\!\left[{G}_{\sigma/\sqrt{2}}(\vx-\vzeta_1),\cdots,{G}_{\sigma/\sqrt{2}}(\vx-\vzeta_Z)\right]^T\!\!\!\!,$ and $c$ is a normalization constant related to the normalization in Eq. \eqref{eq:gauss_integral} and the width of approximating rectangles. 

Finally, one dim. input features (Cartesian coordinates normalized in range 0--1) from which we form feature maps correspond to spatial locations in the conv. layer, and so are distributed uniformly. Thus, to cover the entire support set ($\{s\!: f(s)\!>\!0\}$) of random variable $s\!\sim\!S$ and obtain roughly a uniform approximation quality across the support (size of support set), we select $Z$ pivots at equally spaced intervals, that is $[\zeta_1;\cdots;\zeta_Z]\!=\![-0.2:1.4/(Z\!-\!1):1.2]$. The range exceeds $[0,1]$ as the arms ($2\!\times$ standard deviation) of Gaussians ($\sigma/\sqrt{2}$) at extreme locations $0$ and $1$ require the support of $S$ to be roughly $[-2\sigma/\sqrt{2}; 1\!+\!2\sigma/\sqrt{2}]$ to cover $\sim\!95.0$ of the support set. In practice, we found that the support $[-0.2; 1.2]$ and $Z$ in range 3--10 are sufficient. As $0\!<\!c\!<\!\infty$ is a constant, its exact value does not influence the information captured by the maps, thus we set $c\!=\!1$.
}

\revisedd{
\vspace{\widsup}
\subsection*{M. MaxExp in few-shot learning.}
\label{app:maxexp_fsl}
Below we present a motivation similar to one presented in Section \ref{sec:max_exp_mot}. However, we employ the Binomial PMF and the variance-based modeling in contrast to the uniform PMF and its set support modeling from Section \ref{sec:max_exp_mot}. 
Theorem \ref{pr:cooc} states that MaxExp performs a co-occurrence detection rather than counting. For classification problems, let a probability mass function $p_{X_{kl}}(x)\!=\!\text{Binom}(x; N,p)$ for $x\!=\!0,\cdots,N$ and some $p$  tell the probability that co-occurrence $(k,l)$ between $\phi_{kn}$ and $\phi_{ln}$ happened $x\!=\!0,\cdots,N$ times given an image, where $N$ is the number of feature vectors to aggregate. 
Using second-order pooling without MaxExp requires a classifier to observe $N\!+\!1$ training samples of two features co-occurring in quantities $0,\cdots,N$ to memorize their possible co-occurrence counts. For similarity learning, our $\vartheta$ stacks pairs of samples to compare, thus a similarity learner has to deal with a probability mass function ${R_{kl}}\!=\!{X_{kl}}\!+\!{Y_{kl}}$ describing configurations of two features co-occurring whose  $\text{var}(p_{R_{kl}})\!=\!2Np(p\!-\!1)\!>\!\text{var}(p_{X_{kl}})\!=\!Np(p\!-\!1)$ as random variable $X\!=\!Y$ (same class). For $J$-shot learning which stacks one query and $J$ support matrices (per class) in relation descriptor, ${R'_{kl}}\!=\!{X^{(1)}_{kl}}\!+\!\cdots\!+\!{X^{(J)}_{kl}}\!+\!{Y_{kl}}$, $X^{(j)}\!=\!Y, \forall j\!\in\!\idx{J}$ and we have  $\text{var}(p_{R'_{kl}})\!=\!(J\!+\!1)Np(p\!-\!1)$ indicating that the similarity learner has to memorize more configurations as $J$ and/or $N$ grow.

However, this situation is alleviated by MaxExp whose  probability mass function can be modeled as $p_{X^\text{MaxExp}_{kl}}(x)\!=\!p$ if $x\!=\!1$, $p_{X^\text{MaxExp}_{kl}}(x)\!=\!1\!-\!p$ if $x\!=\!0$, as MaxExp detects a co-occurrence (or its lack). Thus, for $J$-shot learning, $\text{var}(p_{{R'}^{\text{MaxExp}}_{kl}})\!=(J\!+\!1)p(p\!-\!1)\!\ll\!\text{var}(p_{R'_{kl}})\!=\!(J\!+\!1)Np(p\!-\!1)$. The ratio of variances of  MaxExp normalized to non-normalized few-shot learning equals $\kappa'\!=\!1/N$ which shows that if MaxExp is used, the similarity learner has to memorize representations which have $N\!\times$ less variance 
compared to the case without MaxExp. 

In our work, we considered an intuitive relation descriptor:
\vspace{-0.2cm}
\begin{equation}
\!\!\!\!\fontsize{7}{8}\selectfont\vartheta_{\text{($\otimes$+P)}}\!\left(\{\mPhi_n\}_{n\in\mathcal{J}},\mPhi^*\!\right)\!=\!\left[\tG_{\text{i}}\left(\frac{1}{J}\!\sum_{n\in\mathcal{J}}\!\frac{1}{N}\mPhi_n\mPhi_n^T\right)\!;_1 \tG_{\text{i}}\left(\frac{1}{N}\mPhi^*\!\mPhi^{*T}\right)\right]\!.\!\!
\label{eq:concat_intuitive}
\end{equation}
\vspace{-0.3cm}

Despite its intuitive nature, $\vartheta_{\text{($\otimes$+P)}}$ in Eq. \eqref{eq:concat_intuitive} performed $\sim$1--2\% worse than $\vartheta_{\text{($\otimes$+L)}}$ in Eq. \eqref{eq:concat_new} for $J\!>\!1$,  and thus was deemed not fit for presentation. Analyzing variance, one can notice that query and support parts of Eq. \eqref{eq:concat_intuitive} are described by $p_{X^\text{MaxExp}_{kl}}(x)$ each, thus $\text{var}(p_{{X^\text{MaxExp}_{kl}}+{Y^\text{MaxExp}_{kl}}})\!=\!2p(p\!-\!1)$, $Y\!=\!X$. Due to a very low variance, the representational power of this relation descriptor is simply insufficient if $J\!>\!1$. Detecting if two co-occurring features $(k,l)$ between $\phi_{kn}$ and $\phi_{ln}$ are detected at least once in $JN$ trials can be thought of as creating one simplified image representing such detections which deprives the similarity learner the individual per-image co-occurrence statistics.
}

\begin{figure*}[!b]
\centering
\vspace{-0.3cm}
\hspace{-0.6cm}
\begin{subfigure}[t]{0.325\linewidth}
\centering\includegraphics[trim=0 0 0 0, clip=true, height=\PowH]{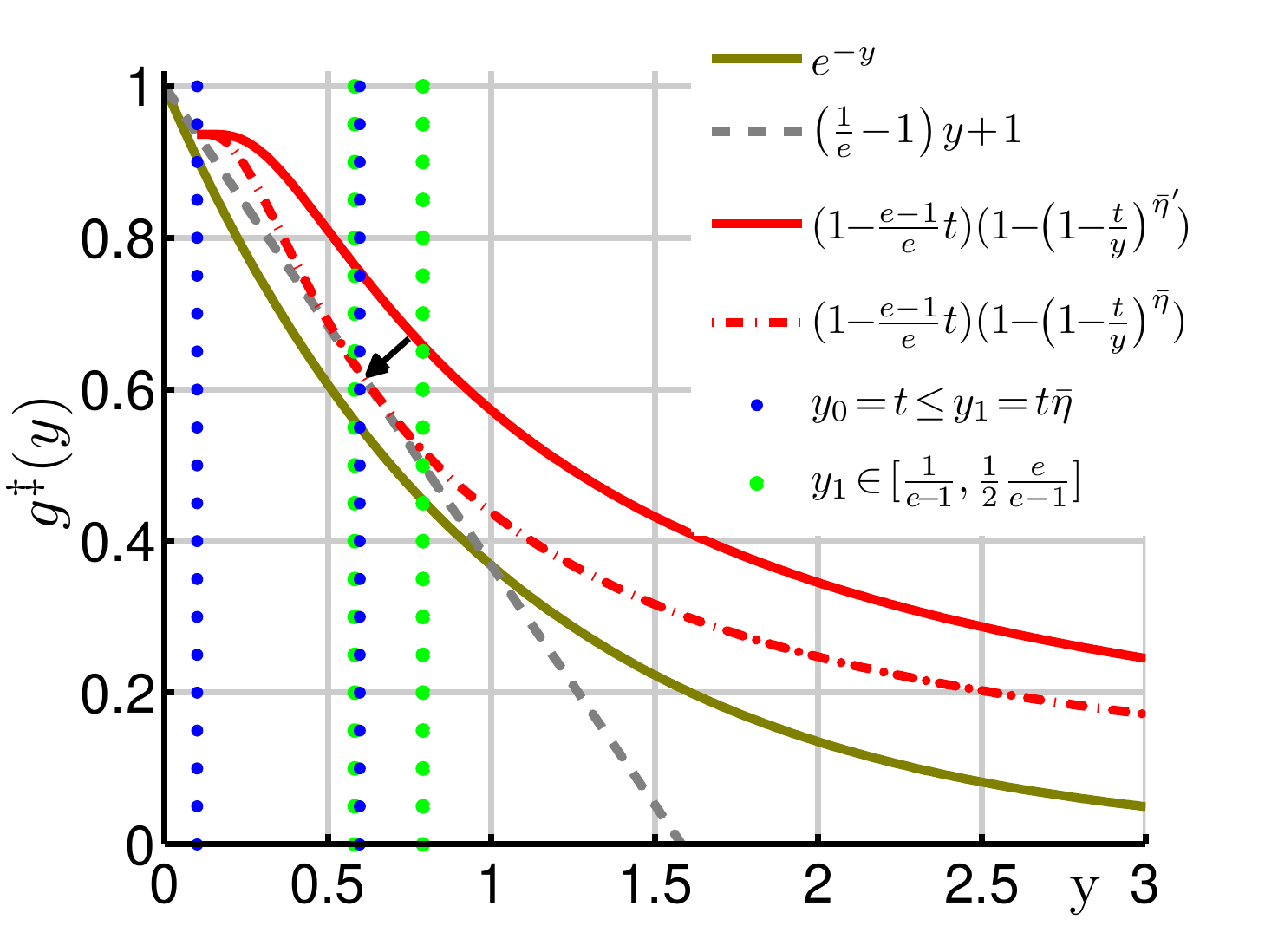}
\caption{$\!\!\!\!\!\!\!\!$}\label{fig:bounds4}
\end{subfigure}
\begin{subfigure}[t]{0.325\linewidth}
\centering\includegraphics[trim=0 0 0 0, clip=true, height=\PowH]{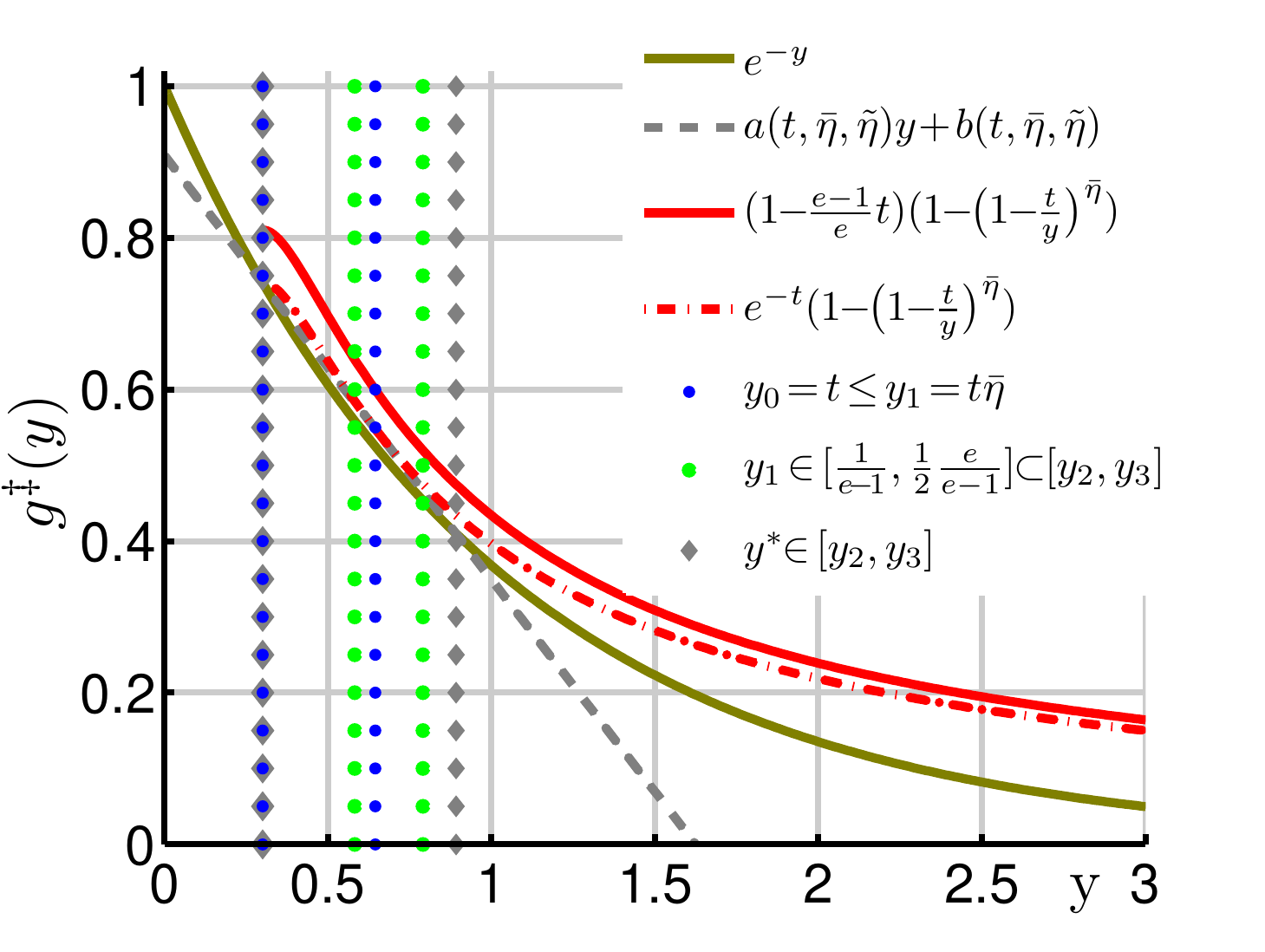}
\caption{$\!\!\!\!\!\!\!\!$}\label{fig:bounds5}
\end{subfigure}
\begin{subfigure}[t]{0.325\linewidth}
\centering\includegraphics[trim=0 0 0 0, clip=true, height=\PowH]{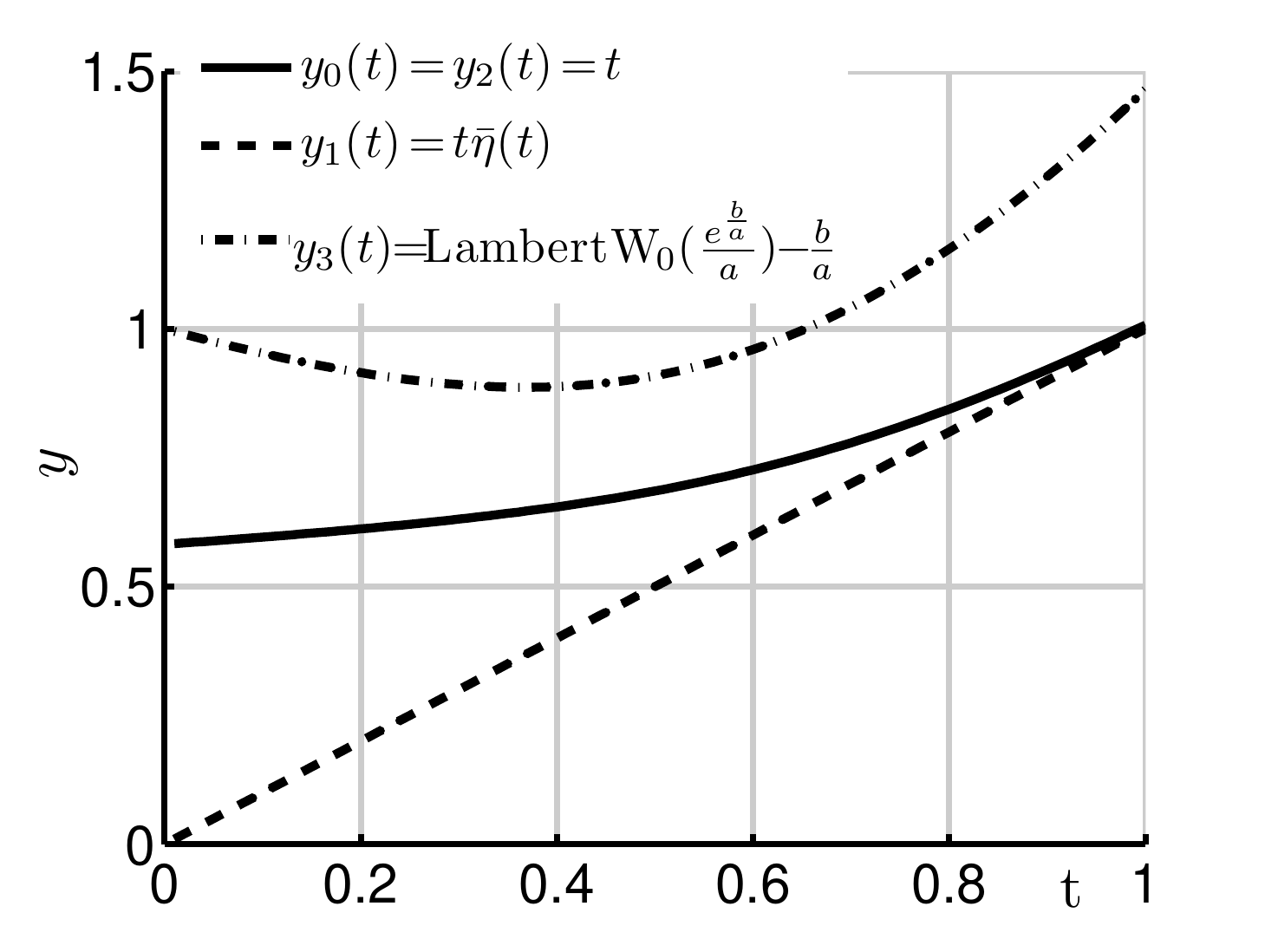}
\caption{$\!\!\!\!\!\!\!\!$}\label{fig:bounds6}
\end{subfigure}
%
\caption{In Fig. \ref{fig:bounds4}, we show pooling functions reparametrized according to $y\!=\!\frac{t}{\lambda}$ and scaled where appropriate to match the value of HDP at $\lambda\!=\!1$. Specifically, we have HDP given by $e^{-y}$, its upper bound $\left(\frac{1}{e}\!-\!1\right)y\!+\!1$ for $y\!\in\!(0,1)$, and a scaled by $1\!-\!\frac{e\!-\!1}{e}t$ MaxExp given as $\left(1\!-\!\frac{e\!-\!1}{e}t\right)\big(1\!-\!\big(1\!-\!\frac{t}{y}\big)^{\bar{\eta}'}\!\big)$ which we `lower down' onto $\left(\frac{1}{e}\!-\!1\right)y\!+\!1$ as indicated by the black arrow. As we tighten the bound, some initial MaxExp with $\bar{\eta}'\!(t)$ becomes MaxExp with $\bar{\eta}(t)$. Blue vertical lines are $y_0\!=\!t$ and $y_1\!=\!t\bar{\eta}(t)$ at which the scaled MaxExp touches $\left(\frac{1}{e}\!-\!1\right)y\!+\!1$. The green lines indicate the range of $y_1\!\in\!\big[\frac{1}{e-1},\frac{1}{2}\frac{e}{e-1}\big]$ given by the image of $y(1\!\leq\!\bar{\eta}\!\leq\!\infty)$. Fig. \ref{fig:bounds5} that the scaled MaxExp given as $\left(1\!-\!\frac{e\!-\!1}{e}t\right)\big(1\!-\!\big(1\!-\!\frac{t}{y}\big)^{\bar{\eta}'}\!\big)$ is lower-bounded by a tightly-scaled MaxExp $e^{-t}\big(1\!-\!\big(1\!-\!\frac{t}{y}\big)^{\bar{\eta}'}\!\big)$ which is lower-bounded by a linear function $a(y(t,\bar{\eta}(t),\widetilde{\eta}(t))\!+\!b(y(t,\bar{\eta}(t),\widetilde{\eta}(t))$ (the blue lines are where the latter two functions touch). Again, the green lines indicate the range of $y_1$ for $1\!\leq\!\bar{\eta}\!\leq\!\infty$ and the gray lines indicate $y_2\!=\!y_0\!=t$ and $y_3$ at which $ay\!+\!b$ touches $r^{-y}$. Finally, \ref{fig:bounds6} verifies that $y_2(t)\!\leq\!y_1(t)\!\leq\!y_3(t)$ for $t\!\in\![0,1]$ which is necessary to prove that the tightly-scaled MaxExp is an upper bound of HDP.}
\vspace{-0.4cm}
\label{fig:repar_new}
\end{figure*}

\revisedd{
\vspace{\widsup}
\subsection*{N. Deriving derivative of fast spectral MaxExp.}
\label{app:maxexp_fast_spec}
%

Let $\ell(\mPsi,\mW)$ be some classification loss (or any layer with param. $\mW$) where $\mPsi\!=\!\mygthreeehat{MaxExp}{\mM}\!\in\!\semipd{d}\!$ (or $\spd{}$) are our feature maps. Then, we obtain a versatile equation from Eq. \eqref{eq:binom_mat_der}:
\begin{align}
& \!\!\!\!\!\!\!\!\!\text{\fontsize{8}{9}\selectfont $\sum\limits_{k,l}\frac{\partial \ell(\mPsi,\mW)}{\partial  \Psi_{kl}}\frac{\partial \Psi_{kl}}{\partial \mM}\!=\!-\!\sum\limits_{n\!=\!0}^{{\eta}-1}\left(\mIdent\!-\!\mM\right)^n\!\frac{\partial \ell(\mPsi,\mW)}{\partial \mPsi}\left(\mIdent\!-\!\mM\right)^{{\eta}-1-n}\!$},\!\!\!
\label{eq:der_maxexp0}
\end{align}
which simplifies to:
\begin{align}
& \!\!\!\!\!\!\text{\fontsize{8}{9}\selectfont $
\sum\limits_{k,l}\frac{\partial \ell(\mPsi,\mW)}{\partial  \Psi_{kl}}\frac{\partial \Psi_{kl}}{\partial \mM}\!=\!-\!2\sym\!\left(\sum\limits_{n\!=\!0}^{\;\lfloor\frac{\eta}{2}\rfloor\!-\!1\!}\!\!\left(\mIdent\!-\!\mM\right)^n\!\frac{\partial \ell(\mPsi,\mW)}{\partial \mPsi}\left(\mIdent\!-\!\mM\right)^{{\eta}-1-n}\!\right)\!$}\!\!\!\nonumber\\
& \qquad\qquad\quad\;\;\text{\fontsize{8}{9}\selectfont$-$}\begin{cases}
\begin{array}{@{}cl}
\text{\fontsize{8}{9}\selectfont $\left(\mIdent\!-\!\mM\right)^{\lfloor\frac{\eta}{2}\rfloor}\!\frac{\partial \ell(\mPsi,\mW)}{\partial \mPsi}\left(\mIdent\!-\!\mM\right)^{\lfloor\frac{\eta}{2}\rfloor}$} & \!\!\!\!\!\text{\fontsize{8}{9}\selectfont if $\eta$  is odd}\\
0 & \!\!\!\!\!\text{\fontsize{8}{9}\selectfont otherwise.}
\end{array}
\end{cases}\!\!\!\!\!\!\!\!\!\!\!\!
\label{eq:der_maxexp}
\end{align}
\vspace{-0.1cm}

\noindent{For} the second half of indexes $0,\cdots,\eta\!-\!1$ (even $\eta$), Eq. \eqref{eq:der_maxexp} uses transposed summation terms corresponding to the first half of indexes instead of recomputing them as in Eq. \eqref{eq:der_maxexp0}. For odd $\eta$, only the term corresponding to index $\lfloor\frac{\eta}{2}\rfloor$ is not aggregated twice.

\vspace{\widsup}
\subsection*{O. Derivative of fast spectral Gamma for integers $\gamma\!\geq\!1$.}
\label{app:gamma_fast_spec}
Let $\ell(\mPsi,\mW)$ be some classification loss/layer with param. $\mW$ (as in Eq. \eqref{eq:der_maxexp}) and $\mPsi\!=\!\mygthreeehat{Gamma}{\mM}\!\in\!\semipd{d}\!$ (or $\spd{}$) be our feature maps. Then, the derivative of Gamma for integers $\gamma(t)\!\geq\!1$ is:
\begin{align}
& \!\!\!\!\!\!\!\!\!\!\!\text{\fontsize{8}{9}\selectfont $\sum\limits_{k,l}\frac{\partial \ell(\mPsi,\mW)}{\partial  \Psi_{kl}}\frac{\partial \Psi_{kl}}{\partial \mM}\!=\!\sum\limits_{n\!=\!0}^{{\gamma}-1}\mM^n\!\frac{\partial \ell(\mPsi,\mW)}{\partial \mPsi}\mM^{{\gamma}-1-n}\!=\!$}\label{eq:der_gamma}\\
& \!\!\!\!\!\!\!\!\!\!\!\text{\fontsize{8}{9}\selectfont $2\sym\!\left(\sum\limits_{n\!=\!0}^{\;\lfloor\frac{\gamma}{2}\rfloor\!-\!1\!}\!\!\mM^n\!\frac{\partial \ell(\mPsi,\mW)}{\partial \mPsi}\mM^{{\gamma}-1-n}\!\right)\!+\!$}
\begin{cases}
\begin{array}{@{}cl}
\text{\fontsize{8}{9}\selectfont $\mM^{\lfloor\frac{\gamma}{2}\rfloor}\!\frac{\partial \ell(\mPsi,\mW)}{\partial \mPsi}\mM^{\lfloor\frac{\gamma}{2}\rfloor}$} & \!\!\!\!\!\text{\fontsize{8}{9}\selectfont if $\gamma$  is odd}\\
0 & \!\!\!\!\!\text{\fontsize{8}{9}\selectfont otherwise.}
\end{array}
\end{cases}\!\!\!\!\!\!\!\!\!\!\!\!\nonumber
\end{align}
\vspace{-0.1cm}
%

\vspace{-0.3cm}
\vspace{\widsup}
\subsection*{P. Derivations of Fast Approximate HDP (FAHDP).}
\label{app:fahdp_der}
Let $0\!\leq\!\lambda\!\leq\!1$. Note that for $\lambda\!=\!1$ and $t\!\geq\!\log(10/9)\!\approx\!-0.105$, HDP given by $e^{-t/\lambda}$ yields $e^{-t}\!\leq\!0.9$ (it drops more for larger $t$) while MaxExp given by $1\!-\!(1\!-\!\lambda)^\eta$ and Gamma given by $\lambda^\gamma$ yield $1$ for $\lambda\!=\!1$ and $t\!\geq\!0$. For this reason, we reparametrize MaxExp and Gamma as $e^{-t}(1\!-\!(1\!-\!\lambda)^{\bar{\eta}(t)})$ and $e^{-t}\lambda^{\bar{\gamma}(t)}$, respectively, where $e^{-t}$ is the scaling factor ensuring that MaxExp/Gamma (and thus FAHDP) and HDP have the same magnitude at $\lambda\!=\!1$.

\vspace{0.05cm}
\noindent{\textbf{Time-reversed HDP.}} To obtain a good approximation of time-reversed HDP ($t\!<\!1$) by a scaled MaxExp, we take steps as those in Appendix \hyperref[{app:pr_bound}]{F}, that is we use the substitution $y\!=\!\frac{t}{\lambda}$, we tie together $t$ and $\bar{\eta}$ as $t\bar{\eta}\!=\!\alpha$. 
Subsequently, we seek to  `lowered down' MaxExp given now as $e^{-t}\!\big(1\!-\!\big(1\!-\!\frac{t}{y}\big)^{\bar{\eta}}\big)$ onto $\left(\frac{1}{e}\!-\!1\right)y\!+\!1$, an upper bound of reparametrized HDP given as $e^{-y}$. We obtain: 
\vspace{-0.1cm}
\begin{equation}
\!\!\!\!\begin{cases}
\left(\frac{1}{e}\!-\!1\right)y\!+\!1=e^{-t}\big(1\!-\!\big(1\!-\!\frac{t}{y}\big)^{\bar{\eta}}\big)\\
\frac{\partial \left(\frac{1}{e}\!-\!1\right)y\!+\!1}{\partial y}=-\frac{\partial e^{-t}\left(1\!-\!\frac{t}{y}\right)^{\bar{\eta}}}{\partial y}.\!
\label{eq:touch_maxexp_scaled1}
\end{cases}
\end{equation}
We notice that \eqref{eq:touch_maxexp_scaled1} may not have a closed form solution. 

\vspace{0.05cm}
\noindent{\textbf{Scaled MaxExp (Fig. \ref{fig:bounds4}).}} 
We notice that $1\!-\!\frac{e-1}{e}t$ is an upper bound of $e^t$ on interval $t\!\in\!(0,1)$, thus we use scaling $1\!-\!\frac{e-1}{e}t$ 
and we solve the following set of equations instead of Eq. \eqref{eq:touch_maxexp_scaled1}:
\vspace{-0.2cm}
\begin{equation}
\!\!\!\!\begin{cases}
\left(\frac{1}{e}\!-\!1\right)y\!+\!1=\big(1\!-\!\frac{e-1}{e}t\big)\big(1\!-\!\big(1\!-\!\frac{t}{y}\big)^{\bar{\eta}}\big)\\
\frac{\partial \left(\frac{1}{e}\!-\!1\right)y\!+\!1}{\partial y}=-\big(1\!-\!\frac{e-1}{e}t\big)\frac{\partial\left(1\!-\!\frac{t}{y}\right)^{\bar{\eta}}}{\partial y}.\!
\label{eq:touch_maxexp_scaled2}
\end{cases}
\end{equation}
By solving Eq. \eqref{eq:touch_maxexp_scaled2}, we get $y\!=\!\alpha\!=\!t\bar{\eta}\vee y\!=\!t\!=\!\frac{\alpha}{\bar{\eta}}$. We denote the first result as $y_1\!=\!\alpha\!=\!t\bar{\eta}$ and the later one as $y_0\!=\!t\!=\!\frac{\alpha}{\bar{\eta}}$. We notice that $y_0\!\leq\!y_1$ for $\bar{\eta}\!\geq\!1$ and that $y_0$ corresponds to the solution where $\lambda\!=\!\frac{t}{y_0}\!=\!\frac{t}{t}\!=\!1$ which is one of the two points for which the set of equations in \eqref{eq:touch_maxexp_scaled2} is fulfilled. The other point for which the above set of eq. is fulfilled is $y\!=\!\alpha$. Thus, we have:
%
\vspace{-0.1cm}
\begin{equation}
\alpha(\bar{\eta})\!=\!y(\bar{\eta})\!=\!t\bar{\eta}
\;\text{ and }\; t(\bar{\eta})\!=\!\frac{\frac{e}{e\!-\!1}\!\left(\frac{\bar{\eta}\!-\!1}{\bar{\eta}}\right)^{\bar{\eta}}}{\left(\frac{\bar{\eta}\!-\!1}{\bar{\eta}}\right)^{\bar{\eta}}\!\!\!\!+\!\bar{\eta}\!-\!1},
\label{eq:yalpha_scaled}
\end{equation}
where $\bar{\eta}\!>\!1$. For $\bar{\eta}\!=\!1$, $t(\bar{\eta})$ and $\alpha(\bar{\eta})$ are undefined but  $\lim\limits_{\bar{\eta}\!\rightarrow\!1}t(\bar{\eta})\!=\!\lim\limits_{\bar{\eta}\!\rightarrow\!1}\alpha(\bar{\eta})\!=\!\frac{e}{2(e\!-\!1)}$ from the L'Hospital's rule.

\begin{figure*}[!b]
\centering
\vspace{-0.3cm}
\hspace{-0.6cm}
\begin{subfigure}[t]{0.325\linewidth}
\centering\includegraphics[trim=0 0 0 0, clip=true, height=\PowH]{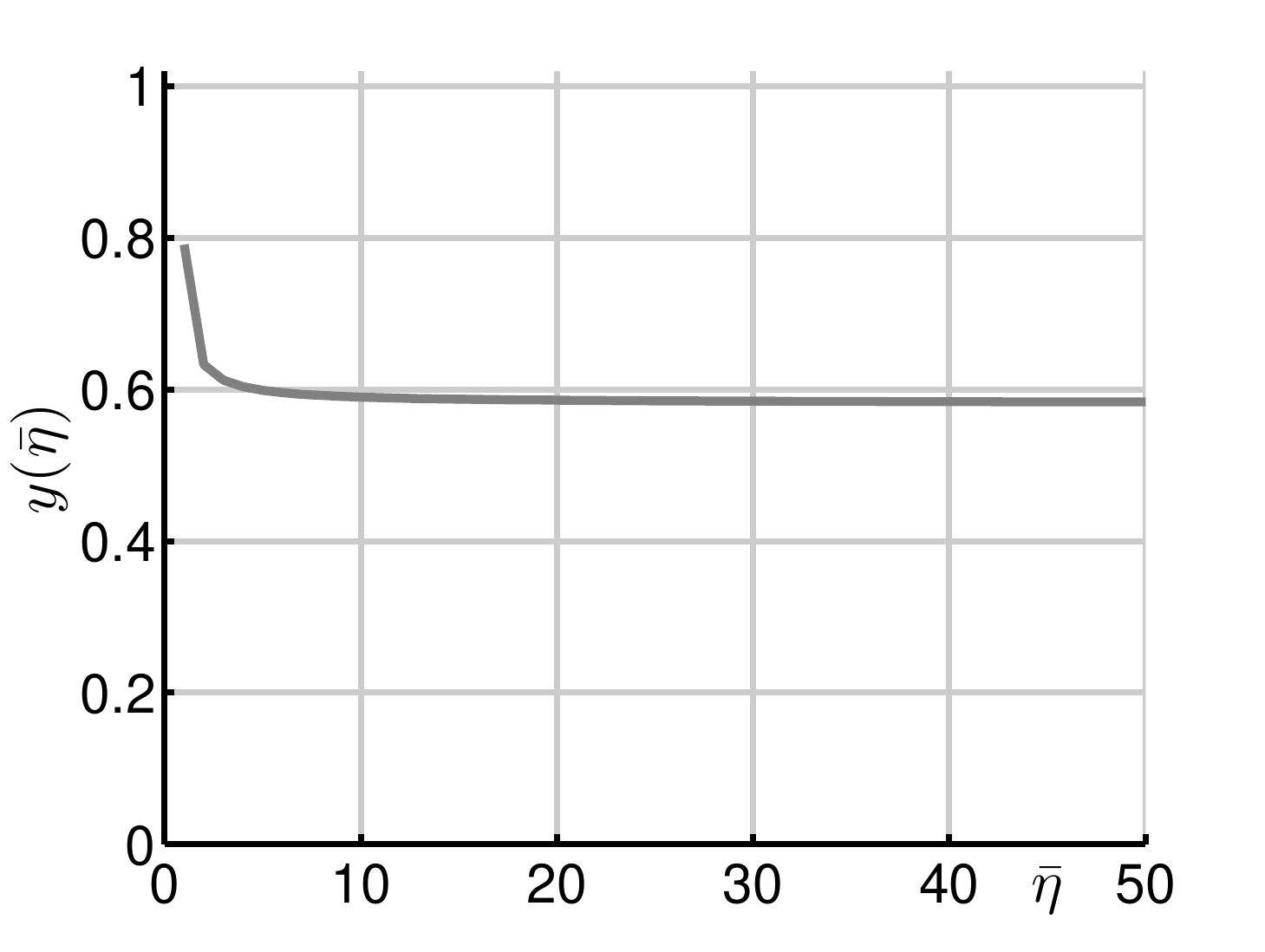}
\caption{$\!\!\!\!\!\!\!\!$}\label{fig:bounds7}
\end{subfigure}
\begin{subfigure}[t]{0.325\linewidth}
\centering\includegraphics[trim=0 0 0 0, clip=true, height=\PowH]{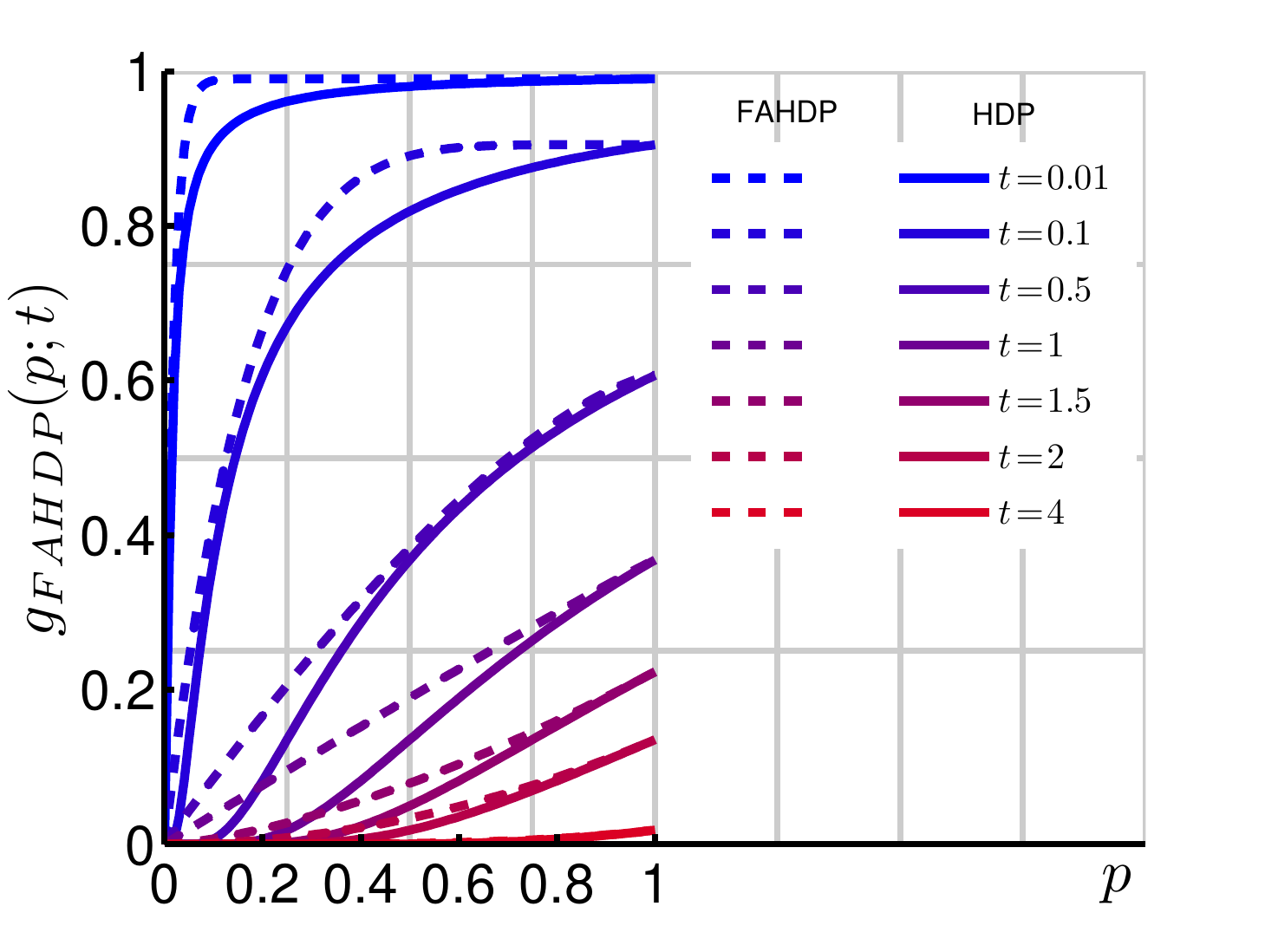}
\caption{$\!\!\!\!\!\!\!\!$}\label{fig:bounds8}
\end{subfigure}
%
\caption{In Fig. \ref{fig:bounds7}, we show that $y(\bar{\eta})$ decreases monotonically as $\bar{\eta}\!\rightarrow\!\infty$. Fig. \ref{fig:bounds8} illustrates that the FAHDP (combination of the tightly-scaled MaxExp and the scaled Gamma) is a tight upper bound of HDP.}
\vspace{-0.4cm}
\label{fig:repar_new2}
\end{figure*}

To obtain $\bar{\eta}(t)$, the inverse of $t(\bar{\eta})$, we firstly apply a substitution $\bar{\eta}\!=\!\hat{\eta}\!+\!1$ and evaluate $t(\hat{\eta}\!+\!1)$:
\vspace{-0.2cm}
\begin{equation}
t(\hat{\eta}\!+\!1)\!=\!\frac{e}{e\!-\!1}\frac{\hat{\eta}^{\hat{\eta}}}{\hat{\eta}^{\hat{\eta}}\!+\!(\hat{\eta}\!+\!1)^{\hat{\eta}\!+\!1}}.
\label{eq:t_hat}
\end{equation}
\vspace{-0.3cm}

\noindent{Next}, we employ the Stirling approximation $\hat{\eta}^{\hat{\eta}}\!\approx\!\hat{\eta}!\,e^{\hat{\eta}}/\sqrt{2\pi\hat{\eta}}\!=\!\widetilde{\eta}^{\widetilde{\eta}}$. We note that $\hat{\eta}^{\hat{\eta}}\!\leq\!\widetilde{\eta}^{\widetilde{\eta}}, \forall\hat{\eta},\widetilde{\eta}\!\geq\!1$ but the gap decreases monotonically as $\hat{\eta},\widetilde{\eta}\!\rightarrow\!\infty$. Subsequently, we obtain:
\vspace{-0.2cm}
\begin{equation}
\!\!\!\!\!\!\!\!t(\hat{\eta}\!+\!1)\!\approx\!\frac{\frac{e}{e\!-\!1}\sqrt{\hat{\eta}\!+\!1}}{ \sqrt{\hat{\eta}\!+\!1}\!+\!e(\hat{\eta}\!+\!1)\sqrt{\hat{\eta}}} \Rightarrow
t(\bar{\eta})\!\approx\!\frac{\frac{e}{e\!-\!1}\sqrt{\bar{\eta}}}{\sqrt{\bar{\eta}}\!+\!e\bar{\eta}\sqrt{\bar{\eta}\!-\!1}}\!=\!\widetilde{t}(\bar{\eta}).\!\!
\label{eq:t_hat_approx}
\end{equation}
From Eq. \eqref{eq:t_hat_approx}, it is straightforward to recover $\bar{\eta}(t)$ which is:
%
\begin{equation}
\bar{\eta}(t)\!\approx\!0.5\!+\!\sqrt{0.25\!+\!\left(1/(t(e\!-\!1))\!-\!1/e\right)^2}\!=\!\widetilde{\eta}(t).
\label{eq:eta_hat_approx}
\end{equation}

For Eq. \eqref{eq:touch_maxexp_scaled2},
%
%
%
we have to check if $y\!\in(0,1)$ for $1\!\leq\!\bar{\eta}\!\leq\!\infty$ in order for the bound to hold as $\left(\frac{1}{e}\!-\!1\right)y\!+\!1$ is an upper bound of $e^{-y}$ only for $y\!\in\!(0,1)$. To this end, we firstly notice that $y(\bar{\eta})$ is monotonically decreasing on $1\leq\bar{\eta}\leq\infty$ as \Cref{fig:bounds7} shows. 
Therefore, it suffices to check extremes of $\bar{\eta}$ for $1\!\leq\!\bar{\eta}\!\leq\!\infty$, that is $y(1)\!=\!\frac{1}{2}\frac{e}{e-1}$ and $\lim_{\bar{\eta}\rightarrow\infty}y(\bar{\eta})\!=\!\frac{1}{e-1}$ which verifies that $y(1\!\leq\!\bar{\eta}\!\leq\!\infty)\!\subset\!(0,1)$.
%
Thus, for $y\!\in(0,1)$ we have:
\begin{equation}
\text{\fontsize{8}{9}\selectfont $e^{-y}\!\leq\!\Big(\frac{1}{e}\!-\!1\Big)y\!+\!1\!\leq\!\Big(1\!-\!\frac{e-1}{e}t\Big)\Big(1\!-\!\big(1\!-\!\frac{t}{y}\big)^{\bar{\eta}}\Big)$}.
\end{equation}

\vspace{0.05cm}
\noindent{\textbf{Tightly-scaled MaxExp (Fig. \ref{fig:bounds5}).}} We tighten the bound by replacing $1\!-\!\frac{e-1}{e}t$ with $e^{-t}$ as $e^{-t}\!\leq\!1\!-\!\frac{e-1}{e}t$ for $t\!\in\![0,1]$. Thus, we seek $ay\!+\!b$ in:
\vspace{-0.2cm}
\begin{equation}
\text{\fontsize{8}{9}\selectfont $e^{-y}\!\leq\!ay\!+\!b\!\leq\!e^{-t}\Big(1\!-\!\big(1\!-\!\frac{t}{y}\big)^{\widetilde{\eta}}\Big)$}.
\label{eq:max_exp_final_set}
\end{equation}
Firstly, we notice that $\widetilde{\eta}(t)\!\geq\!\bar{\eta}(t)$ whose gap decreases rapidly and monotonically as $t\!\rightarrow\!0$. 
We can recover $\epsilon_{\widetilde{\eta}}(\bar{\eta})\!=\!\widetilde{\eta}\!-\!\bar{\eta}\!=\!\widetilde{\eta}(t(\bar{\eta}))\!-\!\bar{\eta}$ and $\epsilon_{\widetilde{t}}(t)\!=\!t\!-\!t(\widetilde{\eta}(t))$. This upper bound $\widetilde{\eta}(t)\!\geq\!\bar{\eta}(t)$ assures that $1\!-\!\big(1\!-\!\frac{t}{y}\big)^{\widetilde{\eta}}\!\geq\!1\!-\!\big(1\!-\!\frac{t}{y}\big)^{\bar{\eta}}$ which turns out to be a sufficient upper bound to fulfill \eqref{eq:max_exp_final_set}. 
To this end, we note that as {\fontsize{8}{9}\selectfont $\Big(\frac{1}{e}\!-\!1\Big)y\!+\!1\!=\!\Big(1\!-\!\frac{e-1}{e}t\Big)\Big(1\!-\!\big(1\!-\!\frac{t}{y}\big)^{\bar{\eta}}\Big)$} for $y_0$ and $y_1$  (defined under Eq. \eqref{eq:touch_maxexp_scaled2}), that is the nonlinear function touches the linear one only at $y_0$ and $y_1$, then { $ay\!+\!b\!=\!e^{-t}\big(1\!-\!\big(1\!-\!\frac{t}{y}\big)^{\bar{\eta}}\big)$} at $y_0$ and $y_1$ because $\big(\frac{1}{e}\!-\!1\big)y\!+\!1$ and $e^{-t}$  perform just some scaling of $1\!-\!\big(1\!-\!\frac{t}{y}\big)^{\bar{\eta}}$. Thus, we define $\bar{\eta}(t)\!=\!\widetilde{\eta}(2t\!-\!t(\widetilde{\eta}(t)))$ and we readily obtain:
\vspace{-0.2cm}
\begin{equation}
\!\!\!\!\!\text{\fontsize{8}{9}\selectfont $a(t,\bar{\eta},\widetilde{\eta})\!=\!\frac{e^{-t}}{t}\frac{\left(\frac{\bar{\eta}(t)\!-\!1}{\bar{\eta}(t)}\right)^{\widetilde{\eta}}}{1\!-\!\bar{\eta}(t)}\;$ and $\;b(t,\bar{\eta},\widetilde{\eta})\!=\!e^{-t}\bigg(1-\frac{\left(\frac{\bar{\eta}(t)\!-\!1}{\bar{\eta}(t)}\right)^{\widetilde{\eta}}}{1\!-\!\bar{\eta}(t)}\bigg)$}.
\label{eq:max_exp_final_set2}
\end{equation}

Eq. $a(t,\bar{\eta},\widetilde{\eta})y\!+\!b(t,\bar{\eta},\widetilde{\eta})$ ($a$ and $b$ are defined in \eqref{eq:max_exp_final_set2}) is a lower bound of $e^{-t}\big(1\!-\!\big(1\!-\!\frac{t}{y}\big)^{\widetilde{\eta}}\big)$ on interval $y\!\in\!(t,\infty)$ by design. Now, it suffices to check if $a(t,\bar{\eta}(t),\widetilde{\eta}(t))y\!+\!b(t,\bar{\eta}(t),\widetilde{\eta}(t))$ is also an upper bound of $e^{-y}$ on interval $y\!\in\!(y_0,y_1)\!\equiv\!(t,t\bar{\eta}(t))$. Thus, we substitute $a$ and $b$ from Eq. \eqref{eq:max_exp_final_set2} into following:
\vspace{-0.1cm}
\begin{equation}
e^{-y^*}\!\!\leq\!a(t,\bar{\eta}(t),\widetilde{\eta}(t))y^*\!\!+\!b(t,\bar{\eta}(t),\widetilde{\eta}(t)).
\label{eq:max_exp_lambert1}
\end{equation}
After several manipulations we arrive at the solution $y^*\!\!\in\![y_2,y_3]$:
\vspace{-0.2cm}
\begin{equation}
\!\!\!\!\!\!\!\text{
\fontsize{8}{9}\selectfont 
$y_2\!=\!\text{LambertW}_{\!-1}\Big(\frac{e^{\frac{b}{a}}}{a}\Big)\!-\!\frac{b}{a} \;\text{ and }\; y_3\!=\!\text{LambertW}_{0}\Big(\frac{e^{\frac{b}{a}}}{a}\Big)\!-\!\frac{b}{a}$},\!\!\!\!
\label{eq:max_exp_lambert2}
\end{equation}
where the $\text{LambertW}$ solves $W(z)e^{W(z)}\!=\!z$. We observe that $y_2\!=\!y_0$ and $y_2\!\leq\!y_1\!\leq\!y_3$ (Fig. \ref{fig:bounds6}), and thus we have:
\vspace{-0.1cm}
\begin{equation}
\left[t,t\bar{\eta}(t)\right]\!\subset\!\Big[t,\text{LambertW}_{0}\Big(\frac{e^{\frac{b}{a}}}{a}\Big)\!-\!\frac{b}{a}\Big],
\label{eq:tight_verify}
\end{equation}
where $a(t,\bar{\eta}(t),\widetilde{\eta}(t))$ and $b(t,\bar{\eta}(t),\widetilde{\eta}(t))$ depend explicitly on time $t$. 
Eq. \eqref{eq:tight_verify} verifies the tighter bound achieved by Eq. \eqref{eq:max_exp_final_set}.

Finally, bounds $\epsilon_3$ and $\epsilon_4$  at $y_1$ and $y_3$ may be evaluated by plugging them respectively into $e^{-t}\big(1\!-\!\big(1\!-\!\frac{t}{y}\big)^{\widetilde{\eta}}\big)\!-\!y^{-y}$:

\vspace{-0.3cm}
\begin{equation}
{\fontsize{8}{9}\selectfont\text{
$\!\!\!\!\!\!\!\!\!\!\!\!\!\!\!\!e^{-t}\!\!\!-\!e^{-t}\Big(\frac{\bar{\eta}(t)\!-\!1}{\bar{\eta}(t)}\Big)^{\widetilde{\eta}(t)}\!\!\!\!\!\!\!-\!e^{-t\bar{\eta}(t)}\!\!=\!\epsilon_3\!\leq\!\epsilon_4\!=\!e^{-t}\!\!\!-\!e^{-t}\Big(\frac{y_3(t)\!-\!t}{y_3(t)}\Big)^{\widetilde{\eta}(t)}\!\!\!\!-\!e^{-y_3(t)} \!\!\!,\!\!$}}
\label{eq:bounds_new}
\end{equation}
\vspace{-0.3cm}

\vspace{0.05cm}
\noindent{\textbf{Time-forward HDP.}} To obtain a good approximation of time-forward HDP ($t\!\geq\!1$) by Gamma, we reparametrize and scale Gamma given as $e^{-t}\lambda^{\bar{\gamma}}$. We solve the following set of equations:
\vspace{-0.1cm}
\begin{equation}
\label{eq:touch_gamma_scaled1}
\begin{cases}
e^{-t/\lambda}=e^{-t}\lambda^{\bar{\gamma}}\\
\frac{\partial e^{-t/\lambda}}{\partial \lambda}=\frac{\partial e^{-t}\lambda^{\bar{\gamma}}}{\partial \lambda}  \Rightarrow -\frac{t}{\lambda^2}e^{-t/\lambda}={\bar{\gamma}} e^{-t}\lambda^{\bar{\gamma}-1}\!\!,
\end{cases}
\end{equation}
which yields the following set of equations:
%
\begin{equation}
\begin{cases}
\bar{\gamma}e^{t/\lambda}\lambda^{\bar{\gamma}+1}=te^t\\
e^{t/\lambda}\lambda^{\bar{\gamma}}=e^t\!.
\label{eq:touch_gamma_scaled2}
\end{cases}
\end{equation}
After some manipulations, the above set of equations further reduces to:
\vspace{-0.2cm}
\begin{equation}
\bar{\gamma}\!-\!\bar{\gamma}\!\log(\gamma)\!=\!t\!-\!\gamma\log(t).
\label{eq:touch_gamma_scaled3}
\end{equation}
To solve this equation, it suffices to use an approximation of logarithm $\log(x)\!\approx\!rx^{1/r}\!-\!r$ for order $r\!=\!2$ which is also an upper bound of $\log(x)$, that is $\log(x)\!\leq\!rx^{1/r}\!-\!r$. Solving Eq. \eqref{eq:touch_gamma_scaled3}  yields $\bar{\gamma}(t)\!=\!t$.

\vspace{0.05cm}
\noindent{\textbf{Fast Approximate HDP (FAHDP).}} We combine the right-hand side of Eq. \eqref{eq:max_exp_final_set} given as $e^{-t}\big(1\!-\!\big(1\!-\!\lambda\big)^{\widetilde{\eta}}\big)$ with $\widetilde{\eta}$ in Eq. \eqref{eq:eta_hat_approx} for $t\!\in\!(0,1)$ and $e^{-t}\lambda^{\bar{\gamma}}$ defined just above Eq. \eqref{eq:touch_gamma_scaled1} with $\bar{\gamma}(t)\!=\!t$ for $t\!\geq\!0$. We assume that we operate on an SVD of an autocorrelation matrix $\mM$ that is normalized by its trace. Thus, we have $p_i$ which correspond to trace-normalized $\lambda_i$, and we define: 
\vspace{-0.1cm}
\begin{align}
& \;g_{\text{FAHDP}}{(p;t)}\!=\!e^{-t}\!\cdot\!
\begin{cases}
\begin{array}{@{}cl}
\text{\fontsize{8}{9}\selectfont $1\!-\!\left(1\!-\!p\right)^{\hbar_t(\widetilde{\eta}(t))}$} & \!\!\!\!\!\text{\fontsize{8}{9}\selectfont if $t\!<\!1$}\\
\text{\fontsize{8}{9}\selectfont $p^{\hbar_t(\bar{\gamma}(t))}$} & \!\!\!\!\!\text{\fontsize{8}{9}\selectfont if $t\!\geq\!1$,}
\end{array}
\end{cases}\!\!\!\!\!\!\!\!\!\!\!\!
\label{eq:fahdp2},
\end{align}
\vspace{-0.2cm}

\noindent{where} $\hbar_t(x)\!=\!x$ or $\hbar_t(\cdot)$ is defined in the same way as for Eq. \eqref{eq:fahdp}. 
Eq. \eqref{eq:fahdp2} operates on the trace-normalized eigenvalues $0\!\leq\!p\!\leq\!1$ and it can be combined with Algorithm \ref{code:algorithm}. 
Moreover, Eq. \eqref{eq:fahdp2} has its closed form given in Eq. \eqref{eq:fahdp} which enjoys a fast back-propagation. FAHDP and HDP are illustrated together in Fig. \ref{fig:bounds8}.

\vspace{\widsup}
\subsection*{Q. Additional results on GIN0 with SOP+PN.}
\label{app:graph-extra}

\begin{table}[h]
\hspace{-0.3cm}
\setlength{\tabcolsep}{0.10em}
\renewcommand{\arraystretch}{0.70}
\vspace{0.1cm}
\hspace{-0.105cm}
\begin{tabular}{l l|c|c|c|c|}
Method && \kern-0.1em IMDB-BIN\kern-0.1em & \kern-0.0em IMDB-MULTI\kern-0.0em & \kern0.05em PTC\kern-0.0em & \kern0.05em NCI1\kern0.1em\\
\hline
{\em GIN0}  & & 75.5$\!\pm\!$4.0 & 51.3$\!\pm\!$3.4 & 66.1$\!\pm\!$6.7 & 82.1$\!\pm\!$1.8 \\
\hline
{\em \fontsize{8}{9}\selectfont SOP+Newton-Schulz} & & 77.3$\!\pm\!$4.4  & 52.2$\!\pm\!$3.2 & 68.8$\!\pm\!$6.4 & 82.3$\!\pm\!$2.2 \\
{\em \fontsize{8}{9}\selectfont SOP+Spec. MaxExp(F)} && \textbf{77.8}$\!\pm\!$3.6 & \textbf{53.5}$\!\pm\!$2.1 & \textbf{72.2}$\!\pm\!$7.0 & \textbf{82.5}$\!\pm\!$2.5 \\
\end{tabular}
\caption{Classification with GIN0 and various pooling methods on IMDB-BIN, IMDB-MULTI, PTC and NCI1 (after obtaining hyperparameters, the train and validation splits were combined for retraining).}
\label{tab:graph2}
\vspace{-0.35cm}
\end{table}
}

\end{document}